\documentclass[11pt,letterpaper]{article}

\usepackage[margin=1in]{geometry}
\usepackage[T1]{fontenc}
\usepackage{times}
\usepackage[authoryear,round]{natbib}
\usepackage{xcolor}
\usepackage[english]{babel}
\usepackage{amsmath,amssymb,amsthm,mathtools}
\usepackage{mathrsfs}
\usepackage{booktabs}
\usepackage{enumitem}
\usepackage{graphicx}
\usepackage{placeins}
\usepackage{microtype}
\usepackage{xurl}
\usepackage[colorlinks=true,allcolors=blue]{hyperref}
\usepackage{needspace}
\hypersetup{
  pdftitle={Universality and Generalization of Causal Transformers Across Context Lengths},
  pdfauthor={Takashi Furuya; Maarten V. de Hoop; Gabriel Peyr\'e}
}

\newtheorem{theorem}{Theorem}[section]
\newtheorem{proposition}[theorem]{Proposition}
\newtheorem{lemma}[theorem]{Lemma}
\newtheorem{corollary}[theorem]{Corollary}
\theoremstyle{definition}
\newtheorem{definition}[theorem]{Definition}

\newtheorem{example}[theorem]{Example}
\theoremstyle{remark}
\newtheorem{remark}[theorem]{Remark}

\newcommand{\R}{\mathbb{R}}
\newcommand{\N}{\mathbb{N}}
\newcommand{\cP}{\mathcal{P}}
\newcommand{\cM}{\mathcal{M}}
\newcommand{\cA}{\mathcal{A}}
\newcommand{\dd}{\mathop{}\!\mathrm{d}}
\newcommand{\norm}[1]{\left\lVert #1\right\rVert}
\newcommand{\abs}[1]{\left\lvert #1\right\rvert}
\newcommand{\ip}[2]{\left\langle #1,#2\right\rangle}
\newcommand{\eqdef}{\coloneqq}
\newcommand{\MAtt}{\operatorname{MAtt}}
\newcommand{\MLP}{\operatorname{MLP}}
\newcommand{\Unif}{\operatorname{Unif}}
\newcommand{\In}{\mathcal{I}}
\newcommand{\Cn}{\mathcal{C}}

\title{Universality and Generalization of\\
Causal Transformers Across Context Lengths}

\author{
  Takashi Furuya\\
  {\small Doshisha University, RIKEN AIP}\\
  {\small\texttt{tfuruya@mail.doshisha.ac.jp}}
  \and
  Maarten V. de Hoop\\
  {\small Rice University}\\
  {\small\texttt{mdehoop@rice.edu}}
  \and
  Gabriel Peyr\'e\\
  {\small CNRS, ENS, PSL Universit\'e}\\
  {\small\texttt{gabriel.peyre@ens.fr}}
}
\date{}

\begin{document}

\maketitle

\begin{abstract}
Long contexts are central to modern transformer systems, but most expressivity results choose a different network for each fixed sequence length. We study whether one masked transformer can approximate causal token-to-token maps uniformly over sequences of arbitrary length sampling a fixed normalized horizon. To relate sampling resolutions, we model tokens by \(\alpha\)-H\"older sequences or, more generally, a common modulus of continuity. Our notion of continuity across resolutions characterizes the causal families admitting uniform approximation on these compact input classes by a single transformer with length-independent parameters. The result extends to the infinite-length mean-field limit, where tokens form continuous curves and masked attention becomes a causal time integral. For bounded regression with target maps satisfying a \(\beta\)-smooth stability condition defined using regular test functions, quantitative approximation yields a generalization bound: exact empirical risk minimization over suitably sized bounded-weight transformers gives root mean-square prediction error \(O((\log\log N/\log N)^{\beta/(d+2)})\) from \(N\) iid labeled sequences. The bound holds at fixed confidence on the same sampling distribution, with \(d\) the token dimension and no maximum-length factor. Finally, experiments on physical time series support the H\"older-regular token model at observed scales, with dataset-dependent fitted exponents, whereas text input embeddings provide a contrasting case. Native and dense sampling, shuffled controls, and refinement checks delimit this empirical regularity regime.
\end{abstract}

\section{Introduction}
\label{sec:introduction}

Long-context transformers apply one learned rule across sequence lengths, yet
most universality results fix the length before choosing the network.  We show
that common input regularity, together with asymptotic continuity of the target
across resolutions, enables uniform causal approximation at every resolution,
connects the discrete model to its continuous-time limit, and yields approximation
and statistical learning rates under additional target regularity. We also clarify the role of
positional information and examine the input geometry empirically.

\paragraph{Transformers and long-context architectures.}
The original transformer combines multi-head attention, pointwise feed-forward
layers, and positional encodings~\citep{vaswani2017attention}.  Long-context
variants use recurrence, sparse or efficient attention, and position
extrapolation~\citep{dai2019transformerxl,beltagy2020longformer,
dao2022flashattention,ding2024longrope}, while RoPE, ALiBi, and positional
interpolation target generalization beyond the training
window~\citep{su2024roformer,press2022train,chen2023position}.  Short-to-long
generalization has also been studied under an idealized learning
rule~\citep{huang2025length}.  We ask a different expressivity question: can one
parameter set approximate a causal rule uniformly over every resolution?
Here, ``all resolutions'' means increasingly fine samples of a fixed normalized
horizon under a common regularity model, not unrestricted growth of a symbolic
string.  In infinite-token limits,
unmasked attention naturally acts on empirical token
distributions~\citep{vuckovic2020attention}; causal masking instead requires token--position
information to preserve chronology~\citep{castin2024smooth,
furuya2025transformers}.  We compactify all grids jointly and recover temporal
attention as the refinement limit of the same finite architecture.  This
token-number limit is distinct from a continuous-depth
limit~\citep{karagodin2024causal}.

\paragraph{Neural-network and transformer universality.}
Classical results establish universality of feed-forward networks on compact
finite-dimensional sets and of neural networks between function
spaces~\citep{cybenko1989approximation,leshno1993nonpolynomial,chen1995operators}.
Neural operators extend this viewpoint to discretization-invariant maps,
without generally enforcing causality~\citep{kovachki2023neuraloperator}.
Transformer universality is known for fixed-length equivariant or
position-aware maps and for sparse or shallow
variants~\citep{yun2020transformers,yun2020sparse,kajitsuka2024one}; related results
treat fixed-window causal models, shift-equivariant infinite sequences, and
structured next-token targets~\citep{yang2023parrot,takakura2023infinite,
sander2025nexttoken}.  Quantitative results cover fixed-length smooth targets
and bidirectional sliding-window maps under distributional
error~\citep{jiang2024rate,takakura2023infinite}.  Our result instead controls
one causal architecture uniformly over all resolutions and its continuous-time
limit.  For \(\beta>1\), our \(\beta\)-smooth target condition is stability
in a smooth-test metric on token--position laws, rather than generic
Fr\'echet \(C^\beta\) regularity. Logarithmic statistical rates also arise in functional
regression~\citep{meister2016optimal}; neural-operator analyses connect
approximation complexity to empirical-risk
guarantees~\citep{liu2024deep,kovachki2024data}. Our additional step controls
the capacity of bounded-weight causal transformers uniformly in context length.
Positional features remove permutation obstructions in
fixed-length universality~\citep{yun2020transformers}, although some relative
schemes remain limited~\citep{luo2022powerful}.  In our all-resolution setting,
normalized position also records chronology and grid resolution.
Appendix~\ref{app:position-needed} proves a sharper obstruction: even when
total length is known and position enters arbitrary normalized attention
scores, the normalized clock cannot be recovered if position is absent from
values, residual features, and pointwise maps.

\paragraph{Universality for causal sequence operators.}
Before transformers, recurrent networks and causal convolutions encoded
non-anticipativity~\citep{hochreiter1997lstm,hanson2019tcn}.
Their approximation theory either fixes the sequence length or
horizon~\citep{funahashi1993dynamical,song2023minimal}, or obtains time-uniform
guarantees through fading memory~\citep{grigoryeva2018echo,gonon2021fading}.
Rates are available for temporal convolutions and random
reservoirs~\citep{hanson2019tcn,gonon2023random}, while recent causal
neural-operator and temporal-integral constructions treat continuous paths on
prescribed spans~\citep{acciaio2024geometric,galimberti2026causal,
cuchiero2026global}.  Our normalized-horizon setting may retain the entire
prefix: compactness comes from regularity across sampling resolutions, and one
standard masked-attention architecture provides uniform approximation and,
under additional target regularity, parameter rates.

\paragraph{Contributions and organization.}
Section~\ref{sec:framework} introduces an all-resolution causal setting in
which token grids share an arbitrary admissible modulus and finite targets
need only converge asymptotically, rather than obey exact projective
identities, to a continuous path operator.  Our main result,
Theorem~\ref{thm:discrete-universality}, shows that one shallow masked
transformer---one attention block and one shared pointwise ReLU
network---approximates every target in such a family uniformly over sequence
length, input, and prediction index; its parameters may depend on the target
and accuracy, but not on length. The converse characterizes their uniform
closure (Corollary~\ref{cor:uniform-closure-characterization}). The completed-space proof gives, with
the same parameters, temporal-attention universality for the continuous path
limit (Theorem~\ref{thm:path-universality-main}). For normalized
\(\beta\)-smooth teachers, Theorem~\ref{thm:generalization-main} gives
root mean-square prediction error
\(O((\log\log N/\log N)^{\beta/(d+2)})\) at fixed confidence from \(N\)
iid variable-length sequences, using clipped exact empirical risk
minimization over suitably sized bounded-weight transformers. This
same-distribution bound relies on quantitative uniform approximation
(Appendix~\ref{app:quantitative-rates}) and a capacity estimate independent
of context length (Appendix~\ref{app:proof-generalization}). Finally,
Section~\ref{sec:numerics} finds dataset-dependent finite-scale increment
decay in physical signals and continuous pretrained content features, while
text input embeddings provide a contrasting case. Native and
dense patch sampling, shuffled controls, and refinement diagnostics delimit
this empirical support for the H\"older input model; the detailed proofs and
protocol are collected in the appendices.
The companion computational toolbox, figure-reproduction scripts, and
illustrative Jupyter notebooks are hosted at
\url{https://github.com/gpeyre/transformers-universality}.

\section{All-resolution causal universality}
\label{sec:framework}

We ask whether one causal transformer can approximate a family of prediction rules uniformly across all context lengths. A shared architecture, a common input modulus, and asymptotic cross-resolution compatibility make this possible. We first establish discrete and continuous-time universality, then derive a learning guarantee under additional target regularity.

Fix dimensions \(d,d'\geq1\) and \([n]\eqdef\{1,\ldots,n\}\). Write \(z=(z_i)_{i=1}^n\) for discrete inputs and \(x=(x(t))_{t\in[0,1]}\) for continuous inputs. We use \(\norm{\cdot}\) for Euclidean norms, \(\abs{\cdot}\) for absolute values, and \(\norm{\cdot}_\infty\) for uniform or maximum norms.

\subsection{Discrete causal transformer}
\label{sec:architecture-v2}

The architecture is defined once on the space of all nonempty finite
sequences. Its dimensions, weights, and biases are independent of sequence
length, so the same formulas apply for every \(n\geq1\).

\paragraph{A common domain across lengths.} For each \(p\geq1\), the disjoint union
\begin{equation}
\label{eq:finite-sequence-space-v2}
\mathsf{Seq}_p
\eqdef
\bigsqcup_{n\geq1}(\R^p)^n
\end{equation}
collects all nonempty finite sequences of \(p\)-dimensional vectors.

\paragraph{Residual masked multi-head attention.} Hidden states \(u_i\in\R^p\) are internal representations of input tokens \(z_i\in\R^d\). Each head summarizes the current prefix; the residual connection adds \(u_i\). For \(u\in\mathsf{Seq}_p\) of length \(n\) and \(i\in[n]\), define
\begin{equation}
\label{eq:finite-attention-v2}
\MAtt_\theta(u)_i
\eqdef u_i+
\sum_{r=1}^{H}W_r
\frac{
\sum_{j=1}^{i}
e^{\ip{Q_ru_i}{K_ru_j}}V_ru_j
}{
\sum_{j=1}^{i}e^{\ip{Q_ru_i}{K_ru_j}}
}.
\end{equation}
Equation~\eqref{eq:finite-attention-v2} therefore defines a single operator \(\MAtt_\theta:\mathsf{Seq}_p\to\mathsf{Seq}_p\). On each component, the input length determines the summation range, while the expression and parameters remain unchanged.
% We use the same convention for pointwise maps and finite transformers;
For head \(r\), if the query--key dimension is \(q_r\) and the value
dimension is \(v_r\), the matrices in \eqref{eq:finite-attention-v2} have
dimensions
\begin{equation}
\label{eq:attention-matrix-dimensions-v2}
 Q_r,K_r\in\R^{q_r\times p},
 \qquad
 V_r\in\R^{v_r\times p},
 \qquad
 W_r\in\R^{p\times v_r}.
\end{equation}
Here \(\theta\eqdef((Q_r,K_r,V_r,W_r))_{r=1}^H\). Thus the inner product in the score is taken in \(\R^{q_r}\), each normalized head value lies in \(\R^{v_r}\), and the output again lies in \(\R^p\).

We use \(\infty\) to distinguish continuous-time realizations. Along any grid refinement whose hidden-state interpolants converge uniformly, masked softmax averages converge to causal temporal integrals. For a continuous hidden path \(u:[0,1]\to\R^p\), the limiting attention block is
\begin{equation} \label{eq:limattblk}
\MAtt_\theta^\infty(u)(t)
\eqdef u(t)+\sum_{r=1}^H W_r
\begin{cases}
\displaystyle
\frac{\int_0^t e^{\ip{Q_ru(t)}{K_ru(s)}}V_ru(s)\,\dd s}
{\int_0^t e^{\ip{Q_ru(t)}{K_ru(s)}}\,\dd s},&t>0,\\[2ex]
V_ru(0),&t=0.
\end{cases}
\end{equation}

\paragraph{Shared pointwise networks.} A shared tokenwise affine layer applies one affine map at every token and every length:
\[
 G_{A,b}:\mathsf{Seq}_p\to\mathsf{Seq}_q,
 \qquad
 (G_{A,b}(u))_i\eqdef Au_i+b,
 \qquad
 A\in\R^{q\times p},\quad b\in\R^q.
\]
Its continuous-time realization applies the same affine map at every time:
\(
G_{A,b,\infty}(u)(t)\eqdef Au(t)+b
\).
A shared ReLU network applies the same finite-dimensional map to every token. For input dimension \(p\), output dimension \(q\), and hidden width \(M\geq1\), a one-hidden-layer pointwise ReLU network is the map
\begin{equation}
\label{eq:pointwise-mlp-v2}
 \MLP_\eta(u)
 \eqdef
 A_2\rho(A_1u+b_1)+b_2,
 \qquad u\in\R^p,
\end{equation}
where \(\rho\) is the ReLU activation function acting coordinatewise and
\begin{equation}
\label{eq:pointwise-mlp-dimensions-v2}
 A_1\in\R^{M\times p},\quad b_1\in\R^M,
 \qquad
 A_2\in\R^{q\times M},\quad b_2\in\R^q;
 \qquad
 \eta\eqdef(A_1,b_1,A_2,b_2).
\end{equation}
We use the same notation for its pointwise action on sequences:
\(\MLP_\eta:\mathsf{Seq}_p\to\mathsf{Seq}_q\), with
\((\MLP_\eta(u))_i\eqdef\MLP_\eta(u_i)\).

A finite causal transformer is any finite composition of these masked-attention blocks, shared affine layers, and pointwise ReLU networks with compatible widths. Its finite parameter list is shared across lengths.

\paragraph{The shallow universal architecture.} On regular input classes and continuously extendable target families introduced below, one masked-attention block followed by one pointwise MLP suffices for qualitative universality when width may grow with the target and accuracy. Given an integer \(H\geq1\), we set \(p_H\eqdef d+H+2\) and define the fixed tokenwise lift, \(\mathcal E_H:\mathsf{Seq}_{d+1}\to \mathsf{Seq}_{p_H}\), for \(y=(y_i)_{i=1}^n\in(\R^{d+1})^n\) by
\begin{equation}
\label{eq:theorem-lift-v2}
 \mathcal E_H(y)
 \eqdef
 \bigl((y_i,1,0_H)\bigr)_{i=1}^n
 \in(\R^{p_H})^n.
\end{equation}
The attention input is \(u_i\eqdef\mathcal E_H(y)_i\). The lift contains no learned parameters and acts identically at every token.

All \(H\) attention heads are taken to be scalar ($q_r=v_r=1$), so
\begin{equation} \label{eq:theorem-head-dimensions-v2}
 Q_r,K_r,V_r\in\R^{1\times p_H},
 \qquad W_r\in\R^{p_H\times1},
 \qquad r\in[H].
\end{equation}
At each token,
% the first \(d+1\) residual coordinates retain the token and normalized position, the next is a constant coordinate, and
the final \(H\) coordinates are ``work'' coordinates. In the construction, the \(W_r\)'s have disjoint ranges, so head \(r\) stores one scalar probe in work coordinate \(r\). For an MLP hidden width \(M\geq1\) and output dimension \(q=d'\) in \eqref{eq:pointwise-mlp-dimensions-v2}, we define
\begin{equation} \label{eq:shallow-transformer-v2}
 T_\Theta
 \eqdef
 \MLP_\eta\circ\MAtt_\theta\circ \mathcal E_H,
 \qquad
 \Theta\eqdef(\theta,\eta).
\end{equation}
This defines \(T_\Theta:\mathsf{Seq}_{d+1}\to\mathsf{Seq}_{d'}\) by the same composition at every length. In particular, \(A_1\in\R^{M\times p_H}\) and \(A_2\in\R^{d'\times M}\). The integers \(H,M\), all matrices, and all biases are fixed across sequence lengths, but may depend on the target and accuracy.

\paragraph{Normalized position in the hidden state.} We append normalized position \(i/n\) to each input token, making it available through the residual connection. For \(z\in(\R^d)^n\), define the positional lift
\begin{equation} \label{eq:canonical-lift-v2}
 \phi_n(z)_i\eqdef(z_i,i/n)\in\R^{d+1}.
\end{equation}
% The lift is parameter-free, and all architectural widths and learned parameters remain independent of \(n\).
The transformer receives \(y\eqdef\phi_n(z)\). No separate length channel is needed: one uniform-attention head computes the prefix mean position, and the residual position gives
\(2i^{-1}\sum_{j=1}^i(j/n)-i/n=1/n\).
Thus the resolution is accessible even though the current position alone does not identify it. Appendix~\ref{app:position-needed} shows why position confined to normalized attention scores cannot replace this residual information.

\subsection{Discrete causal targets with a continuous path extension}
\label{sec:targets-v2}

Uniform universality across arbitrary context lengths rests on two joint hypotheses on the context and the target: a common modulus places all token grids in one compact all-resolution family, and the discrete causal targets converge to a single continuous path operator under grid refinement.

\paragraph{Regular token sequences of arbitrary length.} A common concave modulus controls every resolution and is preserved exactly by piecewise-affine interpolation.

\begin{definition}[Admissible common modulus]
\label{def:admissible-modulus}
An \emph{admissible common modulus} is a continuous, nondecreasing, concave function \(\omega:[0,1]\to[0,\infty)\) satisfying \(\omega(0)=0\).
\end{definition}

In particular, \(\omega(r)\to0\) as \(r\downarrow0\). The zero modulus is
allowed and gives constant token paths, while \(\omega_{\alpha,L}(r)\eqdef Lr^\alpha\) is admissible for every \(0<\alpha\leq1\) and \(L>0\).

Let \(\Omega\eqdef\overline{B_{\R^d}(0,1)}\), and fix an admissible common modulus \(\omega\). For each length \(n\), we define
\begin{equation}
\label{eq:Xn-v2}
X_n^\omega
\eqdef
\left\{
z=(z_1,\ldots,z_n)\in\Omega^n:
\norm{z_j-z_i}\leq
\omega\!\left(\frac{\abs{j-i}}{n}\right),
\quad i,j\in[n]
\right\}.
\end{equation}
The estimate is imposed on every pair, not only on adjacent tokens.  This
distinction is essential for a general modulus: adjacent control gives only
\(\norm{z_j-z_i}\leq\abs{j-i}\omega(1/n)\), which need not imply the required
\(\omega(\abs{j-i}/n)\) estimate.  For the linear modulus
\(\omega(r)=Lr\), the all-pairs and adjacent conditions are equivalent.
This models regular input vectors, not arbitrary sequences of token embeddings.

\paragraph{Continuous path completion.} Sampling and interpolation identify the discrete token grids with approximations of a single compact class of continuous paths. Define
\begin{equation} \label{eq:Xinfty-v2}
 X_\infty^\omega
 \eqdef
 \left\{
 x\in C([0,1];\Omega):
 \norm{x(t)-x(s)}\leq\omega(\abs{t-s}),
 \quad s,t\in[0,1]
 \right\}.
\end{equation}
For a path \(x:[0,1]\to\R^d\), define its samples by \(\mathcal S_nx\eqdef(x(j/n))_{j=1}^n\). For a sequence \(z\in(\R^d)^n\), let the interpolant \(\In_nz\) be constant and equal to \(z_1\) on \([0,1/n]\), affine between consecutive grid values, and satisfy
\begin{equation} \label{eq:interpolation-v2}
 (\In_nz)(i/n)=z_i,
 \qquad i\in[n].
\end{equation}
The same notation applies to output interpolation. Lemma~\ref{lem:modulus-reconstruction} proves, for every \(x\in X_\infty^\omega\),
\begin{equation} \label{eq:modulus-reconstruction-v2}
 \In_nX_n^\omega\subset X_\infty^\omega,
 \qquad
 \mathcal S_n(X_\infty^\omega)=X_n^\omega,
 \qquad
 \norm{\In_n\mathcal S_nx-x}_\infty\leq\omega(1/n).
\end{equation}

\paragraph{Continuously extendable causal targets.} The finite target maps need not agree exactly across context lengths; we require only that they be causal at each length and converge uniformly to one continuous path operator.

\begin{definition}[Continuously extendable discrete causal family]
\label{def:compatible-family}
A discrete family \((F_n^\star)_{n\geq1}\) is a \emph{continuously extendable causal family} if each map \(F_n^\star:X_n^\omega\to(\R^{d'})^n\) is continuous and prefix-causal, meaning that, for all \(z,z'\in X_n^\omega\) and \(i\in[n]\), \(z_{1:i}=z'_{1:i}\) implies \(F_n^\star(z)_i=F_n^\star(z')_i\), and there exists a continuous map \(F_\infty^\star:X_\infty^\omega\to C([0,1];\R^{d'})\) such that
\begin{equation}
\label{eq:continuous-extension-v3}
\sup_{z\in X_n^\omega}
\norm{\In_n(F_n^\star(z))-F_\infty^\star(\In_nz)}_\infty
\longrightarrow0
\qquad(n\to\infty).
\end{equation}
Such an \(F_\infty^\star\) is a \emph{continuous path extension} of the discrete family.
\end{definition}

Controlled dynamical systems provide a concrete example.

\begin{example}[Controlled ODEs]
\label{ex:control-ode-family-main}
Fix \(y_0\in\R^{d'}\) and a continuous vector field
\(b:\R^{d'}\times\Omega\to\R^{d'}\), globally Lipschitz in its state
argument uniformly over controls. For \(x\in X_\infty^\omega\), let \(y_x\) solve
\[
 \dot y_x(t)=b(y_x(t),x(t)),\quad y_x(0)=y_0,
 \qquad
 F_\infty^\star(x)\eqdef y_x,\quad
 F_n^\star(z)_i\eqdef y_{\In_nz}(i/n).
\]
Thus the discrete map samples the exact state trajectory driven by the
reconstructed tokens. This family satisfies
Definition~\ref{def:compatible-family}: the extension error in
\eqref{eq:continuous-extension-v3} is at most \(M/n\), with \(M\) depending
only on \(b,y_0\). Appendix~\ref{app:control-odes} gives the construction,
uniform estimates, and proof.
\end{example}

Although causality is imposed only at finite resolutions, it passes to the extension: \(x|_{[0,t]}=x'|_{[0,t]}\) implies \(F_\infty^\star(x)(t)=F_\infty^\star(x')(t)\). Lemma~\ref{lem:target-compatibility} proves this fact and also shows that the extension is unique.

\paragraph{Compatibility with transformers across resolutions.} For every fixed parameter list \(\Theta\), the family induced by the transformer in \eqref{eq:shallow-transformer-v2} and the normalized positional lifts,
\(
 \mathcal T_{\Theta,n}(z)
 :=
 T_\Theta(\phi_n(z)),
 \ n\geq1,
\)
is continuously extendable and causal in the sense of Definition~\ref{def:compatible-family}. Corollary~\ref{cor:transformer-compatible-family} proves this for every fixed finite causal transformer.

\paragraph{Cross-resolution compatibility is asymptotic.} No separate output modulus is required: the compact image of the continuous extension supplies the needed equicontinuity. We write
\(
\mathbf F^\star\eqdef((F_n^\star)_{n\geq1},F_\infty^\star)
\)
for the completed family. Lemma~\ref{lem:target-compatibility} shows that \eqref{eq:continuous-extension-v3} is equivalent to
\(\In_{n_k}(F_{n_k}^\star(z^k))\to F_\infty^\star(x)\) uniformly whenever
\(n_k\to\infty\) and \(\In_{n_k}z^k\to x\) uniformly. It allows arbitrary continuous causal changes at finitely many lengths and imposes no exact identity either with \(F_\infty^\star\) or between two finite lengths. Such projective consistency would be unnatural here: index \(i\) represents time \(i/n\), so even the clock \(F_n^\star(z)_i\eqdef i/n\) violates a same-index identity across lengths.

\subsection{Universality theorems}
\label{sec:main-theorem-v2}

We establish universality for the shared discrete architecture and its continuous-time realization. Section~\ref{sec:rates-generalization} then turns target regularity into a statistical learning bound.

\paragraph{Universality for discrete tokens.} The first result is stated entirely at finite resolutions. Continuous paths appear only through the extension hypothesis; neither temporal attention nor the measure representation used in the proof enters the conclusion.

\begin{theorem}[All-resolution universality for discrete tokens]
\label{thm:discrete-universality}
Fix an admissible common modulus \(\omega\). Let \((F_n^\star)_{n\geq1}\) be a continuously extendable causal family on \((X_n^\omega)_{n\geq1}\) in the sense of Definition~\ref{def:compatible-family}. For every \(\varepsilon>0\), there exist integers \(H,M\geq1\) and a parameter list \(\Theta=(\theta,\eta)\), independent of \(n\), with residual width \(p_H=d+H+2\), scalar-head dimensions \(Q_r,K_r,V_r\in\R^{1\times p_H}\), \(W_r\in\R^{p_H\times1}\), and readout dimensions \(A_1\in\R^{M\times p_H}\), \(b_1\in\R^M\), \(A_2\in\R^{d'\times M}\), \(b_2\in\R^{d'}\), such that the transformer \(T_\Theta\) defined by \eqref{eq:theorem-lift-v2}-\eqref{eq:shallow-transformer-v2} satisfies
\begin{equation} \label{eq:discrete-universality-bound-v3}
 \sup_{n\geq1}\ \sup_{z\in X_n^\omega}\ \max_{i\in[n]}
\norm{T_\Theta(\phi_n(z))_i-F_n^\star(z)_i}
 <\varepsilon.
\end{equation}
Every attention matrix \(Q_r,K_r,V_r,W_r\) can be chosen with Euclidean-induced operator (spectral) norm at most one.
\end{theorem}

\begin{proof}[Sketch of proof]
The common modulus compactifies finite evaluation states by adding
continuous-path limits. Joint token--position prefix laws retain exactly
the causal information; causality and compatibility therefore make the
target a continuous function of these laws. Scalar attention probes are
log-Laplace derivatives and separate distinct laws. Stone--Weierstrass approximates the
target by a polynomial of finitely many probes: one attention block
computes them in parallel, and a shared ReLU readout approximates the
polynomial. Compactness makes the error uniform over all lengths.
Appendix~\ref{app:detailed-proof-guide} gives the detailed reading guide;
Proposition~\ref{prop:prefix-law} and
Lemmas~\ref{lem:laplace-probes}--\ref{lem:parallel-realization} supply the
main ingredients.
\end{proof}

The converse also holds: Corollary~\ref{cor:uniform-closure-characterization} proves that continuous extendability is necessary for uniform all-resolution approximation, even by arbitrary finite-depth transformers. Thus shallow and finite-depth architectures have the same uniform closure. Each approximating network is fixed across lengths; width, readout weights, and, in the larger class, depth may vary with accuracy. The construction keeps attention scores uniformly bounded, so increasingly sharp attention is unnecessary. However, position cannot be confined to scores: Appendix~\ref{app:position-needed} proves an error of at least \(1/2\) on the normalized clock in this case, even if total length is known. These are expressivity guarantees, not guarantees of learning or length-independent computation.

\medskip

\begin{remark}[Width versus depth]
\label{rem:width-versus-depth}
The parallel construction uses \(H\) scalar heads and residual width \(p_H=d+H+2\) to store \(H\) probes. Alternatively, the Stone--Weierstrass polynomial can be evaluated sequentially, reusing coordinates for the current probe, running monomial product, and accumulated sum. Lemma~\ref{lem:serialized-realization}, following \citet{furuya2025transformers}, gives residual width \(d+1+3d'\) and at most \(d'\) scalar heads per block. Depth and pointwise hidden widths may then depend on the target and accuracy, whereas residual width and head count remain fixed; none depends on sequence length. Appendix~\ref{app:quantitative-rates} analyzes only the shallow parallel construction.
\end{remark}

% The shallow construction evaluates \(H\) probes in parallel with \(H\) scalar heads and work coordinates, hence \(p_H=d+H+2\). These widths may depend on the target and accuracy, but not on \(n\). Lemma~\ref{lem:serialized-realization}, following \citet{furuya2025transformers}, instead reuses registers across depth with residual width \(d+1+3d'\) and at most \(d'\) scalar heads per block; depth and internal widths may then grow. Appendix~\ref{app:quantitative-rates} analyzes only the shallow parallel construction.

\paragraph{Continuous-time limit and universality for paths.}
Replacing masked sums by the temporal integrals in \eqref{eq:limattblk}
defines the same architecture on paths. Set
\(\phi_\infty(x)(t)\eqdef(x(t),t)\) and
\(T_{\Theta,\infty}\eqdef\MLP_\eta\circ\MAtt_\theta^\infty\circ\mathcal E_H\),
where \((\mathcal E_Hy)(t)\eqdef(y(t),1,0_H)\) and the readout acts pointwise.

\begin{theorem}[Continuous-time causal universality]
\label{thm:path-universality-main}
Fix an admissible common modulus \(\omega\). Let
\(F_\infty^\star:X_\infty^\omega\to C([0,1];\R^{d'})\) be continuous
in the uniform norm and causal:
\(x|_{[0,t]}=x'|_{[0,t]}\) implies
\(F_\infty^\star(x)(t)=F_\infty^\star(x')(t)\).
For every \(\varepsilon>0\), there exist integers \(H,M\geq1\) and parameters
\(\Theta\), with the dimensions and attention-matrix norm bounds of
Theorem~\ref{thm:discrete-universality}, such that
\begin{equation}
\label{eq:path-universality-main}
 \sup_{x\in X_\infty^\omega}
 \norm{T_{\Theta,\infty}(\phi_\infty(x))-F_\infty^\star(x)}_\infty
 <\varepsilon.
\end{equation}
If \(F_\infty^\star\) extends a discrete family as in
Definition~\ref{def:compatible-family}, the same parameters can also satisfy
\eqref{eq:discrete-universality-bound-v3}.
\end{theorem}

Thus a continuous causal operator \(F_\infty^\star\) can be approximated without
specifying any discrete target family. Temporal attention is also the
refinement limit: for fixed \(\Theta\), if \(n_k\to\infty\) and
\(\In_{n_k}z^k\to x\) uniformly with \(z^k\in X_{n_k}^\omega\), then
\(\In_{n_k}T_\Theta(\phi_{n_k}(z^k))\to
T_{\Theta,\infty}(\phi_\infty(x))\) uniformly.
Appendix~\ref{app:continuous-time-universality} details the realization and
proves this convergence (Proposition~\ref{prop:temporal-closure}); the common
universality proof is in Appendix~\ref{app:proof-joint}.

\subsection{Quantitative rates and generalization bound}
\label{sec:rates-generalization}

To obtain approximation and learning rates uniform over context lengths,
we impose \(\beta\)-smooth regularity on the common discrete--path target
(Appendix~\ref{app:learning-scope}).

\paragraph{Quantitative rates for \(\beta\)-smooth targets.}
We impose regularity after factoring the targets through measures of causal
histories (Appendix~\ref{sec:quantitative-regularity}).
For \(E\eqdef\Omega\times[0,1]\), let \(\cM_\omega\subset\cP(E)\)
be the weak closure of the prefix laws
\(\mu_{n,z,i}\eqdef i^{-1}\sum_{j\leq i}\delta_{(z_j,j/n)}\), for
\(z\in X_n^\omega\). The endpoint map
\(\mathfrak e(\mu_{n,z,i})\eqdef(z_i,i/n)\) extends continuously to
\(\cM_\omega\). Continuously extendable families factor as
\(F_n^\star(z)_i=f^\star(\mu_{n,z,i})\), with a unique continuous
\(f^\star:\cM_\omega\to\R^{d'}\)
(Proposition~\ref{prop:prefix-law}).
For \(\beta\geq1\), the law distance below is a H\"older integral probability
metric, closely related to smooth Wasserstein
metrics~\citep{block2022intrinsic,gaunt2023bounding}.

\begin{definition}[\(\beta\)-smooth regularity class]
\label{def:attention-teacher-class}
For \(\beta>0\), put \(\beta_-\eqdef\min\{\beta,1\}\),
\(\beta_+\eqdef\max\{\beta,1\}\), and
\(\delta_e(\mu,\nu)\eqdef\norm{\mathfrak e(\mu)-\mathfrak e(\nu)}_\infty\).
Define on \(\cM_\omega\), with smoothness index \(\beta\),
\begin{equation}
\label{eq:teacher-regularity-main}
 W_{1,\beta}(\mu,\nu)
 \eqdef\Bigl[\sup_{\norm{\varphi}_{C^{\beta_+}(E)}\leq1}
 \int_E\varphi\,\dd(\mu-\nu)\Bigr]^{\beta_-},
 \qquad
 \Delta_\beta\eqdef W_{1,\beta}+\delta_e^{\beta_-}.
\end{equation}
A continuously extendable family \(\mathbf F^\star\) is
\emph{\(\beta\)-smooth} if \(f^\star\) is Lipschitz on
\((\cM_\omega,\Delta_\beta)\). Write
\(R_{\beta;\omega}(\mathbf F^\star)\) for its optimal Lipschitz constant.
The normalized class \(\mathcal F^\beta(\omega)\) is the set of such families
\(\mathbf F^\star\) with \(R_{\beta;\omega}(\mathbf F^\star)\leq1\)
and \(\norm{f^\star}_\infty\leq1\) on \(\cM_\omega\).
As defined in Appendix~\ref{sec:quantitative-regularity},
\(C^\sigma(E)\) uses the full isotropic restriction norm from
\([-1,1]^d\times[0,1]\), with \(C^m=C^{m-1,1}\) at integer orders.
At \(\beta=1\), \(W_{1,1}\) is equivalent to Wasserstein-1 (equal with
Lipschitz-seminorm normalization; Lemma~\ref{lem:unified-history-metric}).
\end{definition}

\begin{example}[Regular controlled integrators]
\label{ex:regular-control-integrator}
Define \(F_\infty^\star(x)\eqdef y_x\) as in
Example~\ref{ex:control-ode-family-main} with \(b(y,u)=g(u)\), where
\(g\in C^\beta(\Omega;\R^{d'})\). The targets
\(F_n^\star(z)_i\eqdef y_0+n^{-1}\sum_{j=1}^i g(z_j)\) use the constant rate
\(g(z_j)\) on each interval \(((j-1)/n,j/n]\). These right-endpoint Riemann sums
extend to the same continuous map. Their history factor
\(f^\star(\mu)=y_0+t\int g(u)\,\dd\mu(u,s)\), with
\(t\eqdef\mathfrak e(\mu)_{d+1}\), satisfies
\(R_{\beta;\omega}(\mathbf F^\star)\leq C_{d,d',\beta}\norm{g}_{C^\beta}\),
with an \(\omega\)-independent upper bound.
Unlike Example~\ref{ex:control-ode-family-main}, these maps do not integrate
the piecewise-affine control exactly; Appendix~\ref{app:regular-control-integrator}
details the distinction and proves the bound.
\end{example}

Our regularity condition is H\"older stability for \(0<\beta\leq1\), and
stability through smooth token--time averages with Lipschitz endpoint
dependence for \(\beta>1\), not generic Fr\'echet smoothness
(Remark~\ref{rem:quantitative-statistic-examples}).
Theorem~\ref{thm:quantitative-rates} (Appendix~\ref{app:quantitative-rates})
gives uniform discrete--path error
\(O(R_{\beta;\omega}(\log\log p/\log p)^{\beta/(d+2)})\) with \(p\)
parameters. For H\"older inputs, finite-description lower bounds are also
logarithmically slow (Appendix~\ref{sec:rate-optimality}); matching
approximation and generalization lower bounds remain open.

\paragraph{Generalization bound.}
We fit the target family by squared-loss empirical risk minimization (ERM).
For \(H\) scalar heads and an \(M\)-unit readout, counting all matrix and
bias entries, including zeros, gives
\(\operatorname{par}(T_\Theta)\eqdef 4Hp_H+(p_H+d'+1)M+d'\).
The lift is parameter-free.
For a deterministic integer budget \(p\), let
\(\mathfrak T_p\) contain all shallow transformers
\eqref{eq:shallow-transformer-v2}, over integer \(H,M\geq1\), with
\begin{equation}
\label{eq:learning-class-main}
 \operatorname{par}(T_\Theta)\leq p,
 \qquad \norm{\theta}_\infty\leq1,
 \qquad \norm{\eta}_\infty\leq p+1,
\end{equation}
where the parameter norms bound every matrix entry and bias in their respective
blocks. For training sequences \((n_k,Z^k,Y^k)_{k=1}^N\), with
\(Z^k\in X_{n_k}^\omega\) and \(Y^k\in(\R^{d'})^{n_k}\), use the
Frobenius norm over tokens and define
\begin{equation}
\label{eq:learning-erm-main}
 \widehat{\mathcal L}_N(T_\Theta)
 \eqdef\frac1N\sum_{k=1}^N
 \frac{\norm{T_\Theta(\phi_{n_k}(Z^k))-Y^k}_F^{2}}{n_k},
 \quad
 \widehat{\mathcal L}_N(T_{\widehat\Theta_N})
 \leq\inf_{T_\Theta\in\mathfrak T_p}\widehat{\mathcal L}_N(T_\Theta)+\tau_{\mathrm{opt}}.
\end{equation}
Here \(\widehat\Theta_N\) measurably selects a network in \(\mathfrak T_p\)
from the data, within \(\tau_{\mathrm{opt}}\geq0\) of the minimum empirical
loss (zero for exact ERM). We do not bound this optimization gap for SGD
(Appendix~\ref{app:proof-generalization}). Define the predicted family
\(\widehat F\eqdef(\widehat F_n)_{n\geq1}\) by
\(\widehat F_n(z)_i\eqdef\operatorname{clip}(T_{\widehat\Theta_N}(\phi_n(z))_i)\), where
\(\operatorname{clip}(v)\eqdef v/\max\{1,\norm{v}\}\) projects onto the output
unit ball, giving a budget-independent loss bound for concentration with
bounded labels. Conditional on the training data, the prediction
risk on an independent test sequence from the same law is
\begin{equation}
\label{eq:learning-risk-main}
 \mathcal R_{\mathrm{pred}}(\widehat F)
 \eqdef\mathbb E_{(n,Z)}
 \left[\frac1n\sum_{i=1}^n
 \norm{\widehat F_n(Z)_i-F_n^\star(Z)_i}^{2}\right].
\end{equation}

\begin{theorem}[Generalization from variable-length sequences]
\label{thm:generalization-main}
Fix an admissible common modulus \(\omega\), \(\beta>0\), and a teacher
\(\mathbf F^\star\in\mathcal F^\beta(\omega)\).
Let \((n_k,Z^k,Y^k)_{k=1}^N\) be iid copies of
\((n,Z,Y)\), with \(Z\in X_n^\omega\), \(\norm{Y_i}\leq1\), and
\(\mathbb E[Y_i\mid n,Z_{1:i}]=F_n^\star(Z)_i\), for each \(i\) almost surely
on \(\{n\geq i\}\).
There are constants \(C\geq1\) and \(p_0\geq16\), depending only on
\((d,d',\beta)\), such that
for every deterministic integer \(p\geq p_0\), \(N\geq2\), and \(0<\delta<1\),
with probability at least \(1-\delta\), every estimator satisfying
\eqref{eq:learning-erm-main} obeys
\begin{equation}
\label{eq:generalization-bound-main}
 \mathcal R_{\mathrm{pred}}(\widehat F)
 \leq C\left[
 \left(\frac{\log\log p}{\log p}\right)^{2\beta/(d+2)}
 +\frac{p\log(CpN)+\log(2/\delta)}{N}\right]
 +2\tau_{\mathrm{opt}}.
\end{equation}
The constants are independent of the input modulus, the sampling distribution,
and the context lengths; no upper bound or moment assumption on \(n\) is needed.
\end{theorem}

\begin{proof}[Sketch of proof]
Corollary~\ref{cor:quantitative-finite-precision} supplies a bounded-weight
comparator with the stated approximation error. Normalized-attention
stability gives the clipped class a uniform cover of logarithmic size
\(O(p\log(p/\epsilon))\), independent of context length
(Lemma~\ref{lem:gen-parameter-cover}). Sequence-averaged excess losses have
variance controlled by prediction risk, so Bernstein concentration gives the
\(N^{-1}\) statistical term. Projection decreases empirical loss; comparing
the clipped raw ERM with the comparator then adds
\(2\tau_{\mathrm{opt}}\). Appendix~\ref{app:proof-generalization} gives the proof.
\end{proof}

The bound separates approximation from estimation. For exact ERM,
\(p\eqdef\lfloor\sqrt N\rfloor\) gives root mean-square rate
\(O((\log\log N/\log N)^{\beta/(d+2)})\) at fixed confidence and large \(N\).
Here \(N\) counts independent sequences, allowing dependent tokens.

\section{Empirical diagnostics of H\"older path geometry}
\label{sec:numerics}

The common-modulus assumption links sequences sampled at different resolutions
of a fixed horizon. We examine its H\"older specialization through the decay
of worst increments in measured signals and pretrained representations.
The experiments reveal a positive but representation-dependent finite-scale
regime; they do not certify one common H\"older ball at all resolutions.

\paragraph{A common scale diagnostic.}
Classical multiscale methods assess roughness through the log--log scaling of
increments or wavelet coefficients~\citep{gneiting2012fractal,abry2015irregularities}.
We use a maximum over positions, rather than an average, to probe the
worst-case increments controlled by our uniform H\"older hypothesis.
For vectors \(z_0,\ldots,z_{N-1}\in\R^d\) on a uniform time grid of spacing
\(\Delta\), define, for \(1\le r<N\) and \(h_r\eqdef r\Delta\),
\begin{equation}
 M_\infty(r)\eqdef
 \max_{0\le k<N-r}d^{-1/2}\norm{z_{k+r}-z_k}_2.
 \label{eq:numeric-uniform-increment}
\end{equation}
We evaluate dyadic lags with \(h_r\le1/4\) to limit finite-window effects,
and fit each window on the finer band:
\begin{equation}
 \log M_\infty(r)=\log C+\widehat\alpha_\infty\log h_r+\varepsilon_r,
 \qquad h_r\le1/32.
 \label{eq:numeric-uniform-fit}
\end{equation}
We use \(\Delta\eqdef1/(N-1)\), or \(\Delta\eqdef s/(N_{\rm raw}-1)\) for patches
extracted every \(s\) raw measurements to preserve physical time.
We report mean slopes and sample SD across windows, require at least four
scales per primary fit, and retain all declared alternative bands.
An \(L\)-H\"older-\(\alpha\) path satisfies
\(M_\infty(r)\le Lh_r^\alpha/\sqrt d\); a fitted slope alone does not
establish this bound. Thus \(C\) is a fitted prefactor and
\(\widehat\alpha_\infty\) a finite-scale diagnostic, not a certified
asymptotic uniform exponent.

\paragraph{Physical signals and continuous content.}
Physical recordings admit fixed-horizon refinement; raw mean slopes span
\(0.041\)--\(0.389\) (Table~\ref{tab:uniform-holder-results}). On six scalar
channels, we apply Chronos-Bolt-tiny's frozen frontend~\citep{chronosbolt2024}
to identical 16-sample patches, with per-window normalization, at strides
1 (dense) and 16 (native). Native complete-patch outputs are a subset of
dense outputs; this tests neither forecasting nor causal processing.
All 29 dense windows have positive slopes, largely lost after shuffling
(Figure~\ref{fig:continuous-holder-main}, Table~\ref{tab:continuous-holder-main}).
Weather and oil temperature retain positive native-stride slopes;
appliances and ETT load are weaker. Raw delay patches often have larger
slopes, so learning alone does not explain the geometry. Fixed Lipschitz
maps preserve H\"older bounds at fixed physical patch offsets, but overlap,
scale selection, jumps, and mixed nested-grid seminorm trends qualify these
findings (Appendices~\ref{app:continuous-holder-protocol}
and~\ref{app:holder-calibration}). SD is descriptive; the ETT channels share one source.

\paragraph{Text as a contrasting input model.}
Across seven text sources, BigBird input embeddings~\citep{zaheer2020bigbird}
with interpolated positions have nearly flat worst-increment curves at
\(N=16{,}384\): mean slopes \(0.011\)--\(0.013\)
(Appendix~\ref{app:text-corpus-screen}). Appending words is not refinement
of a fixed signal, and persistent order-one adjacent jumps would preclude
a common H\"older bound. These finite-scale results do not support that
hypothesis for this pipeline, but prove neither asymptotic failure nor
impossibility for other representations. At native context lengths,
BigBird's mean slope rises from \(0.013\)--\(0.014\) at input to
\(0.028\)--\(0.033\) in intermediate blocks; this modest gain is
non-monotone, also occurs after input shuffling, and does not persist to
the final block. Appendix~\ref{app:text-depth} gives layer-wise curves
and a causal DistilGPT2 comparison. Regularity therefore remains an
assumption on the represented input family, not a conclusion of these fits.

\begin{figure}[!t]
 \centering
 {\scriptsize
 \textcolor[HTML]{626A73}{\rule[.45ex]{1.2em}{.6pt}} Raw signal\quad
 \textcolor[HTML]{A5A8AD}{\rule[.45ex]{1.2em}{.6pt}} Raw patch\quad
 \textcolor[HTML]{087F8C}{\rule[.45ex]{1.2em}{.6pt}} Bolt dense\quad
 \textcolor[HTML]{D77F29}{\rule[.45ex]{1.2em}{.6pt}} Shuffled, dense\par}
 {\scriptsize Text (d):\quad
 \textcolor[HTML]{3569AD}{\rule[.45ex]{1em}{.6pt}} WikiText\quad
 \textcolor[HTML]{E69F00}{\rule[.45ex]{1em}{.6pt}} AG News\quad
 \textcolor[HTML]{8B5BB5}{\rule[.45ex]{1em}{.6pt}} IMDb\quad
 \textcolor[HTML]{737373}{\rule[.45ex]{1em}{.6pt}} \LaTeX\quad
 \textcolor[HTML]{D55E00}{\rule[.45ex]{1em}{.6pt}} Shakespeare\quad
 \textcolor[HTML]{009E73}{\rule[.45ex]{1em}{.6pt}} KJV\quad
 \textcolor[HTML]{CC79A7}{\rule[.45ex]{1em}{.6pt}} Python\par}
 \smallskip
 \begin{minipage}[c]{.025\textwidth}
  \centering\rotatebox{90}{\scriptsize $M_\infty(r)/M_\infty(1)$}
 \end{minipage}\hfill
 \begin{minipage}[c]{.96\textwidth}
 \centering
 \setlength{\tabcolsep}{1pt}
 \begin{tabular}{@{}cccc@{}}
  {\footnotesize (a) Jena weather} & {\footnotesize (b) Beijing PM$_{2.5}$} &
  {\footnotesize (c) Appliances} & {\footnotesize (d) BigBird text}\\
  \includegraphics[width=.242\linewidth]{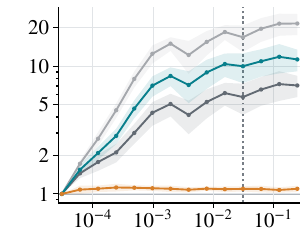} &
  \includegraphics[width=.242\linewidth]{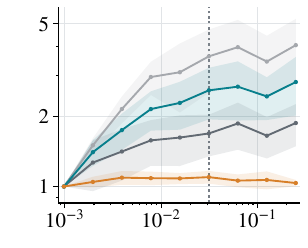} &
  \includegraphics[width=.242\linewidth]{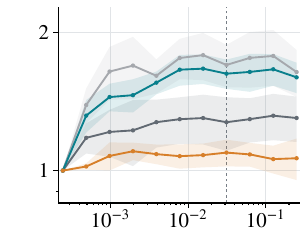} &
  \includegraphics[width=.242\linewidth]{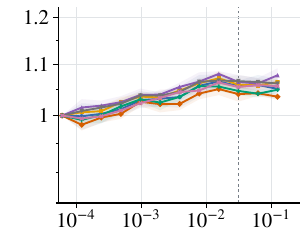}
 \end{tabular}
 \end{minipage}
 \caption{\textbf{Representation-dependent finite-scale regularity.}
 Worst-increment curves against normalized lag $h$ (horizontal axis): physical
 time in (a)--(c), token position in (d). Curves are normalized per window by
 the finest-lag increment before averaging; shading is one sample SD.
 This preserves each window's slope
 and fixes the first point at one with zero SD. Raw-patch and shuffled
 controls probe patch geometry and temporal organization; \LaTeX{} and
 Python are technical text controls. Dashed lines mark the fit cutoff;
 all panels have lower ordinate $0.85$.}
 \label{fig:continuous-holder-main}
\end{figure}

\begin{table}[!htbp]
 \centering
 \caption{\textbf{Pretrained content comparison.}
 Fine-band \(\widehat\alpha_\infty\), mean \(\pm\) sample SD across \(m\)
 windows. Physical rows use Bolt content; WikiText-103 uses BigBird input.
 Shuffling precedes patching. Native dashes denote insufficient scales;
 the text shuffled control is unmeasured.}
 \label{tab:continuous-holder-main}
 \footnotesize
 \setlength{\tabcolsep}{1.1pt}
 \begin{minipage}[t]{.495\textwidth}
 \centering
 \begin{tabular}{@{}lrccc@{}}
 \toprule
 Domain & \(m\) & Input/dense & Native & Shuffled\\
 \midrule
 Jena weather & 8 & $.335{\pm}.037$ & $.167{\pm}.040$ & $.005{\pm}.003$\\
 Beijing PM$_{2.5}$ & 6 & $.257{\pm}.085$ & -- & $.023{\pm}.016$\\
 Appliances & 4 & $.087{\pm}.010$ & $.020{\pm}.017$ & $.015{\pm}.012$\\
 Road traffic & 3 & $.303{\pm}.009$ & -- & $.035{\pm}.011$\\
 \bottomrule
 \end{tabular}
 \end{minipage}\hfill
 \begin{minipage}[t]{.495\textwidth}
 \centering
 \begin{tabular}{@{}lrccc@{}}
 \toprule
 Domain & \(m\) & Input/dense & Native & Shuffled\\
 \midrule
 Oil temp. & 4 & $.253{\pm}.055$ & $.175{\pm}.041$ & $.023{\pm}.004$\\
 ETT load & 4 & $.181{\pm}.034$ & $.043{\pm}.055$ & $.026{\pm}.014$\\
 \midrule
 WikiText-103 & 4 & $.0122{\pm}.0012$ & n/a & --\\
 \bottomrule
 \end{tabular}
 \end{minipage}
\end{table}
\FloatBarrier

\section{Conclusion}
\label{sec:conclusion}

Continuously extendable causal families are exactly those uniformly approximable across all resolutions by a length-independent masked transformer, whose parameters also approximate the continuous path limit. The proof represents finite prefixes and continuous histories by position-tagged laws on one compact space, then separates them with attention probes. The \(\beta\)-smooth target condition yields parameter-error rates and, on normalized target classes, finite-precision bounds. Bounded-weight empirical risk minimization gives logarithmic prediction rates from independent labeled sequences without a maximum-length factor; matching minimax exponents remain open. The positional obstruction confirms that masking alone cannot recover the normalized clock. Empirically, continuous pretrained features show dataset-dependent finite-scale regularity, whereas text input embeddings have nearly flat worst increments. Sampling and jump controls qualify these diagnostics, which do not certify a common all-resolution H\"older bound. Future work includes geometry-adaptive rates, efficient width--depth tradeoffs, and learning guarantees for dependent samples and computationally feasible optimization.

\section*{Acknowledgments}

Takashi Furuya acknowledges funding from JSPS KAKENHI (grants JP24K16949
and 25H01453), JST CREST (JPMJCR24Q5), and JST ASPIRE (JPMJAP2329).
Maarten V.\ de Hoop acknowledges funding from the Department of Energy's
BES program (grant DE-SC0020345), Oxy, and the Simons Foundation's MATH~+~X Program.
Gabriel Peyr\'e acknowledges funding from the European Research Council
(ERC) through project WOLF and
from the French government's France 2030 program through the Agence
Nationale de la Recherche (ANR-23-IACL-0008, PRAIRIE-PSAI).

\bibliographystyle{plainnat}
\bibliography{ref}

\begin{thebibliography}{60}
\providecommand{\natexlab}[1]{#1}
\providecommand{\url}[1]{\texttt{#1}}
\expandafter\ifx\csname urlstyle\endcsname\relax
  \providecommand{\doi}[1]{doi: #1}\else
  \providecommand{\doi}{doi: \begingroup \urlstyle{rm}\Url}\fi

\bibitem[Abry et~al.(2015)Abry, Jaffard, and Wendt]{abry2015irregularities}
Patrice Abry, St{\'e}phane Jaffard, and Herwig Wendt.
\newblock Irregularities and scaling in signal and image processing:
  Multifractal analysis.
\newblock In Michael Frame and Nathan Cohen, editors, \emph{Benoit Mandelbrot:
  A Life in Many Dimensions}, pages 31--116. World Scientific, 2015.
\newblock \doi{10.1142/9789814366076_0003}.
\newblock URL \url{https://arxiv.org/abs/1210.0482}.

\bibitem[Acciaio et~al.(2024)Acciaio, Kratsios, and
  Pammer]{acciaio2024geometric}
Beatrice Acciaio, Anastasis Kratsios, and Gudmund Pammer.
\newblock Designing universal causal deep learning models: The geometric
  (hyper)transformer.
\newblock \emph{Mathematical Finance}, 34\penalty0 (2):\penalty0 671--735,
  2024.
\newblock \doi{10.1111/mafi.12389}.

\bibitem[{Amazon}(2024)]{chronosbolt2024}
{Amazon}.
\newblock {Chronos-Bolt-Tiny}: Pretrained time-series model and input-patch
  embedding.
\newblock Hugging Face model card and checkpoint, 2024.
\newblock URL \url{https://huggingface.co/amazon/chronos-bolt-tiny}.
\newblock Checkpoint revision 93a8129; input embedding implementation in
  chronos-forecasting v2.0.0. Accessed 5 September 2026.

\bibitem[Ansari et~al.(2024)Ansari, Stella, Turkmen, Zhang, Mercado, Shen,
  Shchur, Rangapuram, {Pineda Arango}, Kapoor, Zschiegner, Maddix, Wang,
  Mahoney, Torkkola, Wilson, Bohlke-Schneider, and Wang]{ansari2024chronos}
Abdul~Fatir Ansari, Lorenzo Stella, Caner Turkmen, Xiyuan Zhang, Pedro Mercado,
  Huibin Shen, Oleksandr Shchur, Syama~Sundar Rangapuram, Sebastian {Pineda
  Arango}, Shubham Kapoor, Jasper Zschiegner, Danielle~C. Maddix, Hao Wang,
  Michael~W. Mahoney, Kari Torkkola, Andrew~Gordon Wilson, Michael
  Bohlke-Schneider, and Yuyang Wang.
\newblock Chronos: Learning the language of time series.
\newblock \emph{Transactions on Machine Learning Research}, 2024.
\newblock ISSN 2835-8856.
\newblock URL \url{https://openreview.net/forum?id=gerNCVqqtR}.

\bibitem[Beltagy et~al.(2020)Beltagy, Peters, and Cohan]{beltagy2020longformer}
Iz~Beltagy, Matthew~E. Peters, and Arman Cohan.
\newblock Longformer: The long-document transformer.
\newblock \emph{arXiv preprint arXiv:2004.05150}, 2020.
\newblock URL \url{https://arxiv.org/abs/2004.05150}.

\bibitem[Block et~al.(2022)Block, Jia, Polyanskiy, and
  Rakhlin]{block2022intrinsic}
Adam Block, Zeyu Jia, Yury Polyanskiy, and Alexander Rakhlin.
\newblock Intrinsic dimension estimation using {Wasserstein} distance.
\newblock \emph{Journal of Machine Learning Research}, 23\penalty0
  (313):\penalty0 1--37, 2022.
\newblock URL \url{https://jmlr.org/papers/v23/21-1483.html}.

\bibitem[Candanedo(2017)]{candanedo2017appliances}
Luis Candanedo.
\newblock Appliances energy prediction.
\newblock UCI Machine Learning Repository, 2017.
\newblock URL \url{https://doi.org/10.24432/C5VC8G}.
\newblock Dataset.

\bibitem[Castin et~al.(2024)Castin, Ablin, and Peyr{\'e}]{castin2024smooth}
Val{\'e}rie Castin, Pierre Ablin, and Gabriel Peyr{\'e}.
\newblock How smooth is attention?
\newblock In \emph{Proceedings of the 41st International Conference on Machine
  Learning}, volume 235 of \emph{Proceedings of Machine Learning Research},
  pages 5817--5840, 2024.
\newblock URL \url{https://proceedings.mlr.press/v235/castin24a.html}.

\bibitem[Chen et~al.(2023)Chen, Wong, Chen, and Tian]{chen2023position}
Shouyuan Chen, Sherman Wong, Liangjian Chen, and Yuandong Tian.
\newblock Extending context window of large language models via positional
  interpolation.
\newblock \emph{arXiv preprint arXiv:2306.15595}, 2023.
\newblock URL \url{https://arxiv.org/abs/2306.15595}.

\bibitem[Chen(2015)]{chen2015beijing}
Song Chen.
\newblock {Beijing PM2.5}.
\newblock UCI Machine Learning Repository, 2015.
\newblock URL \url{https://doi.org/10.24432/C5JS49}.
\newblock Dataset.

\bibitem[Chen and Chen(1995)]{chen1995operators}
Tianping Chen and Hong Chen.
\newblock Universal approximation to nonlinear operators by neural networks
  with arbitrary activation functions and its application to dynamical systems.
\newblock \emph{IEEE Transactions on Neural Networks}, 6\penalty0 (4):\penalty0
  911--917, 1995.
\newblock \doi{10.1109/72.392253}.

\bibitem[Cuchiero et~al.(2026)Cuchiero, Schmocker, and
  Teichmann]{cuchiero2026global}
Christa Cuchiero, Philipp Schmocker, and Josef Teichmann.
\newblock Global universal approximation of functional input maps on weighted
  spaces.
\newblock \emph{Constructive Approximation}, 63\penalty0 (2):\penalty0
  537--612, 2026.
\newblock \doi{10.1007/s00365-025-09726-3}.

\bibitem[Cybenko(1989)]{cybenko1989approximation}
George Cybenko.
\newblock Approximation by superpositions of a sigmoidal function.
\newblock \emph{Mathematics of Control, Signals and Systems}, 2\penalty0
  (4):\penalty0 303--314, 1989.
\newblock \doi{10.1007/BF02551274}.

\bibitem[Dai et~al.(2019)Dai, Yang, Yang, Carbonell, Le, and
  Salakhutdinov]{dai2019transformerxl}
Zihang Dai, Zhilin Yang, Yiming Yang, Jaime Carbonell, Quoc~V. Le, and Ruslan
  Salakhutdinov.
\newblock Transformer-{XL}: Attentive language models beyond a fixed-length
  context.
\newblock In \emph{Proceedings of the 57th Annual Meeting of the Association
  for Computational Linguistics}, pages 2978--2988, 2019.
\newblock \doi{10.18653/v1/P19-1285}.

\bibitem[Dao et~al.(2022)Dao, Fu, Ermon, Rudra, and
  R{\'e}]{dao2022flashattention}
Tri Dao, Daniel~Y. Fu, Stefano Ermon, Atri Rudra, and Christopher R{\'e}.
\newblock {FlashAttention}: Fast and memory-efficient exact attention with
  {IO}-awareness.
\newblock In \emph{Advances in Neural Information Processing Systems},
  volume~35, pages 16344--16359, 2022.
\newblock \doi{10.52202/068431-1189}.

\bibitem[DeVore and Lorentz(1993)]{devore1993constructive}
Ronald~A. DeVore and George~G. Lorentz.
\newblock \emph{Constructive Approximation}, volume 303 of \emph{Grundlehren
  der mathematischen Wissenschaften}.
\newblock Springer Berlin, Heidelberg, 1993.
\newblock \doi{10.1007/978-3-662-02888-9}.

\bibitem[Ding et~al.(2024)Ding, Zhang, Zhang, Xu, Shang, Xu, Yang, and
  Yang]{ding2024longrope}
Yiran Ding, Li~Lyna Zhang, Chengruidong Zhang, Yuanyuan Xu, Ning Shang, Jiahang
  Xu, Fan Yang, and Mao Yang.
\newblock {LongRoPE}: Extending {LLM} context window beyond 2 million tokens.
\newblock In \emph{Proceedings of the 41st International Conference on Machine
  Learning}, volume 235 of \emph{Proceedings of Machine Learning Research},
  pages 11091--11104, 2024.
\newblock URL \url{https://proceedings.mlr.press/v235/ding24i.html}.

\bibitem[Funahashi and Nakamura(1993)]{funahashi1993dynamical}
Ken-ichi Funahashi and Yuichi Nakamura.
\newblock Approximation of dynamical systems by continuous time recurrent
  neural networks.
\newblock \emph{Neural Networks}, 6\penalty0 (6):\penalty0 801--806, 1993.
\newblock \doi{10.1016/S0893-6080(05)80125-X}.

\bibitem[Furuya et~al.(2025)Furuya, de~Hoop, and
  Peyr{\'e}]{furuya2025transformers}
Takashi Furuya, Maarten~V. de~Hoop, and Gabriel Peyr{\'e}.
\newblock Transformers are universal in-context learners.
\newblock In \emph{International Conference on Learning Representations}, 2025.
\newblock URL
  \url{https://proceedings.iclr.cc/paper_files/paper/2025/hash/c9028f7874df04843e7bf435ee4cd3c3-Abstract-Conference.html}.

\bibitem[Galimberti et~al.(2026)Galimberti, Kratsios, and
  Livieri]{galimberti2026causal}
Luca Galimberti, Anastasis Kratsios, and Giulia Livieri.
\newblock Designing universal causal deep learning models: The case of
  infinite-dimensional dynamical systems from stochastic analysis.
\newblock \emph{Constructive Approximation}, 2026.
\newblock \doi{10.1007/s00365-026-09745-8}.
\newblock Online first.

\bibitem[Gaunt and Li(2023)]{gaunt2023bounding}
Robert~E. Gaunt and Siqi Li.
\newblock Bounding {Kolmogorov} distances through {Wasserstein} and related
  integral probability metrics.
\newblock \emph{Journal of Mathematical Analysis and Applications},
  522\penalty0 (1):\penalty0 126985, 2023.
\newblock \doi{10.1016/j.jmaa.2022.126985}.

\bibitem[Gneiting et~al.(2012)Gneiting, {\v S}ev{\v c}{\'i}kov{\'a}, and
  Percival]{gneiting2012fractal}
Tilmann Gneiting, Hana {\v S}ev{\v c}{\'i}kov{\'a}, and Donald~B. Percival.
\newblock Estimators of fractal dimension: Assessing the roughness of time
  series and spatial data.
\newblock \emph{Statistical Science}, 27\penalty0 (2):\penalty0 247--277, 2012.
\newblock \doi{10.1214/11-STS370}.
\newblock URL \url{https://arxiv.org/abs/1101.1444}.

\bibitem[Gonon and Ortega(2021)]{gonon2021fading}
Lukas Gonon and Juan-Pablo Ortega.
\newblock Fading memory echo state networks are universal.
\newblock \emph{Neural Networks}, 138:\penalty0 10--13, 2021.
\newblock \doi{10.1016/j.neunet.2021.01.025}.

\bibitem[Gonon et~al.(2023)Gonon, Grigoryeva, and Ortega]{gonon2023random}
Lukas Gonon, Lyudmila Grigoryeva, and Juan-Pablo Ortega.
\newblock Approximation bounds for random neural networks and reservoir
  systems.
\newblock \emph{The Annals of Applied Probability}, 33\penalty0 (1):\penalty0
  28--69, 2023.
\newblock \doi{10.1214/22-AAP1806}.

\bibitem[Grigoryeva and Ortega(2018)]{grigoryeva2018echo}
Lyudmila Grigoryeva and Juan-Pablo Ortega.
\newblock Echo state networks are universal.
\newblock \emph{Neural Networks}, 108:\penalty0 495--508, 2018.
\newblock \doi{10.1016/j.neunet.2018.08.025}.

\bibitem[Hanson and Raginsky(2019)]{hanson2019tcn}
Joshua Hanson and Maxim Raginsky.
\newblock Universal approximation of input--output maps by temporal
  convolutional nets.
\newblock In \emph{Advances in Neural Information Processing Systems},
  volume~32, pages 14071--14081, 2019.
\newblock URL
  \url{https://proceedings.neurips.cc/paper/2019/hash/39555391eb0624a439c5131b1bb8a2e0-Abstract.html}.

\bibitem[Hochreiter and Schmidhuber(1997)]{hochreiter1997lstm}
Sepp Hochreiter and J{\"u}rgen Schmidhuber.
\newblock Long short-term memory.
\newblock \emph{Neural Computation}, 9\penalty0 (8):\penalty0 1735--1780, 1997.
\newblock \doi{10.1162/neco.1997.9.8.1735}.

\bibitem[Hogue(2019)]{hogue2019metro}
John Hogue.
\newblock Metro interstate traffic volume.
\newblock UCI Machine Learning Repository, 2019.
\newblock URL \url{https://doi.org/10.24432/C5X60B}.
\newblock Dataset.

\bibitem[Huang et~al.(2025)Huang, Yang, Bhattamishra, Sarrof, Krebs, Zhou,
  Nakkiran, and Hahn]{huang2025length}
Xinting Huang, Andy Yang, Satwik Bhattamishra, Yash Sarrof, Andreas Krebs,
  Hattie Zhou, Preetum Nakkiran, and Michael Hahn.
\newblock A formal framework for understanding length generalization in
  transformers.
\newblock In \emph{International Conference on Learning Representations}, 2025.
\newblock URL
  \url{https://proceedings.iclr.cc/paper_files/paper/2025/hash/928170bcb050fe64a63fe781b82265aa-Abstract-Conference.html}.

\bibitem[Jiang and Li(2024)]{jiang2024rate}
Haotian Jiang and Qianxiao Li.
\newblock Approximation rate of the transformer architecture for sequence
  modeling.
\newblock In \emph{Advances in Neural Information Processing Systems},
  volume~37, pages 68926--68955, 2024.
\newblock \doi{10.52202/079017-2202}.

\bibitem[Kajitsuka and Sato(2024)]{kajitsuka2024one}
Tokio Kajitsuka and Issei Sato.
\newblock Are transformers with one layer self-attention using low-rank weight
  matrices universal approximators?
\newblock In \emph{International Conference on Learning Representations}, 2024.
\newblock URL
  \url{https://proceedings.iclr.cc/paper_files/paper/2024/hash/96ee35c170cb8720033b16259c305da9-Abstract-Conference.html}.

\bibitem[Karagodin et~al.(2024)Karagodin, Polyanskiy, and
  Rigollet]{karagodin2024causal}
Nikita Karagodin, Yury Polyanskiy, and Philippe Rigollet.
\newblock Clustering in causal attention masking.
\newblock In \emph{Advances in Neural Information Processing Systems},
  volume~37, pages 115652--115681, 2024.
\newblock \doi{10.52202/079017-3673}.

\bibitem[Klusowski and Barron(2018)]{klusowski2018approximation}
Jason~M. Klusowski and Andrew~R. Barron.
\newblock Approximation by combinations of {ReLU} and squared {ReLU} ridge
  functions with $\ell^1$ and $\ell^0$ controls.
\newblock \emph{IEEE Transactions on Information Theory}, 64\penalty0
  (12):\penalty0 7649--7656, 2018.
\newblock \doi{10.1109/TIT.2018.2874447}.

\bibitem[Kovachki et~al.(2023)Kovachki, Li, Liu, Azizzadenesheli, Bhattacharya,
  Stuart, and Anandkumar]{kovachki2023neuraloperator}
Nikola Kovachki, Zongyi Li, Burigede Liu, Kamyar Azizzadenesheli, Kaushik
  Bhattacharya, Andrew Stuart, and Anima Anandkumar.
\newblock Neural operator: Learning maps between function spaces with
  applications to {PDEs}.
\newblock \emph{Journal of Machine Learning Research}, 24\penalty0
  (89):\penalty0 1--97, 2023.
\newblock URL \url{https://jmlr.org/papers/v24/21-1524.html}.

\bibitem[Kovachki et~al.(2024)Kovachki, Lanthaler, and
  Mhaskar]{kovachki2024data}
Nikola~B. Kovachki, Samuel Lanthaler, and Hrushikesh Mhaskar.
\newblock Data complexity estimates for operator learning, 2024.
\newblock URL \url{https://arxiv.org/abs/2405.15992}.
\newblock arXiv:2405.15992, version 2.

\bibitem[Leshno et~al.(1993)Leshno, Lin, Pinkus, and
  Schocken]{leshno1993nonpolynomial}
Moshe Leshno, Vladimir~Ya. Lin, Allan Pinkus, and Shimon Schocken.
\newblock Multilayer feedforward networks with a nonpolynomial activation
  function can approximate any function.
\newblock \emph{Neural Networks}, 6\penalty0 (6):\penalty0 861--867, 1993.
\newblock \doi{10.1016/S0893-6080(05)80131-5}.

\bibitem[Liu et~al.(2024)Liu, Yang, Chen, Zhao, and Liao]{liu2024deep}
Hao Liu, Haizhao Yang, Minshuo Chen, Tuo Zhao, and Wenjing Liao.
\newblock Deep nonparametric estimation of operators between infinite
  dimensional spaces.
\newblock \emph{Journal of Machine Learning Research}, 25\penalty0
  (24):\penalty0 1--67, 2024.
\newblock URL \url{https://jmlr.org/papers/v25/22-0719.html}.

\bibitem[Luo et~al.(2022)Luo, Li, Zheng, Liu, Wang, and He]{luo2022powerful}
Shengjie Luo, Shanda Li, Shuxin Zheng, Tie-Yan Liu, Liwei Wang, and Di~He.
\newblock Your transformer may not be as powerful as you expect.
\newblock In \emph{Advances in Neural Information Processing Systems},
  volume~35, pages 4301--4315, 2022.
\newblock \doi{10.52202/068431-0311}.

\bibitem[Maas et~al.(2011)Maas, Daly, Pham, Huang, Ng, and Potts]{maas2011imdb}
Andrew~L. Maas, Raymond~E. Daly, Peter~T. Pham, Dan Huang, Andrew~Y. Ng, and
  Christopher Potts.
\newblock Learning word vectors for sentiment analysis.
\newblock In \emph{Proceedings of the 49th Annual Meeting of the Association
  for Computational Linguistics: Human Language Technologies}, pages 142--150,
  Portland, Oregon, USA, June 2011. Association for Computational Linguistics.
\newblock URL \url{https://aclanthology.org/P11-1015/}.

\bibitem[{Max Planck Institute for Biogeochemistry}(n.d.)]{jenaweather}
{Max Planck Institute for Biogeochemistry}.
\newblock {Jena Climate Dataset}: Weather-station observations, 2009--2016.
\newblock
  \url{https://storage.googleapis.com/tensorflow/tf-keras-datasets/jena_climate_2009_2016.csv.zip},
  n.d.
\newblock Accessed 24 August 2026.

\bibitem[Meister(2016)]{meister2016optimal}
Alexander Meister.
\newblock Optimal classification and nonparametric regression for functional
  data.
\newblock \emph{Bernoulli}, 22\penalty0 (3):\penalty0 1729--1744, 2016.
\newblock \doi{10.3150/15-BEJ709}.
\newblock URL \url{https://arxiv.org/abs/1603.09130}.

\bibitem[Merity et~al.(2017)Merity, Xiong, Bradbury, and
  Socher]{merity2017wikitext}
Stephen Merity, Caiming Xiong, James Bradbury, and Richard Socher.
\newblock Pointer sentinel mixture models.
\newblock In \emph{5th International Conference on Learning Representations,
  {ICLR} 2017, Conference Track Proceedings}. OpenReview.net, 2017.
\newblock URL \url{https://openreview.net/forum?id=Byj72udxe}.

\bibitem[Perez~Alday et~al.(2020)Perez~Alday, Gu, Shah, Robichaux, Wong, Liu,
  Liu, Bahrami~Rad, Elola, Seyedi, Li, Sharma, Clifford, and
  Reyna]{alday2020challenge}
Erick~Andres Perez~Alday, Annie Gu, Amit~J. Shah, Chad Robichaux, An-Kwok~Ian
  Wong, Chengyu Liu, Feifei Liu, Ali Bahrami~Rad, Andoni Elola, Salman Seyedi,
  Qiao Li, Ashish Sharma, Gari~D. Clifford, and Matthew~A. Reyna.
\newblock Classification of 12-lead {ECGs}: the {{PhysioNet}/Computing in
  Cardiology Challenge 2020}.
\newblock \emph{Physiological Measurement}, 41\penalty0 (12):\penalty0 124003,
  2020.
\newblock \doi{10.1088/1361-6579/abc960}.

\bibitem[{Perez Alday} et~al.(2022){Perez Alday}, Gu, Shah, Liu, Sharma,
  Seyedi, {Bahrami Rad}, Reyna, and Clifford]{alday2022physionetdataset}
Erick~Andres {Perez Alday}, Annie Gu, Amit Shah, Chengyu Liu, Ashish Sharma,
  Salman Seyedi, Ali {Bahrami Rad}, Matthew Reyna, and Gari Clifford.
\newblock {Classification of 12-lead ECGs: The PhysioNet/Computing in
  Cardiology Challenge 2020}.
\newblock \emph{PhysioNet}, July 2022.
\newblock \doi{10.13026/dvyd-kd57}.
\newblock Version 1.0.2.

\bibitem[Pollard et~al.(2026)Pollard, Moody, Lehman, Gow, Fernandes, Xie,
  Johnson, Mark, and Heldt]{pollard2026physionet}
Tom Pollard, Benjamin~E. Moody, Li-wei~H. Lehman, Brian~J. Gow, Chrystinne
  Fernandes, Chen Xie, Alistair Johnson, Roger~G. Mark, and Thomas Heldt.
\newblock {PhysioNet} as a global platform for biomedical research.
\newblock \emph{Nature Health}, 1\penalty0 (8):\penalty0 792--795, 2026.
\newblock \doi{10.1038/s44360-026-00096-z}.

\bibitem[Press et~al.(2022)Press, Smith, and Lewis]{press2022train}
Ofir Press, Noah~A. Smith, and Mike Lewis.
\newblock Train short, test long: Attention with linear biases enables input
  length extrapolation.
\newblock In \emph{International Conference on Learning Representations}, 2022.
\newblock URL \url{https://openreview.net/forum?id=R8sQPpGCv0}.

\bibitem[Radford et~al.(2019)Radford, Wu, Child, Luan, Amodei, and
  Sutskever]{radford2019language}
Alec Radford, Jeffrey Wu, Rewon Child, David Luan, Dario Amodei, and Ilya
  Sutskever.
\newblock Language models are unsupervised multitask learners.
\newblock Technical report, OpenAI, 2019.
\newblock URL
  \url{https://cdn.openai.com/better-language-models/language-models.pdf}.

\bibitem[Sander and Peyr{\'e}(2025)]{sander2025nexttoken}
Micha{\"e}l~E. Sander and Gabriel Peyr{\'e}.
\newblock Towards understanding the universality of transformers for next-token
  prediction.
\newblock In \emph{International Conference on Learning Representations}, 2025.
\newblock URL
  \url{https://proceedings.iclr.cc/paper_files/paper/2025/hash/d846c59be138a704e800f36e7fcb696a-Abstract-Conference.html}.

\bibitem[Song et~al.(2023)Song, Hwang, Lee, and Kang]{song2023minimal}
Chang~hoon Song, Geonho Hwang, Jun~ho Lee, and Myungjoo Kang.
\newblock Minimal width for universal property of deep {RNN}.
\newblock \emph{Journal of Machine Learning Research}, 24\penalty0
  (121):\penalty0 1--41, 2023.
\newblock URL \url{https://jmlr.org/papers/v24/22-1191.html}.

\bibitem[Su et~al.(2024)Su, Ahmed, Lu, Pan, Bo, and Liu]{su2024roformer}
Jianlin Su, Murtadha Ahmed, Yu~Lu, Shengfeng Pan, Wen Bo, and Yunfeng Liu.
\newblock {RoFormer}: Enhanced transformer with rotary position embedding.
\newblock \emph{Neurocomputing}, 568:\penalty0 127063, 2024.
\newblock \doi{10.1016/j.neucom.2023.127063}.

\bibitem[Takakura and Suzuki(2023)]{takakura2023infinite}
Shokichi Takakura and Taiji Suzuki.
\newblock Approximation and estimation ability of transformers for
  sequence-to-sequence functions with infinite dimensional input.
\newblock In \emph{Proceedings of the 40th International Conference on Machine
  Learning}, volume 202 of \emph{Proceedings of Machine Learning Research},
  pages 33416--33447, 2023.
\newblock URL \url{https://proceedings.mlr.press/v202/takakura23a.html}.

\bibitem[Totik(2020)]{totik2020multivariate}
Vilmos Totik.
\newblock Polynomial approximation in several variables.
\newblock \emph{Journal of Approximation Theory}, 252:\penalty0 105364, 2020.
\newblock \doi{10.1016/j.jat.2019.105364}.

\bibitem[Vaswani et~al.(2017)Vaswani, Shazeer, Parmar, Uszkoreit, Jones, Gomez,
  Kaiser, and Polosukhin]{vaswani2017attention}
Ashish Vaswani, Noam Shazeer, Niki Parmar, Jakob Uszkoreit, Llion Jones,
  Aidan~N. Gomez, {\L}ukasz Kaiser, and Illia Polosukhin.
\newblock Attention is all you need.
\newblock In \emph{Advances in Neural Information Processing Systems},
  volume~30, 2017.
\newblock URL
  \url{https://proceedings.neurips.cc/paper_files/paper/2017/hash/3f5ee243547dee91fbd053c1c4a845aa-Abstract.html}.

\bibitem[Vuckovic et~al.(2020)Vuckovic, Baratin, and Tachet~des
  Combes]{vuckovic2020attention}
James Vuckovic, Aristide Baratin, and R{\'e}mi Tachet~des Combes.
\newblock A mathematical theory of attention.
\newblock \emph{arXiv preprint arXiv:2007.02876}, 2020.
\newblock URL \url{https://arxiv.org/abs/2007.02876}.

\bibitem[Yang et~al.(2023)Yang, Meng, Lin, and Zhang]{yang2023parrot}
Haotong Yang, Fanxu Meng, Zhouchen Lin, and Muhan Zhang.
\newblock {Parrot Mind}: Towards explaining the complex task reasoning of
  pretrained large language models with template-content structure.
\newblock \emph{arXiv preprint arXiv:2310.05452}, 2023.
\newblock URL \url{https://arxiv.org/abs/2310.05452}.
\newblock Revised 5 April 2024.

\bibitem[Yun et~al.(2020{\natexlab{a}})Yun, Bhojanapalli, Rawat, Reddi, and
  Kumar]{yun2020transformers}
Chulhee Yun, Srinadh Bhojanapalli, Ankit~Singh Rawat, Sashank~J. Reddi, and
  Sanjiv Kumar.
\newblock Are transformers universal approximators of sequence-to-sequence
  functions?
\newblock In \emph{International Conference on Learning Representations},
  2020{\natexlab{a}}.
\newblock URL \url{https://openreview.net/forum?id=ByxRM0Ntvr}.

\bibitem[Yun et~al.(2020{\natexlab{b}})Yun, Chang, Bhojanapalli, Rawat, Reddi,
  and Kumar]{yun2020sparse}
Chulhee Yun, Yin-Wen Chang, Srinadh Bhojanapalli, Ankit~Singh Rawat, Sashank~J.
  Reddi, and Sanjiv Kumar.
\newblock {O(n)} connections are expressive enough: Universal approximability
  of sparse transformers.
\newblock In \emph{Advances in Neural Information Processing Systems},
  volume~33, pages 13783--13794, 2020{\natexlab{b}}.
\newblock URL
  \url{https://proceedings.neurips.cc/paper_files/paper/2020/hash/9ed27554c893b5bad850a422c3538c15-Abstract.html}.

\bibitem[Zaheer et~al.(2020)Zaheer, Guruganesh, Dubey, Ainslie, Alberti,
  Onta{\~n}{\'o}n, Pham, Ravula, Wang, Yang, and Ahmed]{zaheer2020bigbird}
Manzil Zaheer, Guru Guruganesh, Kumar~Avinava Dubey, Joshua Ainslie, Chris
  Alberti, Santiago Onta{\~n}{\'o}n, Philip Pham, Anirudh Ravula, Qifan Wang,
  Li~Yang, and Amr Ahmed.
\newblock {Big Bird}: Transformers for longer sequences.
\newblock In \emph{Advances in Neural Information Processing Systems},
  volume~33, pages 17283--17297, 2020.
\newblock URL
  \url{https://proceedings.neurips.cc/paper/2020/hash/c8512d142a2d849725f31a9a7a361ab9-Abstract.html}.

\bibitem[Zhang et~al.(2015)Zhang, Zhao, and LeCun]{zhang2015character}
Xiang Zhang, Junbo Zhao, and Yann LeCun.
\newblock Character-level convolutional networks for text classification.
\newblock In \emph{Advances in Neural Information Processing Systems},
  volume~28, pages 649--657, 2015.
\newblock URL
  \url{https://proceedings.neurips.cc/paper/2015/hash/250cf8b51c773f3f8dc8b4be867a9a02-Abstract.html}.

\bibitem[Zhou et~al.(2021)Zhou, Zhang, Peng, Zhang, Li, Xiong, and
  Zhang]{zhou2021informer}
Haoyi Zhou, Shanghang Zhang, Jieqi Peng, Shuai Zhang, Jianxin Li, Hui Xiong,
  and Wancai Zhang.
\newblock Informer: Beyond efficient transformer for long sequence time-series
  forecasting.
\newblock \emph{Proceedings of the AAAI Conference on Artificial Intelligence},
  35\penalty0 (12):\penalty0 11106--11115, 2021.
\newblock \doi{10.1609/aaai.v35i12.17325}.

\end{thebibliography}

\clearpage
\appendix

% Order of first main-text reference; keep companion proofs with their statements.
\section{Positional encoding is needed for universality}
\label{app:position-needed}

The theorems place normalized position in the hidden state but
use no explicit length channel.  We show that total length and score-side
position cannot replace a position-revealing operation outside the normalized
score channel.  The obstruction is independent of the admissible modulus:
constant sequences and paths belong to every class considered above.

\begin{definition}[Score-only positional normalized architecture]
\label{def:position-free-v2}
At a fixed length \(n\), a score-only positional normalized transformer has
finite hidden dimensions and finite depth, and initializes according to
\begin{equation}
\label{eq:position-free-init-v2}
u_i^0\eqdef\chi_n(z_i),
\end{equation}
where the same map \(\chi_n\) is used at every position.  The map may depend on
the known total length \(n\), but not on \(i\).  Each normalized attention
sublayer has the form
\begin{equation}
\label{eq:score-only-attention-v2}
\mathcal A_{\ell,n}(u)_i
\eqdef u_i+\sum_{r=1}^{H_\ell}W_{\ell,r}
\sum_{j=1}^i\omega_{ij}^{\ell,r}(u_{1:i})V_{\ell,r}u_j,
\qquad
\omega_{ij}^{\ell,r}\in\R,
\quad
\sum_{j=1}^i\omega_{ij}^{\ell,r}=1.
\end{equation}
The weights may depend arbitrarily on \(n,i,j\), on the causal hidden prefix,
and on absolute or relative positions.  Thus they include ordinary masked
softmax, score biases such as ALiBi, and position-dependent query/key maps such
as RoPE.  The value and output matrices are shared across positions, every
other sublayer is a shared pointwise map, and position is absent from the
initialization, values, residual features, and pointwise maps.  There is no
distinguished start token or other position-dependent value input.

A score-only positional normalized temporal transformer initializes
\(u^0(t)\eqdef\chi_\infty(x(t))\) and replaces \eqref{eq:score-only-attention-v2} by
\begin{equation}
\label{eq:score-only-temporal-v2}
\mathcal A_{\ell,\infty}(u)(t)
\eqdef u(t)+\sum_{r=1}^{H_\ell}W_{\ell,r}
\int_{[0,t]}V_{\ell,r}u(s)\,\dd\rho_{\ell,r,t,u}(s),
\end{equation}
where \(\rho_{\ell,r,t,u}\) is any finite signed Borel measure on \([0,t]\)
with
\(
\rho_{\ell,r,t,u}([0,t])=1
\), possibly depending on \(t\), absolute or relative time, and the causal
hidden path \(u|_{[0,t]}\).  Thus the integral is well-defined against the
current bounded Borel value path.  Admissible kernels are required only to make
every displayed integral well-defined and every sublayer map bounded Borel
paths to bounded Borel paths.  All non-attention maps remain pointwise and
independent of \(t\).  No continuity in the parameter \(t\) is imposed on the
resulting output: the lower bound below holds for this broader Borel-path
class, and therefore also for any subclass required to map
continuous paths to continuous paths.
\end{definition}

Even upon allowing the architecture and parameters to vary with resolution, this
class cannot recover the normalized clock on constant inputs.

\begin{theorem}[Sharp clock obstruction when position is confined to scores]
\label{thm:position-obstruction}
Fix an admissible common modulus \(\omega\), with input classes
\(X_n^\omega\) and \(X_\infty^\omega\) as in
\eqref{eq:Xn-v2} and \eqref{eq:Xinfty-v2}, and consider the
continuously extendable causal clock family
\begin{equation}
\label{eq:clock-family-v2}
\Theta_n(z)_i\eqdef i/n,
\qquad
\Theta_\infty(x)(t)\eqdef t.
\end{equation}
Then the following two statements hold.
\begin{enumerate}[label=(\alph*),leftmargin=*,itemsep=3pt,topsep=3pt]
\item Let \((S_n)_{n\geq1}\) be any family in which each \(S_n\) is a scalar-output score-only positional normalized transformer at length \(n\) in the sense of Definition~\ref{def:position-free-v2}.  The architectures and parameters may vary arbitrarily with \(n\).  For every \(c\in\Omega\) and \(n\geq1\), there is a scalar \(b_n(c)\) such that
\begin{equation}
\label{eq:constant-output-v2}
S_n(c,\ldots,c)_i=b_n(c),
\qquad i\in[n],
\end{equation}
and
\begin{equation}
\label{eq:finite-clock-per-resolution-v2}
\sup_{z\in X_n^\omega}\ \max_{i\in[n]}
\abs{S_n(z)_i-i/n}
\geq\frac{1-1/n}{2},
\qquad n\geq1,
\end{equation}
and consequently
\begin{equation}
\label{eq:finite-clock-lower-v2}
\sup_{n\geq1}\ \sup_{z\in X_n^\omega}\ \max_{i\in[n]}
\abs{S_n(z)_i-i/n}
\geq\frac12.
\end{equation}
\item Every scalar-output score-only positional normalized temporal transformer \(S_\infty\) satisfies
\begin{equation}
\label{eq:continuous-clock-lower-v2}
\sup_{x\in X_\infty^\omega}\ \sup_{t\in[0,1]}
\abs{S_\infty(x)(t)-t}
\geq\frac12.
\end{equation}
\end{enumerate}
The per-resolution bound \eqref{eq:finite-clock-per-resolution-v2} and both
all-resolution lower bounds are sharp.
\end{theorem}

The obstruction is constant-input invariance: every unit-mass normalized
attention aggregate, even with signed and position-dependent weights, maps
identical values to an identical value, and shared pointwise layers preserve
this property.  Comparing the first and last clock values yields the
fixed-resolution bound and its sharp all-resolution limit.  The claim is
architecture-specific: distinguished start tokens, position-dependent values
or residual features, position-dependent pointwise maps, unnormalized prefix
sums, and distribution-specific guarantees lie outside its scope.  Together
with Theorems~\ref{thm:discrete-universality} and
\ref{thm:continuous-universality}, it shows that normalized residual position
is sufficient, whereas masking, total length, and score-side position alone
are not.

\subsection{Proof}
\label{app:proof-position}

The proof uses a blockwise invariance of normalized affine aggregation.  It
holds in arbitrary hidden dimension, with any finite number of heads and any
real, position-dependent unit-mass weights.

\begin{lemma}[Constant-state invariance]
\label{lem:constant-state-invariance}
Consider a finite sequence whose hidden vectors are all equal:
\(
u_1=\cdots=u_n=v
\).
Then every output of a score-only positional normalized attention block
\eqref{eq:score-only-attention-v2} is the same vector,
\begin{equation}
\label{eq:constant-attention-state-v2}
\mathcal A_{\ell,n}(u)_i
=v+\sum_{r=1}^{H_\ell}W_{\ell,r}V_{\ell,r}v,
\qquad i\in[n].
\end{equation}
A shared pointwise map also preserves equality across positions.  The temporal counterparts preserve paths that are constant in time.
\end{lemma}

\begin{proof}
\smallskip\noindent\emph{Finite sequences.}\hspace{0.25em}
Fix a head \(r\) and a query position \(i\).  Every value is
\(V_{\ell,r}v\), while the normalized weights may be completely different
across keys and query positions.  Nevertheless,
\[
\sum_{j=1}^i\omega_{ij}^{\ell,r}(u_{1:i})V_{\ell,r}v
=\left(\sum_{j=1}^i\omega_{ij}^{\ell,r}(u_{1:i})\right)V_{\ell,r}v
=V_{\ell,r}v.
\]
The result is independent of the prefix size and of every score-side positional
mechanism.  Summing the heads and adding the position-independent residual
gives \eqref{eq:constant-attention-state-v2}.  Applying the same pointwise
function to identical vectors again gives identical vectors.

\smallskip\noindent\emph{Continuous time.}\hspace{0.25em}
For a constant path \(u(t)=v\), each \(\rho_{\ell,r,t,u}\) in
\eqref{eq:score-only-temporal-v2} has total mass one, so
\[
\int_{[0,t]}V_{\ell,r}v\,\dd\rho_{\ell,r,t,u}=V_{\ell,r}v
\]
also at \(t=0\).  Thus the temporal block produces the same constant path.
Induction proves the claim through any finite composition.
\end{proof}

\begin{proof}[Proof of Theorem~\ref{thm:position-obstruction}]
\smallskip\noindent\emph{Discrete lower bound.}\hspace{0.25em}
For a constant input \((c,\ldots,c)\), the shared initialization \(\chi_n\) produces a constant hidden sequence.  Lemma~\ref{lem:constant-state-invariance} and induction through the length-\(n\) architecture prove \eqref{eq:constant-output-v2}, even when every attention sublayer uses position-dependent scores.  This argument is separate at each \(n\), so neither the parameters nor the architecture need be shared across lengths.

For every scalar \(b\), the triangle inequality at the first and last positions gives
\begin{equation}
\label{eq:endpoint-clock-bound-v2}
\abs{1-1/n}
\leq
\abs{b-1/n}+\abs{b-1}
\leq
2\max_{i\in[n]}\abs{b-i/n}.
\end{equation}
Using \(b=b_n(c)\), then taking the supremum over \(z\in X_n^\omega\), proves
\eqref{eq:finite-clock-per-resolution-v2}.  Taking the supremum over \(n\)
proves \eqref{eq:finite-clock-lower-v2}.

\smallskip\noindent\emph{Continuous lower bound and sharpness.}\hspace{0.25em}
In continuous time,
Lemma~\ref{lem:constant-state-invariance} gives a constant output \(b\) on
every constant input path, and
\(
1\leq\abs{b}+\abs{b-1}
\)
proves \eqref{eq:continuous-clock-lower-v2}.

The clock family is continuously extendable: all finite maps are continuous and
prefix-causal, the boundary map is continuous, and its joint evaluation is
\(
\overline\Theta(\xi)=\tau(\xi)
\), which is continuous by \eqref{eq:joint-metric-v2}.  At each fixed length
\(n\), the constant predictor
\(
S_n\equiv(1+1/n)/2
\)
attains an error exactly equal to \((1-1/n)/2\), proving sharpness of
\eqref{eq:finite-clock-per-resolution-v2}.  The constant shared predictor
\(S\equiv1/2\), implemented by zero attention contributions and a final shared
affine bias, attains all-resolution error \(1/2\) in both discrete and
continuous time.  Hence both all-resolution bounds are sharp as well.
\end{proof}

\clearpage
\section{Quantitative approximation rates}
\label{app:quantitative-rates}

The upper bound
% does not require a power-law model for the input paths: it
holds on every compact all-resolution class generated by an admissible common
modulus. The rate comes instead from regularity of the target on causal
histories.  We therefore fix an arbitrary admissible modulus \(\omega\) and
write
\begin{equation}
\label{eq:quant-generic-classes}
 X_n\eqdef X_n^\omega,
 \qquad X_\infty\eqdef X_\infty^\omega.
\end{equation}
Let \(\mathfrak X_\omega\) be the corresponding completed state space from
\eqref{eq:joint-state-v2}--\eqref{eq:joint-metric-v2}, and let
\(\cM_\omega\) be its compact prefix-law image from
Proposition~\ref{prop:prefix-law}.  These constructions are summarized in
the detailed proof guide in Appendix~\ref{app:detailed-proof-guide} and proved
in Appendix~\ref{app:proof-joint}.  In the upper-bound statement and proof,
\(\mathfrak X\) and \(\cM\) abbreviate these two spaces.  We first prove one
estimate at every discrete resolution and then transfer it, without loss, to
the continuous case.  Only after proving this generic-modulus result do
we specialize to H\"older histories to investigate optimality.

\subsection{\texorpdfstring{\(\beta\)}{Beta}-smooth target regularity}
\label{sec:quantitative-regularity}

To turn qualitative density into a rate, we measure how much the target can
vary between histories that finite collections of attention probes cannot yet
distinguish.  A single family of smooth-test distances covers both H\"older
and higher-order regularity on the completed history space used in the
qualitative proof. This section expands
Definition~\ref{def:attention-teacher-class}; its normalized teacher class is
the ball \(R_{\beta;\omega}\leq1\), \(\norm{f^\star}_\infty\leq1\).

\paragraph{Causal-history factor.}
Recall from \eqref{eq:factored-target-v2} that every continuously
extendable causal family factors uniquely as
\begin{equation}
\label{eq:quant-same-factor}
 \overline F^\star(\xi)=f^\star(\mu_\xi),
 \qquad f^\star:\cM\to\R^{d'}.
\end{equation}
The map \(f^\star\) is therefore the common readout of the discrete and
continuous targets after histories with the same causal information have
been identified.

\paragraph{Causal-history metric.}
Recall the continuous endpoint extractor from
\eqref{eq:prefix-endpoint-extractor-v2}.  Concretely,
\begin{equation}
\label{eq:causal-endpoint-map}
 \mathfrak e(\mu_\xi)
 \eqdef y_\xi
 \eqdef
 \begin{cases}
 (z_i,i/n),&\xi=(n,z,i),\\
 (x(t),t),&\xi=(\infty,x,t).
 \end{cases}
\end{equation}
Set \(E\eqdef\Omega\times[0,1]\).  Consequently
\begin{equation}
\label{eq:causal-history-chart}
 \iota:\cM\to\cP(E)\times E,
 \qquad \iota(\mu)\eqdef(\mu,\mathfrak e(\mu)),
\end{equation}
is a homeomorphism onto its image.  The augmented state is precisely the
graph of \(\mathfrak e\), hence is canonically identified with \(\cM\).

Equip \(E\) with
\begin{equation}
\label{eq:E-metric-rates}
 d_E((x,s),(y,t))\eqdef \norm{x-y}_\infty+\abs{s-t},
\end{equation}
and let \(W_1\) be the associated \(1\)-Wasserstein distance.  Set
\begin{equation}
\label{eq:causal-chart-metric}
 D_{\mathrm{caus}}(\mu,\nu)
 \eqdef W_1(\mu,\nu)
 +\norm{\mathfrak e(\mu)-\mathfrak e(\nu)}_\infty,
 \qquad \mu,\nu\in\cM.
\end{equation}
Although topologically redundant, the endpoint term is essential for the
rate: a positive-time continuous prefix assigns zero mass to its terminal token,
whereas the residual coordinates retain that token exactly.
Here the endpoint is the current token and its position, not necessarily
the end of the full sequence. In our attention construction, the heads write
only into auxiliary coordinates, so the residual connection carries these
endpoint coordinates unchanged to the pointwise readout
(Lemma~\ref{lem:parallel-realization}).

\paragraph{Isotropic H\"older restriction norm.}
\label{par:holder-restriction-norm}
Let \(\mathsf Q\eqdef[-1,1]^d\times[0,1]\), and extend \(d_E\) to
\(\mathsf Q\) by the same formula as in \eqref{eq:E-metric-rates}.
For \(\sigma>0\), put
\(
r_\sigma\eqdef\lceil\sigma\rceil-1
\)
and
\(
\theta_\sigma\eqdef\sigma-r_\sigma\in(0,1]
\).
For a multi-index \(a=(a_1,\ldots,a_{d+1})\in\N_0^{d+1}\), write
\(\lvert a\rvert\eqdef\sum_{j=1}^{d+1}a_j\) and
\(\partial^a\eqdef\partial_1^{a_1}\cdots\partial_{d+1}^{a_{d+1}}\).
We use \(C^\sigma(\mathsf Q)\eqdef
C^{r_\sigma,\theta_\sigma}(\mathsf Q)\): its elements are real-valued
functions whose derivatives through order \(r_\sigma\) exist in the
interior and extend continuously to the boundary, with finite full
isotropic norm
\begin{equation}
\label{eq:holder-box-norm}
\begin{aligned}
 \norm{g}_{C^\sigma(\mathsf Q)}
 &\eqdef
 \sum_{\lvert a\rvert\leq r_\sigma}
       \norm{\partial^a g}_{L^\infty(\mathsf Q)}
 +\sum_{\lvert a\rvert=r_\sigma}
       [\partial^a g]_{\theta_\sigma;\mathsf Q},\\
 [h]_{\theta;\mathsf Q}
 &\eqdef
 \sup_{\substack{u,v\in\mathsf Q\\u\neq v}}
 \frac{\abs{h(u)-h(v)}}{d_E(u,v)^\theta}.
\end{aligned}
\end{equation}
Here \(\partial^0g=g\), so this also defines the norm when
\(r_\sigma=0\). In particular, at integer orders we use the endpoint
convention \(C^m(\mathsf Q)=C^{m-1,1}(\mathsf Q)\), rather than the
classical space defined only by continuous derivatives through order \(m\).
For a function \(\varphi:E\to\R\), the restriction norm is
\begin{equation}
\label{eq:holder-restriction-norm}
 \norm{\varphi}_{C^\sigma(E)}
 \eqdef
 \inf\left\{
   \norm{g}_{C^\sigma(\mathsf Q)}:
   g\in C^\sigma(\mathsf Q),\quad g|_E=\varphi
 \right\},
\end{equation}
with infimum \(+\infty\) if there is no such extension; \(C^\sigma(E)\)
is the space where this norm is finite. Equivalent norms on this fixed
finite-dimensional box change the estimates only by constants depending on
dimension and the fixed smoothness order.

\paragraph{Smooth-test Wasserstein distances.}
Testing against increasingly smooth functions gives a metric scale adapted
to the moments recovered by attention. This is the H\"older integral
probability metric construction~\citep{block2022intrinsic}; at integer orders
it is equivalent, up to norm conventions, to bounded smooth Wasserstein
metrics~\citep{gaunt2023bounding}. Define
\begin{equation}
\label{eq:smooth-test-discrepancy-main}
 d_\sigma(\mu,\nu)
 \eqdef
 \sup_{\norm{\varphi}_{C^\sigma(E)}\leq1}
 \int_E\varphi\,\dd(\mu-\nu).
\end{equation}
The supremum is unchanged if the integral is replaced by its absolute value,
since the test class is closed under \(\varphi\mapsto-\varphi\).
Because restrictions of polynomials belong to \(C^\sigma(E)\) and are dense
in \(C(E)\), this test class separates probability measures and
\(d_\sigma\) is a metric on \(\cP(E)\).

For \(\beta>0\), put \(\beta_-\eqdef\min\{\beta,1\}\) and
\(\beta_+\eqdef\max\{\beta,1\}\). Define
\begin{equation}
\label{eq:smooth-wasserstein-distance}
 W_{1,\beta}(\mu,\nu)\eqdef d_{\beta_+}(\mu,\nu)^{\beta_-},
 \qquad
 \delta_e(\mu,\nu)\eqdef
 \norm{\mathfrak e(\mu)-\mathfrak e(\nu)}_\infty,
\end{equation}
and augment the law distance by the difference of the retained endpoint coordinates:
\begin{equation}
\label{eq:causal-smooth-modulus}
 \Delta_\beta(\mu,\nu)
 \eqdef W_{1,\beta}(\mu,\nu)+\delta_e(\mu,\nu)^{\beta_-}.
\end{equation}
In \(W_{1,\beta}\), \(\beta\) is a smoothness index: it denotes
neither a power of \(W_1\) nor a Wasserstein transport order.
The following lemma shows that these distances retain the original compact
topology while changing its quantitative geometry.

\begin{lemma}[Metric properties and Wasserstein normalization]
\label{lem:unified-history-metric}
For every \(\beta>0\), \(W_{1,\beta}\) is a metric on \(\cP(E)\)
inducing weak convergence, and \(\Delta_\beta\) is a metric on \(\cM\)
inducing its original compact topology. With constants depending only on
\((d,\beta)\),
\begin{equation}
\label{eq:unified-metric-comparison}
 W_{1,1}=d_1\asymp W_1,
 \qquad
 \Delta_\beta\asymp D_{\mathrm{caus}}^\beta\quad(0<\beta\leq1),
 \qquad
 \Delta_\beta=d_\beta+\delta_e\quad(\beta>1).
\end{equation}
If the test norm at order one is replaced by the Lipschitz seminorm for
\(d_E\), then \(W_{1,1}=W_1\) exactly.
\end{lemma}

\begin{proof}
For \(\sigma\geq1\), the full \(C^\sigma\) norm controls the Lipschitz
seminorm, so Kantorovich duality gives \(d_\sigma\leq C_{d,\sigma}W_1\).
Conversely, a one-Lipschitz test on \(E\) extends to \(\mathsf Q\) with
the same Lipschitz constant. Subtracting its value at a fixed point bounds
its uniform norm by \(\operatorname{diam}_{d_E}(\mathsf Q)\), without
changing its integral against \(\mu-\nu\). Its full \(C^1\) restriction
norm is therefore bounded by a dimension-dependent constant, proving
\(W_1\leq C_d d_1\). With the Lipschitz seminorm alone, Kantorovich
duality gives exact equality.

For \(\sigma\geq1\), \(d_\sigma\) is a separating dual metric, and
\(d_\sigma\lesssim W_1\) makes the identity from the weakly compact
space \(\cP(E)\) to its \(d_\sigma\)-metric topology continuous.
Compactness and the Hausdorff property imply that these topologies coincide.
For \(q\in(0,1]\), the inequality \((a+b)^q\leq a^q+b^q\) shows
that the \(q\)-th power of a metric is again a metric with the same
topology. Taking \(q=\beta_-\) proves the assertion for
\(W_{1,\beta}\). The endpoint extractor is continuous on \(\cM\),
so \(\delta_e^{\beta_-}\) is a continuous pseudometric; adding it
preserves the metric and topology. Finally, for \(0<\beta\leq1\),
\[
 (a+b)^\beta\leq a^\beta+b^\beta
 \leq2^{1-\beta}(a+b)^\beta,
\]
which, together with \(d_1\asymp W_1\), proves the comparison with
\(D_{\mathrm{caus}}^\beta\). For \(\beta>1\), the equality is immediate.
\end{proof}

\paragraph{Target regularity.}
We measure regularity of the target by Lipschitz continuity in
\(\Delta_\beta\). The preceding comparison recovers ordinary H\"older
continuity at low orders. At high orders, smooth test functions replace
powers of the underlying distance: using \(D_{\mathrm{caus}}^\beta\)
for \(\beta>1\) would generally violate the triangle inequality and
force a target satisfying that modulus to be constant along every
rectifiable path. Increasing \(\beta>1\) instead shrinks the smooth-test
metric, up to fixed norm-embedding constants, and hence strengthens target
regularity. The endpoint remains first order because the residual coordinates
preserve it exactly.

\begin{definition}[\(\beta\)-smooth regularity on causal histories]
\label{def:causal-smoothness}
Fix \(\beta>0\).  The \(\beta\)-smooth seminorm of a completed causal family
\(\mathbf F^\star\) on the class determined by \(\omega\) is
\begin{equation}
\label{eq:causal-smoothness-seminorm}
 R_{\beta;\omega}(\mathbf F^\star)
 \eqdef
 \sup_{\substack{\mu,\nu\in\cM_\omega\\\mu\neq\nu}}
 \frac{\norm{f^\star(\mu)-f^\star(\nu)}}
 {\Delta_\beta(\mu,\nu)}.
\end{equation}
Equivalently, \(R_{\beta;\omega}(\mathbf F^\star)\) is the optimal
(smallest) Lipschitz constant of
\(f^\star:(\cM_\omega,\Delta_\beta)\to\R^{d'}\), with value
\(+\infty\) if no finite Lipschitz constant exists. The family is
\emph{\(\beta\)-smooth} when this quantity is finite.
\end{definition}

For \(0<\beta\leq1\), this is equivalent to ordinary \(\beta\)-H\"older continuity in
the causal-history metric.  For noninteger
\(\beta=m+\vartheta>1\), with \(m\in\N\) and \(0<\vartheta<1\), the measure
term tests isotropic \(C^{m,\vartheta}\) regularity in the joint token--time
variable.  At an integer \(\beta=m\geq2\), we use the endpoint convention
\(C^{m-1,1}\). In both high-order regimes, the endpoint contribution to
\(\Delta_\beta\) is the first-order term \(\delta_e\), because these
coordinates are retained exactly rather than reconstructed from attention
statistics. This does not require higher differentiability of endpoint-only
functions. Thus \(\beta\)-smoothness means stability in the
smooth-test metric, not generic Fr\'echet \(C^\beta\) regularity of a path
operator.

The order \(\beta\) is fixed while the parameter budget grows: constants may
depend nonuniformly on \(\beta\), and we make no analytic-class claim.
Moreover, because \(\cM=\cM_\omega\) depends on the input class, both
\(R_{\beta;\omega}(\mathbf F^\star)\) and the learned readout may depend on
\(\omega\), even though the approximation constants below do not.

\subsection{Quantitative all-resolution rate}
\label{sec:parameter-rate}

We now strengthen Theorem~\ref{thm:discrete-universality} by converting
target regularity into an explicit parameter--error tradeoff, uniform over
all sequence lengths.  For a completed family
\(\mathbf F^\star=((F_n^\star)_{n\geq1},F_\infty^\star)\) and a shared
transformer with its temporal realization, define the discrete and boundary
errors by
\begin{align}
\label{eq:discrete-error-v3}
E_{\mathrm{disc}}(T;\mathbf F^\star)
&\eqdef
\sup_{n\geq1}\sup_{z\in X_n}\max_{i\in[n]}
\norm{T(\phi_n(z))_i-F_n^\star(z)_i},\\
\label{eq:boundary-error-v3}
E_\infty(T;\mathbf F^\star)
&\eqdef
\sup_{x\in X_\infty}
\norm{T_\infty(\phi_\infty(x))-F_\infty^\star(x)}_\infty.
\end{align}
Their maximum is the joint all-resolution error:
\begin{equation}
\label{eq:quantitative-error-functional}
\operatorname{Err}(T;\mathbf F^\star)
\eqdef
\max\bigl\{E_{\mathrm{disc}}(T;\mathbf F^\star),
E_\infty(T;\mathbf F^\star)\bigr\}.
\end{equation}
For a transformer with \(H\) scalar heads, residual width \(r=d+H+2\), and
an \(M\)-unit pointwise readout, the \emph{dense parameter count} counts
every matrix and bias entry, including structural zeros, rather than only
nonzero coefficients:
\begin{equation}
\label{eq:exact-dense-parameter-count-main}
 \operatorname{par}(T)
 \eqdef 4Hr+(r+d'+1)M+d'.
\end{equation}
Each head contributes four matrices with \(r\) entries; the two readout
matrices and two biases contribute \(rM+M+d'M+d'\).  Thus structural zeros
and biases are charged, shared entries are charged once, and the fixed lift
\(\mathcal E_H\) is parameter-free.  This count charges neither parameter
magnitude nor numerical precision. The following theorem records the joint
discrete--continuous rate, the head budget, and the equivalent
accuracy-to-parameter bound. Its joint conclusion combines the discrete
construction with the lossless boundary transfer proved in
Corollary~\ref{cor:quantitative-temporal-closure}.

\begin{theorem}[Quantitative all-resolution causal universality]
\label{thm:quantitative-rates}
\label{thm:quantitative-universality-main}
Fix an admissible common modulus \(\omega\) and \(\beta>0\), and let \(X_n\)
and \(X_\infty\) be the input classes in
\eqref{eq:quant-generic-classes}.  Let
\((F_n^\star)_{n\geq1}\) be a continuously extendable causal family in the
sense of Definition~\ref{def:compatible-family}, with its unique extension
\(F_\infty^\star\), and write
\(\mathbf F^\star=((F_n^\star)_{n\geq1},F_\infty^\star)\).  Assume
\(\mathbf F^\star\) is \(\beta\)-smooth in the sense of
Definition~\ref{def:causal-smoothness}, so that
\(R_{\beta;\omega}(\mathbf F^\star)<\infty\).
There exist a constant \(C_\beta=C(d,d',\beta)>0\) and an integer
\(p_0=p_0(d,d',\beta)\geq16\), both independent of \(\omega\) and the target.
For every integer
\(p\geq p_0\), there are integers \(H_p,M_p\geq1\) and a transformer \(T_p\)
of the shallow form
\eqref{eq:theorem-lift-v2}--\eqref{eq:shallow-transformer-v2}, with residual
width \(d+H_p+2\), \(H_p\) scalar heads, and an \(M_p\)-unit pointwise
readout.  Its dimensions, matrices, and biases are independent of \(n\), and,
together with its temporal realization, it satisfies
\begin{equation}
\label{eq:quantitative-error-p}
 \operatorname{par}(T_p)\leq p,
 \qquad
 \operatorname{Err}(T_p;\mathbf F^\star)
 \leq
 C_\beta R_{\beta;\omega}(\mathbf F^\star)
 \left(\frac{\log\log p}{\log p}\right)^{\beta/(d+2)}.
\end{equation}
The attention block can be chosen with
\begin{equation}
\label{eq:quantitative-heads-p}
 H_p
 \leq
 C_\beta
 \left(\frac{\log p}{\log\log p}\right)^{(d+1)/(d+2)}
\end{equation}
scalar heads, embedding width \(d+2+H_p\), and every attention-matrix entry
in the interval \([-1,1]\).  Equivalently, if
\(R_{\beta;\omega}(\mathbf F^\star)>0\), then for every
\(0<\varepsilon\leq R_{\beta;\omega}(\mathbf F^\star)\) the construction
attains joint discrete--continuous error at most \(\varepsilon\) using
\begin{equation}
\label{eq:quantitative-parameter-epsilon}
\begin{aligned}
 \log p_\varepsilon
 &\leq C_\beta
 \left(\frac{R_{\beta;\omega}(\mathbf F^\star)}{\varepsilon}\right)^{(d+2)/\beta}
 \log\!\left(2+\frac{R_{\beta;\omega}(\mathbf F^\star)}{\varepsilon}\right),\\
 H_\varepsilon
 &\leq C_\beta
 \left(\frac{R_{\beta;\omega}(\mathbf F^\star)}{\varepsilon}\right)^{(d+1)/\beta}.
\end{aligned}
\end{equation}
If \(R_{\beta;\omega}(\mathbf F^\star)=0\), the target factor is constant
and a final affine bias realizes it exactly.
\end{theorem}

The uniformity is over sequence lengths, not over input classes: the learned
transformer and \(R_{\beta;\omega}(\mathbf F^\star)\) may depend on \(\omega\),
whereas the exponent, constants, threshold \(p_0\), and attention-entry bound
do not.  The parameter count includes the full one-hidden-layer readout,
whose width \(M_p\) generally grows much faster than the head count \(H_p\).  The
denominator \(d+2=(d+1)+1\) combines \(O(K^{d+1})\) joint token--time probes
with an \(O(K\log K)\) stable-recovery cost in the logarithmic readout
complexity.  Target regularity first reduces the oscillation on each feature
cell to \(O(K^{-\beta})\); a merely Lipschitz shallow readout is fitted only
after this reduction, so it does not saturate the rate at \(\beta=1\).
The theorem counts real parameters; Corollary~\ref{cor:quantitative-finite-precision}
below also controls magnitudes and finite precision on a normalized target ball.

The following transfer principle shows why the discrete construction already
controls the continuous boundary. Its error identity does not require target
regularity or the quantitative theorem.

\begin{corollary}[Quantitative rate on the path boundary]
\label{cor:quantitative-temporal-closure}
Fix an admissible common modulus \(\omega\), let
\((F_n^\star)_{n\geq1}\) satisfy
Definition~\ref{def:compatible-family}, and write
\(\mathbf F^\star=((F_n^\star)_{n\geq1},F_\infty^\star)\) for its completed
family. Every shared causal transformer \(T\) covered by
Proposition~\ref{prop:temporal-closure}, together with its temporal
realization, satisfies
\begin{equation}
\label{eq:quantitative-error-equality}
E_\infty(T;\mathbf F^\star)
\leq E_{\mathrm{disc}}(T;\mathbf F^\star),
\qquad
\operatorname{Err}(T;\mathbf F^\star)
=E_{\mathrm{disc}}(T;\mathbf F^\star).
\end{equation}
Consequently, for a \(\beta\)-smooth family, the transformer
\(T_p\) constructed for Theorem~\ref{thm:quantitative-rates}, equipped with
the temporal attention defined by the same parameters, obeys
\begin{equation}
\label{eq:quantitative-joint-error-p}
\operatorname{Err}(T_p;\mathbf F^\star)
\leq
C_\beta R_{\beta;\omega}(\mathbf F^\star)
\left(\frac{\log\log p}{\log p}\right)^{\beta/(d+2)}.
\end{equation}
The bound \(\operatorname{par}(T_p)\leq p\), the head estimate
\eqref{eq:quantitative-heads-p}, and the accuracy form
\eqref{eq:quantitative-parameter-epsilon} are unchanged for this simultaneous
discrete--continuous estimate.
\end{corollary}

\begin{proof}
Fix \(x\in X_\infty\), put \(z^n\eqdef\mathcal S_nx\), and denote
\(e\eqdef E_{\mathrm{disc}}(T;\mathbf F^\star)\).  At every finite resolution,
\[
 \left\|
 \In_n\!\left(T(\phi_n(z^n))-F_n^\star(z^n)\right)
 \right\|_\infty
 \leq e,
\]
because affine interpolation is a convex combination of grid errors.
Proposition~\ref{prop:temporal-closure} sends the first interpolated term to
\(T_\infty(\phi_\infty(x))\).  By compatibility, the distance from the second
term to \(F_\infty^\star(\In_n\mathcal S_nx)\) tends to zero.  Since
\(\In_n\mathcal S_nx\to x\) and \(F_\infty^\star\) is continuous, the second
term therefore converges to \(F_\infty^\star(x)\).  Passing to the limit and
then taking the supremum over
\(x\) proves \(E_\infty\leq E_{\mathrm{disc}}\), hence
\eqref{eq:quantitative-error-equality}; no additional heads, parameters, or
approximation loss are introduced. Applying this identity to the discrete
construction in Appendix~\ref{app:proof-quantitative} yields
\eqref{eq:quantitative-joint-error-p} and the joint conclusion of
Theorem~\ref{thm:quantitative-rates}.
\end{proof}

On a normalized target ball, the same construction also admits a finite
binary description.  The next result charges precision separately from the
dense parameter count.

\begin{corollary}[Finite precision on a normalized target ball]
\label{cor:quantitative-finite-precision}
Under the assumptions of Theorem~\ref{thm:quantitative-rates}, suppose also
that \(R_{\beta;\omega}(\mathbf F^\star)\leq1\) and
\(\norm{f^\star}_\infty\leq1\).  There are constants \(C_\beta,p_0\),
depending only on \((d,d',\beta)\), such that, for every integer
\(p\geq p_0\), a shallow shared transformer \(\widetilde T_p\) satisfies
\begin{equation}
\label{eq:quantitative-rounded-error}
 \operatorname{par}(\widetilde T_p)\leq p,
 \qquad
 \operatorname{Err}(\widetilde T_p;\mathbf F^\star)
 \leq C_\beta
 \left(\frac{\log\log p}{\log p}\right)^{\beta/(d+2)}.
\end{equation}
Its attention entries lie in \([-1,1]\), all other parameters have magnitude
at most \(p+1\), and every parameter is dyadic and has a binary description
of at most \(C_\beta\log_2 p\) bits.  The architecture can be fixed across
the normalized target ball at each budget, and the total description length,
including its dimensions, is at most \(C_\beta p\log_2 p\).
\end{corollary}

The proof, given in Appendix~\ref{app:proof-quantitative}, tracks the ridge
coefficients and rounds them on a common dyadic grid.  This is a normalized
absolute-error guarantee, not the homogeneous bound proportional to
\(R_{\beta;\omega}\): an independently sized additive offset may require
additional bits when the seminorm is small, and an arbitrary constant target
need not have an exact finite-bit representation.  Only parameter storage
is quantized; network evaluation is still analyzed in exact arithmetic.
The result controls an
existence construction, not a claim of conditioning independent of the
parameter budget. Its bounded weights are also the approximation input to
the empirical-risk guarantee in Theorem~\ref{thm:generalization-main};
Appendix~\ref{app:proof-generalization} supplies the additional statistical
argument and sampling assumptions.

\subsection{Proof mechanism: quantitative resolution by causal probes}
\label{sec:quantitative-proof-sketch}

The proof follows the chain: attention probes recover moments, moments resolve
the history metric, target regularity controls the unresolved fibers, and a
counted ReLU readout approximates the resulting finite-dimensional map.  The
detailed estimates are in Appendix~\ref{app:proof-quantitative}; here are the
five steps.
\begin{enumerate}[leftmargin=*,itemsep=2pt,topsep=3pt]
    \item \textbf{Resolve causal histories by finitely many probes.}
    Put \(D=d+1\).  The direction grid and sampling scales in
    \eqref{eq:quant-direction-grid}--\eqref{eq:quant-scales} define
    \(\Phi_K(\mu,\mathfrak e(\mu))\).  It retains the current lifted token
    and uses
    \(H_K=2K(K+1)^{D-1}=O(K^{d+1})\) normalized
    log-Laplace-derivative samples.
    High-order interpolation in Lemma~\ref{lem:quant-stable-projected-moments}
    recovers every projected
    moment through degree \(K\), and tensor interpolation recovers the mixed
    moments.

    \item \textbf{Convert moments into metric resolution.}
    Lemma~\ref{lem:quant-jackson} turns moment control into control of smooth
    token--position test functions.
    Proposition~\ref{prop:quant-feature-resolution} then gives, for
    \(\tau_K\)-close feature vectors on the causal-history graph,
    \[
      d_{\beta_+}(\mu,\nu)\leq C_\beta K^{-\beta_+},
      \qquad \delta_e(\mu,\nu)\leq\tau_K,
    \]
    where \(\tau_K\) decays faster than any fixed inverse power of \(K\).

    \item \textbf{Control target oscillation on each feature cell.}
    Since \(\beta_+\beta_-=\beta\), Step~2 gives
    \(\Delta_\beta(\mu,\nu)\leq C_\beta K^{-\beta}\).
    Definition~\ref{def:causal-smoothness} therefore bounds the oscillation
    of \(f^\star\) on a feature cell by
    \(C_\beta R_{\beta;\omega}(\mathbf F^\star)K^{-\beta}\).  This is the
    step that uses the full order \(\beta\); the later readout need only be
    Lipschitz.

    \item \textbf{Realize the finite-dimensional readout.}
    The clipped infimal envelope \eqref{eq:quant-clipped-envelope} extends the
    approximately fiber-constant target with Lipschitz constant
    \(O(R_{\beta;\omega}(\mathbf F^\star)/\tau_K)\).
    Lemma~\ref{lem:quant-shallow-relu} mollifies this
    envelope and applies a dimension-explicit shallow-ReLU estimate, giving a
    one-hidden-layer readout of width \(W_K\) with
    \(
      \log W_K
      =O_\beta\!\left(
      \underbrace{K^{d+1}}_{\text{joint token--time features}}
      \underbrace{K\log(K+1)}_{\text{stable recovery}}
      \right)
    \).

    \item \textbf{Count parameters and invert the scale relation.}
    Lemma~\ref{lem:parallel-realization} computes all \(H_K\) probes exactly
    in one masked block, simultaneously on every finite grid and on the
    temporal boundary.  Counting the attention matrices and shallow
    readout gives
    \(
      \log\operatorname{par}(T_K)
      =O_\beta(K^{d+2}\log(K+1))
    \); inverting this relation gives
    \eqref{eq:quantitative-error-p} and
    \eqref{eq:quantitative-parameter-epsilon}.  Thus
    \(d+2=(d+1)+1\): the exponent comes from feature dimension and stability,
    not from the input modulus.  Proposition~\ref{prop:temporal-closure} then
    gives Corollary~\ref{cor:quantitative-temporal-closure} without changing
    the rate.
\end{enumerate}

\subsection{Interpreting the smoothness condition}
\label{sec:quantitative-smoothness-examples}

The definition is tailored to the information recovered by attention, but it
can be checked on familiar classes of functionals. We give endpoint-only
examples, a first-variation certificate, and Lipschitz readouts of smooth
history statistics.

\paragraph{Endpoint-only maps.}
Any map
\(
f^\star(\mu)\eqdef g(\mathfrak e(\mu))
\)
with \(g\) Lipschitz satisfies the definition for every fixed \(\beta>1\).
Membership at all such orders reflects the separate endpoint term, not
higher classical smoothness of \(g\).  The generic construction can exploit a
larger \(\beta\) by resolving the endpoint through increasingly fine feature
cells, paid for by a very wide readout. Membership in this scale therefore
does not assert classical differentiability of endpoint-only targets.

The endpoint is exactly determined by a continuous graph law, but its
extraction need not be Wasserstein-Lipschitz.  For a concrete example, fix
\(0<\alpha\leq1\), \(L>0\), and \(0<c\leq\min\{L,1\}\).  The paths
\[
 x_h(s)\eqdef c(s-1+h)_+^\alpha e_1,
 \qquad 0<h<1,
\]
and the zero path belong to the same \(\alpha\)-H\"older ball, where
\(e_1\) denotes the first token coordinate vector.  Their
terminal endpoints differ by \(ch^\alpha\), whereas coupling their terminal
prefix laws at equal times gives
\begin{equation}
\label{eq:endpoint-extraction-not-wasserstein-lipschitz}
 W_1(\mu_{\infty,x_h,1},\mu_{\infty,0,1})
 \leq\int_0^1\norm{x_h(s)}_\infty\,\dd s
 =\frac{c}{\alpha+1}h^{\alpha+1}.
\end{equation}
Thus retaining the endpoint coordinates supplies quantitative stability that does not
follow from the law's exact information alone.

\paragraph{A sufficient first-variation criterion.}
At high order, Lemma~\ref{lem:flat-derivative-certificate} provides a
standard sufficient criterion through first variations.  It suffices that
\(f^\star\) be the trace of an ambient extension that is Lipschitz in the
explicitly retained token--position endpoint and whose linear functional derivative is
uniformly \(C^\beta\), modulo constants, in its integration variable
\(E\).  Values of this extension away from the causal-history graph are
auxiliary: they certify smoothness but do not redefine the target.  For
example, consider the smooth
cylindrical functional
\begin{equation}
\label{eq:smooth-cylinder-example}
 \widetilde f(\mu,y)
 \eqdef\Upsilon\!\left(y,\int_E\varphi_1\,\dd\mu,\ldots,
                 \int_E\varphi_q\,\dd\mu\right)
\end{equation}
with jointly token--position \(C^\beta\) test functions and a smooth
finite-dimensional outer map.  Its functional derivative is a bounded
linear combination of the \(\varphi_j\)'s, and its dependence on \(y\) is
Lipschitz on the relevant compact set.  The ambient formulation is convenient
but stronger than necessary: the proof only uses the derivative identity and
its uniform bound along mixture segments joining pairs of laws in
\(\cM_\omega\), together with endpoint Lipschitz control at the corresponding
states.

\begin{remark}[Lipschitz readouts of smooth history statistics]
\label{rem:quantitative-statistic-examples}
The smooth cylindrical example can be verified without differentiating its
outer map.  If \(\varphi_1,\ldots,\varphi_q\in C^\beta(E)\),
\(\beta>1\), and \(\Psi\) is Lipschitz on the relevant compact product,
then
\[
 f(\mu)\eqdef\Psi\!\left(\mathfrak e(\mu),
          \int_E\varphi_1\,\dd\mu,\ldots,
          \int_E\varphi_q\,\dd\mu\right)
\]
is \(\beta\)-smooth. Indeed, each integral difference is at
most \(\norm{\varphi_j}_{C^\beta(E)}d_\beta(\mu,\nu)\), and the endpoint
is controlled separately.  Writing \(e_1'\) for the first output coordinate
vector, for example,
\(f(\mu)=\abs{\int_E x_1\,\dd\mu}\,e_1'\) belongs at every fixed high
order even though its outer absolute value is not differentiable.  This is
stability in the smooth-test metric, not classical smoothness of the entire
operator.

In sequence form, these examples read
\[
 F_n(z)_i=\Psi\!\left((z_i,i/n),
       \frac1i\sum_{j=1}^i\varphi_1(z_j,j/n),\ldots,
       \frac1i\sum_{j=1}^i\varphi_q(z_j,j/n)\right).
\]
The path realization replaces each normalized sum by
\(t^{-1}\int_0^t\varphi_j(x(s),s)\,\dd s\), with value
\(\varphi_j(x(0),0)\) at zero.
\end{remark}

\subsection{Logarithmic lower bounds on H\"older histories}
\label{sec:rate-optimality}

The upper rate is logarithmic in the parameter budget. To test whether such
slow decay reflects the complexity of long histories, we now specialize the
input modulus to
\begin{equation}
\label{eq:quant-holder-specialization}
 \omega_{\alpha,L}(r)\eqdef Lr^\alpha,
 \qquad 0<\alpha\leq1,\quad L>0,
\end{equation}
and write
\begin{equation}
\label{eq:quant-holder-classes}
 X_n^{\alpha,L}\eqdef X_n^{\omega_{\alpha,L}},
 \qquad
 X_\infty^{\alpha,L}\eqdef X_\infty^{\omega_{\alpha,L}}.
\end{equation}
Let \(\mathfrak X_{\alpha,L}\) and \(\cM_{\alpha,L}\) denote the corresponding
completion and prefix-law image, and abbreviate
\(R_{\beta;\alpha,L}\eqdef R_{\beta;\omega_{\alpha,L}}\).
For every fixed \(\beta>0\), we prove a logarithmic worst-case lower bound
when approximants have a finite binary description. This does not establish
the optimality of the upper exponent, or a lower bound for a count of
unrestricted real parameters. Appendix~\ref{app:proof-rate-optimality}
contains the proofs.

\paragraph{Target class and description budget.}
We use exactly the target class of Theorem~\ref{thm:quantitative-rates},
normalized to separate regularity from amplitude:
\begin{equation}
\label{eq:normalized-target-ball}
\mathfrak F_{\beta;\alpha,L}^{(1)}
\eqdef
\left\{
 \mathbf F^\star:
 R_{\beta;\alpha,L}(\mathbf F^\star)\leq1,\quad
 \sup_{\mu\in\cM_{\alpha,L}}\norm{f^\star(\mu)}\leq1
\right\}.
\end{equation}
Here every \(\mathbf F^\star\) is a continuously extendable causal family
and \(f^\star\) its canonical history factor. A \(B\)-bit description selects
at most \(2^B\) decoded maps, so the smallest worst-case error over all such
descriptions is
\begin{equation}
\label{eq:finite-code-distortion}
\mathcal A_B^{\mathrm{code}}(\beta;\alpha,L)
\eqdef
\inf_{\substack{\mathcal G\subset(\R^{d'})^{\cM_{\alpha,L}}\\
                    |\mathcal G|\leq2^B}}
\sup_{\mathbf F^\star\in\mathfrak F_{\beta;\alpha,L}^{(1)}}
\inf_{g\in\mathcal G}
\sup_{\mu\in\cM_{\alpha,L}}
\norm{g(\mu)-f^\star(\mu)}.
\end{equation}
The decoder is fixed across targets but otherwise unrestricted, so this
lower bound applies in particular to finite-precision transformers.
Only the description is charged; no training, sampling, or noise model is
assumed. The error is uniform on the same completed history space as in the
upper theorem.

\paragraph{Many distinguishable histories force slow approximation.}
The key is a multilevel construction of H\"older paths whose prefix laws
are well separated by the target metric \(\Delta_\beta\). Write
\begin{equation}
\label{eq:history-packing-number}
 \mathsf P_{\beta;\alpha,L}(r)
 \eqdef
 \max\left\{|S|:S\subset\cM_{\alpha,L},\
 \Delta_\beta(\mu,\nu)\geq r\text{ for distinct }\mu,\nu\in S\right\}.
\end{equation}
This number is finite by compactness. The construction below yields enough
histories to encode more admissible target patterns than a \(B\)-bit
description can distinguish.

\begin{proposition}[Logarithmic lower bound on H\"older histories]
\label{prop:holder-history-exponential-packing}
Fix \(d,d'\geq1\), \(0<\alpha\leq1\), \(L>0\), and \(\beta>0\).
There are \(c,r_0>0\), depending only on \((d,\alpha,L,\beta)\), such that
\begin{equation}
\label{eq:explicit-history-packing}
 \log\mathsf P_{\beta;\alpha,L}(r)
 \geq
 c r^{-1/\beta}\left[1+(1-\alpha)\log(1/r)\right]
 \qquad(0<r\leq r_0).
\end{equation}
Consequently, for some \(c>0\), all sufficiently large integers \(B\) satisfy
\begin{equation}
\label{eq:explicit-code-lower-rate}
 \mathcal A_B^{\mathrm{code}}(\beta;\alpha,L)
 \geq c
 \left(
 \frac{1+(1-\alpha)\log\log(B+e^e)}
      {\log(B+e^e)}
 \right)^\beta.
\end{equation}
\end{proposition}

The proof places many token levels in disjoint time cells and separates the
resulting paths by smooth tests. On an \(N\)-point history packing,
independent signed target values extend with controlled \(\beta\)-smooth seminorm,
producing \(2^N\) distinguishable targets; a \(B\)-bit description cannot
resolve them all when \(N>B\). The lower rate is
\((\log B)^{-\beta}\) for \(\alpha=1\) and
\((\log\log B/\log B)^\beta\) for each fixed \(\alpha<1\).
Thus logarithmically slow worst-case decay is unavoidable at finite
description length. Neither the packing exponent nor its logarithmic
dependence on \(\alpha\) are claimed to be sharp.

\paragraph{Comparison with the constructive rate.}
The finite-precision guarantee in
Corollary~\ref{cor:quantitative-finite-precision} uses \(O(p\log p)\) bits.
Choosing \(p\asymp_\beta B/\log B\) therefore gives, for all sufficiently
large \(B\),
\begin{equation}
\label{eq:quantitative-code-upper-rate}
 \mathcal A_B^{\mathrm{code}}(\beta;\alpha,L)
 \leq C_{d,d',\beta}
 \left(\frac{\log\log B}{\log B}\right)^{\beta/(d+2)}.
\end{equation}
The lower and upper bounds are both logarithmic, but their exponents
\(\beta\) and \(\beta/(d+2)\) do not match. In particular, the construction's
dimension penalty is not known to be necessary. The constants in the
lower bound may depend on \((\alpha,L)\), whereas the normalized upper-bound
constant does not. These are finite-description comparisons, not a matching
minimax result for unrestricted real-parameter networks.

\paragraph{Why VC dimension cannot transfer the bound.}
Let \(\mathfrak T_p^{\mathrm{sh}}\) denote the shallow shared transformers
in \eqref{eq:theorem-lift-v2}--\eqref{eq:shallow-transformer-v2} with at most
\(p\) dense parameters, each with its canonical temporal realization.
One might try to replace \(B\) by a capacity bound for \(p\) real
parameters.  This route fails: on the infinite-dimensional path domain, the
threshold class generated by a fixed-size shallow architecture already has
infinite VC dimension.  More precisely, let
\begin{equation}
\label{eq:transformer-sign-capacity}
 v_p(\alpha,L)
 \eqdef
 \operatorname{VCdim}\left\{
 \mu\mapsto\boldsymbol 1_{\{(\widehat T(\mu))_1>0\}}:
 T\in\mathfrak T_p^{\mathrm{sh}}
 \right\},
\end{equation}
where \(\widehat T:\cM_{\alpha,L}\to\R^{d'}\) is the causal-history factor
of the completed transformer.

\begin{proposition}[Infinite threshold capacity at fixed size]
\label{prop:infinite-threshold-capacity}
For every \(0<\alpha\leq1\) and \(L>0\),
\begin{equation}
\label{eq:fixed-size-infinite-vc}
 v_p(\alpha,L)=\infty
 \qquad\text{whenever}\qquad
 p\geq p_\star\eqdef 5(d+3)+1+2d'.
\end{equation}
This remains true if every attention and readout weight is restricted to
\([-1,1]\).
\end{proposition}

\begin{corollary}[Infinite threshold capacity on finite sequences]
\label{cor:finite-sequence-threshold-capacity}
Fix \(0<\alpha\leq1\), \(L>0\), and \(p\geq p_\star\).  For every
\(N\geq1\), there exist a common length \(n\) and \(N\) distinct
sequences in \(X_n^{\alpha,L}\) that are shattered by the terminal-output
threshold class
\[
 z\in X_n^{\alpha,L}\longmapsto
 \boldsymbol 1_{\{(T(\phi_n(z))_n)_1>0\}},
 \qquad T\in\mathfrak T_p^{\mathrm{sh}},
\]
even with every weight in \([-1,1]\).  In particular, this threshold class
has infinite VC dimension on the disjoint union of all finite input lengths.
\end{corollary}

The common length may grow with \(N\); the corollary does not assert
infinite capacity at one fixed resolution.  Its proof samples the finite
collection of smooth paths and preserves their strictly signed moments on
a sufficiently fine common grid.

The obstruction comes from selecting a real-valued Laplace parameter, not
from growing the width.  One
temporal head can threshold the Laplace moment
\(\int_0^1e^{\vartheta s}x(s)\,\dd s\), and the one-parameter family
\(\{e^{\vartheta s}:0<\vartheta<1\}\) shatters arbitrarily large finite
sets of smooth paths chosen for that set.  Hence no finite bound on
\(v_p(\alpha,L)\), polynomial or otherwise, can convert the packing result
into a lower bound for unrestricted real parameters.  Bounding the weights
does not fix this zero-threshold obstruction; even fixed-margin
fat-shattering remains infinite if the final output scale is unrestricted.
By contrast, a class described by \(B\)-bit strings has at most \(2^B\)
members and VC dimension at most \(B\), consistent with the finite-precision
guarantee in Corollary~\ref{cor:quantitative-finite-precision}.
A meaningful capacity comparison must therefore charge another resource,
such as parameter precision, output scale together with a margin, or input
discretization dimension.  The finite-code lower bound remains valid, but
optimality of the real-parameter upper rate remains open.

This zero-threshold obstruction does not preclude the regression guarantee
of Theorem~\ref{thm:generalization-main}. With both attention and readout
weights bounded, the real-valued predictor class admits finite uniform
covers at every positive accuracy, independently of context length
(Lemma~\ref{lem:gen-parameter-cover}). The shattering margins can tend
to zero, so threshold capacity and bounded squared-loss estimation measure
different notions of complexity.

\section{Proofs for quantitative rates and history-complexity bounds}
\label{app:proof-quantitative}

This appendix follows the five-step proof chain summarized in
Subsection~\ref{sec:quantitative-proof-sketch}.  Exact causal attention probes
first resolve finitely many moments; polynomial approximation then converts
this information into a metric resolution of causal histories.  The target
regularity turns that resolution into a \(K^{-\beta}\) oscillation bound, and a
counted shallow ReLU readout converts the latter into the stated parameter
rate.  We close with the optimality analysis on the H\"older specialization.

Put
\(D\eqdef d+1\), write \(y=(x,s)\in E=\Omega\times[0,1]\), and equip \(E\) with
\begin{equation}
\label{eq:quant-E-metric}
  d_E\bigl((x,s),(x',s')\bigr)
  \eqdef \norm{x-x'}_\infty+\abs{s-s'}.
\end{equation}
We denote by \(W_1\) the corresponding Wasserstein distance on \(\cP(E)\).
Changing between this norm and the Euclidean norm used elsewhere only changes
dimension-dependent constants; dependence on the fixed dimensions \(d,d'\) is
suppressed in constant subscripts.  In particular, all constants below are
independent of the sequence length and of the prefix at which the output is
read.  Fix the arbitrary admissible common modulus \(\omega\) from
Theorem~\ref{thm:quantitative-rates}, and work on
\(\cM\eqdef\cM_\omega\).  The modulus only selects this compact subset of the
ambient history space: neither the probe estimates nor the constants below
depend on \(\omega\).
We use the same causal-history factor \(f^\star:\cM\to\R^{d'}\) as in
\eqref{eq:quant-same-factor}.  Its quantitative chart is
\(\iota(\mu)=(\mu,\mathfrak e(\mu))\), defined in
\eqref{eq:causal-history-chart}; the second coordinate is already available
in the residual coordinates. On the ambient product \(\cP(E)\times E\), write
\begin{equation}
\label{eq:quant-product-metric}
  D_{\mathrm{prod}}\bigl((\mu,y),(\nu,y')\bigr)
  \eqdef W_1(\mu,\nu)+\norm{y-y'}_\infty.
\end{equation}
Its restriction to the graph of \(\mathfrak e\) is exactly
\(D_{\mathrm{caus}}\) from \eqref{eq:causal-chart-metric}.

\subsection{A finite family of exact masked-attention probes}

The first step constructs a finite feature map that is computed exactly by
masked attention and whose small cells have small diameter in the metrics
relevant to the target.

We use the box
\(
\mathsf Q=[-1,1]^d\times[0,1]\subset\R^D
\),
the restriction space \(C^\sigma(E)\), and the smooth-test metric
\(d_\sigma\) defined in
\eqref{eq:smooth-test-discrepancy-main}.

Fix \(K\geq3\).  The affine chart
\begin{equation}
\label{eq:quant-direction-grid}
  \mathscr D_K
  \eqdef
  \left\{(u,1):u\in\{0,1,\ldots,K\}^{D-1}\right\}
  \subset\R^D
\end{equation}
contains \((K+1)^{D-1}\) directions.  Choose constants
\(A_h,A_\tau\geq1\), to be fixed below, and set
\begin{equation}
\label{eq:quant-scales}
  h_K\eqdef(K+1)^{-A_h},
  \qquad
  \tau_K\eqdef(K+1)^{-A_\tau K}.
\end{equation}
Using the probe from \eqref{eq:laplace-probe-v2}, define the normalized law
features
\begin{equation}
\label{eq:quant-law-feature-map}
\begin{aligned}
  \widehat{\mathcal L}^{K}_{a,j}
  &\eqdef \mathcal L_{a/K,Kj h_K}
  =K^{-1}\mathcal L_{a,jh_K},\\
  \Psi_K(\mu)
  &\eqdef
  \Bigl(\widehat{\mathcal L}^{K}_{a,j}(\mu)\Bigr)_
  {a\in\mathscr D_K,\ 0\leq j<2K}
  \in\R^{2K(K+1)^{D-1}}.
\end{aligned}
\end{equation}
With the fixed threshold for \(A_h\) chosen below,
\(0\leq Kjh_K\leq2K^2h_K\leq1\).  Thus every normalized probe is an
admissible scalar attention statistic with exponent parameter in \([0,1]\).
We augment these probes by the token--position endpoint coordinates retained
unchanged by the residual connection:
\begin{equation}
\label{eq:quant-feature-map}
  \Phi_K(\mu,y)\eqdef\bigl(y,\Psi_K(\mu)\bigr)\in\R^{m_K},
  \qquad
  m_K\eqdef D+2K(K+1)^{D-1}\leq C_D K^D.
\end{equation}
These coordinates are normalized log-Laplace derivatives.  Their
normalization preserves all information while keeping both the probe values
and the attention matrices uniformly bounded in \(K\).  Unlike an unmasked
construction with a token-dependent query, this causal construction needs no
large bias to suppress query perturbations.

Two analytic ingredients connect the probes to metric resolution.  The first
recovers projected moments by high-order interpolation of the logarithmic
derivative; the second approximates smooth test functions by polynomials while
controlling their coefficients.

\begin{lemma}[Stable recovery of projected moments]
\label{lem:quant-stable-projected-moments}
Fix \(A_h\) above a threshold depending only on \(D\), and let \(h_K\) be as in
\eqref{eq:quant-scales}.  There are constants
\(C_{\mathrm{rec}}=C_{\mathrm{rec}}(D,A_h)>0\) and
\(C_{\mathrm{tr}}=C_{\mathrm{tr}}(D)>0\) such that, for every \(K\geq3\),
\(a\in\mathscr D_K\), \(\mu,\nu\in\cP(E)\), and \(1\leq r\leq K\),
\begin{multline}
\label{eq:quant-projected-moment-stability}
  \abs{\int_E \ip{a}{y}^r\,\dd(\mu-\nu)(y)}
  \\
  \leq
  (K+1)^{C_{\mathrm{rec}} K}
  \max_{0\leq j<2K}
  \abs{\mathcal L_{a,j h_K}(\mu)-\mathcal L_{a,j h_K}(\nu)}
  +(K+1)^{C_{\mathrm{tr}} K}h_K^K.
\end{multline}
\end{lemma}

\begin{proof}
\smallskip\noindent\emph{Analyticity and derivative bounds.}\hspace{0.25em}
Let \(S\eqdef\ip{a}{Y}\) for \(Y\sim\mu\), and write
\[
M_{\mu,a}(c)\eqdef\int_E e^{c\ip{a}{y}}\,\dd\mu(y),
\qquad
\ell_{\mu,a}\eqdef M_{\mu,a}'/M_{\mu,a}.
\]
Uniformly over \(a\in\mathscr D_K\) and \(y\in E\),
\(
\abs{\ip{a}{y}}\leq C_D K
\).
On a complex disk of radius \(c_D/K\),
\(
\abs{M_{\mu,a}(c)-1}\leq e^{C_D K\abs c}-1
\);
after decreasing \(c_D\), this is at most \(1/2\).  Thus
\(M_{\mu,a}\) has no zero there and \(\ell_{\mu,a}\) is analytic.  On a
smaller disk, \(\abs{\ell_{\mu,a}}\leq C_D K\); Cauchy's estimate therefore
gives, for \(0\leq q\leq K\),
\begin{equation}
\label{eq:quant-cauchy-cumulants}
 \sup_{0\leq c\leq2Kh_K}
 \abs{\ell_{\mu,a}^{(q)}(c)}
 \leq q!\,(C_D K)^{q+1}
 \leq (K+1)^{C_D(q+1)}.
\end{equation}
Here \(A_h\) is chosen once so that \([0,2Kh_K]\) remains in a smaller
zero-free disk for all \(K\geq3\).  The same estimate holds for \(\nu\).

\smallskip\noindent\emph{High-order interpolation.}\hspace{0.25em}
Put \(N\eqdef2K\), and let \(I_{N,h_K}\ell\) be the polynomial of degree at
most \(N-1\) interpolating \(\ell\) at \(0,h_K,\ldots,(N-1)h_K\).
The cardinal polynomials on the integer grid are
\[
 \lambda_j(v)\eqdef\prod_{\substack{0\leq m<N\\m\neq j}}
 \frac{v-m}{j-m},
 \qquad
 \sum_{j=0}^{N-1}\norm{\lambda_j}_{\ell^1(\mathrm{coeff})}
 =2^N-1\leq2^N.
\]
Indeed, the nonnegative roots give alternating coefficient signs, so
\(\norm{\lambda_j}_{\ell^1(\mathrm{coeff})}
=\abs{\lambda_j(-1)}=\binom{N}{j+1}\); summing gives the identity.
Since
\(I_{N,h_K}\ell(c)=\sum_{j=0}^{N-1}\ell(jh_K)\lambda_j(c/h_K)\),
perturbing all samples by at most \(\eta\) changes its \(q\)-th derivative
at zero by at most \(q!h_K^{-q}2^N\eta\).

Write \(\ell(c)=\sum_{q\geq0}t_qc^q\) near zero.  The analytic disk bound
gives \(\abs{t_q}\leq C_DK(C_DK)^q\) for every \(q\geq0\), not only the
derivative orders through \(K-1\) that we seek to recover.  Choose the fixed
threshold for \(A_h\) so that \(C_DKNh_K\leq1/2\) for every \(K\geq3\).
Summing the geometric tail of the Taylor polynomial \(T_{N-1}\ell\) then
gives
\[
 \max_{0\leq j<N}\abs{\ell(jh_K)-T_{N-1}\ell(jh_K)}
 \leq C_DK(C_DKNh_K)^N.
\]
The factor from summing the tail is absorbed into \(C_D\).
Interpolation is exact on \(T_{N-1}\ell\), so
for \(0\leq q<K\),
\begin{equation}
\label{eq:quant-interpolation-bias}
 \abs{\ell^{(q)}(0)-(I_{N,h_K}\ell)^{(q)}(0)}
 \leq q!h_K^{-q}2^N C_DK(C_DKNh_K)^N
 \leq (K+1)^{C_2K}h_K^K.
\end{equation}
Here \(C_2\) depends only on \(D\): use \(N=2K\), \(q<K\), and
\(h_K^{N-q}\leq h_K^K\).  The \(2K\) samples leave a full power
\(h_K^K\) after differentiation through order \(K-1\); a polynomially
small spacing now suffices.  This interpolation is used only to prove
feature resolution, not as an additional layer of the transformer.

Set
\(\eta\eqdef\max_{0\leq j<2K}
\abs{\ell_{\mu,a}(jh_K)-\ell_{\nu,a}(jh_K)}\).
The cumulants are \(\kappa_j\eqdef\ell_{\mu,a}^{(j-1)}(0)\), and the preceding
sample and bias estimates give, for \(1\leq j\leq r\leq K\),
\begin{equation}
\label{eq:quant-cumulant-stability}
 \abs{\kappa_j(\mu)-\kappa_j(\nu)}
 \leq
 (j-1)!h_K^{-(j-1)}2^{2K}\eta
 +2(K+1)^{C_2K}h_K^K.
\end{equation}
Because \(h_K=(K+1)^{-A_h}\), the right-hand side is bounded by
\begin{equation}
\label{eq:quant-cumulant-stability-coarse}
 (K+1)^{C_1K}\eta+(K+1)^{C'_2K}h_K^K,
\end{equation}
where \(C_1\) may depend on \(A_h\), while \(C'_2\) depends only on \(D\).

For \(\lambda\in\{\mu,\nu\}\), write
\(\kappa(\lambda)\eqdef(\kappa_1(\lambda),\ldots,\kappa_r(\lambda))\).

\smallskip\noindent\emph{From cumulants to moments.}\hspace{0.25em}
Finally,
\begin{equation}
\label{eq:quant-log-derivative}
  \int_E \ip{a}{y}^r\,\dd\mu(y)
  =B_r\bigl(\kappa_1,\ldots,\kappa_r\bigr),
\end{equation}
where \(B_r\) is the complete exponential Bell polynomial.  The identity
\(
\partial B_r/\partial\kappa_j=\binom rjB_{r-j}
\) will control its gradient.  More explicitly,
\begin{equation}
\label{eq:quant-bell-expansion}
 B_s(x_1,\ldots,x_s)
 =\sum_{\substack{k_1,\ldots,k_s\geq0\\
                   \sum_{j=1}^s j k_j=s}}
 s!\prod_{j=1}^s
 \frac{1}{k_j!}\left(\frac{x_j}{j!}\right)^{k_j}.
\end{equation}
Equation~\eqref{eq:quant-cauchy-cumulants} gives, after enlarging a
dimension-dependent constant,
\(
 \abs{\kappa_j}\leq j!(C_DK)^j
\)
for \(j\leq K\), and the same bound holds along the segment between the two
cumulant vectors.  In each summand of
\eqref{eq:quant-bell-expansion}, the powers of \(C_DK\) total exactly \(s\).
The number of integer partitions, the factor \(s!\), and the remaining
combinatorial sum are all bounded by \((K+1)^{C K}\) for \(s\leq K\).
Using positivity of the coefficients and the derivative identity therefore
gives
\begin{equation}
\label{eq:quant-bell-gradient-bound}
 \sup_{0\leq t\leq1}
 \norm{\nabla B_r((1-t)\kappa(\mu)+t\kappa(\nu))}_1
\leq (K+1)^{C_3K}.
\end{equation}
The mean-value theorem, together with
\eqref{eq:quant-cumulant-stability-coarse}, now gives
\eqref{eq:quant-projected-moment-stability} after enlarging
\(C_{\mathrm{rec}}\) and \(C_{\mathrm{tr}}\).
\end{proof}

\begin{lemma}[Jackson approximation with coefficient control]
\label{lem:quant-jackson}
There is \(C_0>1\), depending only on \(D\), such that, for every fixed
\(\sigma>0\), \(K\geq3\), and \(g\in C^\sigma(E)\), there is a polynomial
\(
p_K(y)=\sum_{\abs{\kappa}\leq K}c_\kappa y^\kappa
\)
satisfying
\begin{equation}
\label{eq:quant-Jackson}
  \norm{g-p_K}_{L^\infty(E)}
  \leq C_\sigma K^{-\sigma}\norm{g}_{C^\sigma(E)},
  \qquad
  \sum_{\abs{\kappa}\leq K}\abs{c_\kappa}
  \leq C_\sigma C_0^K\norm{g}_{C^\sigma(E)}.
\end{equation}
\end{lemma}

\begin{proof}
\smallskip\noindent\emph{Jackson approximation.}\hspace{0.25em}
If \(g=0\), take \(p_K=0\).  Otherwise, choose an extension
\(\widetilde g\) to \(\mathsf Q\) with
\(
\norm{\widetilde g}_{C^\sigma(\mathsf Q)}
\leq2\norm{g}_{C^\sigma(E)}
\).  After mapping the last coordinate
affinely to \([-1,1]\), apply bounded one-dimensional Jackson approximation
operators of order exceeding the fixed \(\sigma\), successively in all
\(D\) coordinates, with coordinate degree \(\lfloor K/D\rfloor\).
Their bounds may depend on \(\sigma\); a fixed low-order positive operator
would not suffice at arbitrary smoothness.  This gives the standard tensor
multivariate Jackson estimate
\citep{devore1993constructive,totik2020multivariate}.  It provides a
polynomial of total degree at most \(K\) with the first stated error above a
threshold depending on \((D,\sigma)\).  The finitely many smaller
\(K\geq3\) are handled by a constant polynomial after enlarging
\(C_\sigma\).  In all cases,
\[
 \norm{p_K}_{L^\infty(\mathsf Q)}
 \leq \norm{\widetilde g}_{L^\infty(\mathsf Q)}
      +\norm{\widetilde g-p_K}_{L^\infty(\mathsf Q)}
 \leq C_\sigma\norm{g}_{C^\sigma(E)}.
\]

\smallskip\noindent\emph{Coefficient control.}\hspace{0.25em}
Write \(m\eqdef\lfloor K/D\rfloor\).  The tensor Jackson construction has
coordinate degree at most \(m\), so after the affine change in the last
coordinate it has an expansion
\[
 p_K(y)=\sum_{j\in\{0,\ldots,m\}^D}\widehat c_j
 \prod_{\ell=1}^D T_{j_\ell}(\widetilde y_\ell),
\]
where \(\widetilde y_D=2y_D-1\) and the other coordinates are unchanged.
The tensor Chebyshev integral formula gives
\(
\abs{\widehat c_j}\leq2^D\norm{p_K}_{L^\infty(\mathsf Q)}
\).
The sum of the absolute monomial coefficients of a degree-\(s\) Chebyshev
polynomial, including after the affine change in the last coordinate, is at
most \(C^s\).  Hence each product above has monomial coefficient norm at most
\(C^{\abs j}\leq C^K\), and it has total degree
\(\abs j\leq Dm\leq K\).  There are \((m+1)^D\leq(K+1)^D\) tensor
coefficients.  Since \(D\) is fixed and \(K\geq3\), this polynomial factor is
absorbed into \(C_0^K\), with \(C_0\) independent of \(\sigma\).  This proves
the second inequality.
\end{proof}

\begin{proposition}[Quantitative resolution by exact causal probes]
\label{prop:quant-feature-resolution}
For every prescribed \(M_0>0\), constants
\(
A_h=A_h(D,M_0)
\)
and
\(
A_\tau=A_\tau(D,M_0)
\)
can be chosen so that, for all
\(K\geq3\), \(\mu,\nu\in\cP(E)\), and \(y,y'\in E\),
\begin{equation}
\label{eq:quant-feature-moment}
  \begin{gathered}
    \norm{\Phi_K(\mu,y)-\Phi_K(\nu,y')}_\infty\leq\tau_K\\
    {}\Longrightarrow
    \max_{1\leq\abs{\kappa}\leq K}
    \abs{\int_E w^\kappa\,\dd(\mu-\nu)(w)}
    \leq (K+1)^{-M_0K}.
  \end{gathered}
\end{equation}
There is a threshold \(M_\star=M_\star(D)\) such that, whenever
\(M_0\geq M_\star\), this implies
\begin{align}
\label{eq:quant-W1-resolution}
  D_{\mathrm{prod}}\bigl((\mu,y),(\nu,y')\bigr)&\leq C K^{-1},\\
\label{eq:quant-smooth-resolution}
  d_\sigma(\mu,\nu)+\norm{y-y'}_\infty
  &\leq C_\sigma K^{-\sigma}
  \qquad (\sigma>0\text{ fixed}).
\end{align}
After fixing one \(M_0\geq M_\star\), the same feature family and the same
\(h_K,\tau_K\) work for every fixed \(\sigma\); only \(C_\sigma\) depends
on \(\sigma\).
\end{proposition}

\begin{proof}
We propagate feature closeness in three stages: normalized probes control
projected moments, interpolation recovers mixed moments, and Jackson
approximation controls smooth tests.

\smallskip\noindent\emph{Projected moments.}\hspace{0.25em}
Put \(\Lambda_D\eqdef4^{D-1}\), and choose an integer
\(C_{\mathrm{int}}(D)\geq0\) such that, for every
\(K\geq3\),
\begin{equation}
\label{eq:quant-interpolation-loss}
 2\Lambda_D^K(K+1)^{-(C_{\mathrm{int}}(D)+1)K}\leq1,
\end{equation}
and set \(M\eqdef M_0+C_{\mathrm{int}}(D)+1\).  First choose
\(
A_h\geq M+C_{\mathrm{tr}}(D)+1
\)
and above the threshold in Lemma~\ref{lem:quant-stable-projected-moments}.
The truncation term is then at most \((K+1)^{-MK}\).  With this \(A_h\)
fixed, choose
\(
A_\tau\geq C_{\mathrm{rec}}(D,A_h)+M+1
\).
The interpolation amplification will then absorb both \(\tau_K\) and the
normalization below.  If the two normalized feature vectors are
\(\tau_K\)-close, \eqref{eq:quant-law-feature-map} gives
\[
 \max_{a\in\mathscr D_K,\ 0\leq j<2K}
 \abs{\mathcal L_{a,jh_K}(\mu)-\mathcal L_{a,jh_K}(\nu)}
 \leq K\tau_K.
\]
The additional factor \(K\) is absorbed by the extra unit in the choice of
\(A_\tau\), since
\[
 (K+1)^{C_{\mathrm{rec}}K}K\tau_K
 \leq (K+1)^{1-(M+1)K}
 \leq (K+1)^{-MK}.
\]
Lemma~\ref{lem:quant-stable-projected-moments} therefore bounds every
projected moment through degree \(K\) by
\(2(K+1)^{-MK}\).

\smallskip\noindent\emph{Mixed moments.}\hspace{0.25em}
For fixed \(r\), the projected moment is the homogeneous polynomial
\begin{equation}
\label{eq:quant-projected-polynomial}
  a\longmapsto
  \int\ip{a}{y}^r\,\dd(\mu-\nu)(y)
  =\sum_{\abs{\kappa}=r}
  \binom r\kappa a^\kappa\int y^\kappa\,\dd(\mu-\nu)(y).
\end{equation}
Set
\(
P_r(a)\eqdef\int\ip{a}{y}^r\,\dd(\mu-\nu)(y)
\)
and \(Q_r(u)\eqdef P_r((u,1))\).  The polynomial \(Q_r\) has total degree at most
\(r\), and tensor-product Lagrange interpolation on
\(u\in\{0,\ldots,r\}^{D-1}\) gives the exact identity
\[
 Q_r(u)=
 \sum_{j\in\{0,\ldots,r\}^{D-1}}Q_r(j)
 \prod_{\ell=1}^{D-1}\lambda_{j_\ell}(u_\ell).
\]
The required grid values are available because \(r\leq K\).  In one variable
the cardinal basis is
\[
 \lambda_j(v)\eqdef\prod_{\substack{0\leq m\leq r\\m\neq j}}
 \frac{v-m}{j-m},
 \qquad
 \norm{\lambda_j}_{\ell^1(\mathrm{coeff})}
 =\binom{r+1}{j+1},
 \qquad
 \sum_{j=0}^r\norm{\lambda_j}_{\ell^1(\mathrm{coeff})}
 \leq2^{r+1}.
\]
Consequently,
\[
 \norm{Q_r}_{\ell^1(\mathrm{coeff})}
 \leq
 \max_{j\in\{0,\ldots,r\}^{D-1}}\abs{Q_r(j)}
 \left(\sum_{j=0}^r
 \norm{\lambda_j}_{\ell^1(\mathrm{coeff})}\right)^{D-1}
 \leq \Lambda_D^r\max_j\abs{Q_r(j)}.
\]
For every \(\kappa'\in\N_0^{D-1}\) with \(\abs{\kappa'}\leq r\), the
coefficient of \(u^{\kappa'}\) is
\(
\binom r{\kappa',r-\abs{\kappa'}}
\int y^{(\kappa',r-\abs{\kappa'})}\,\dd(\mu-\nu)(y),
\)
while coefficients with \(\abs{\kappa'}>r\) vanish.  Thus every degree-\(r\)
mixed moment is recovered.  Division by its multinomial coefficient cannot
enlarge the bound.  Multiplying the projected-moment estimate by
\(\Lambda_D^K\)
and using \eqref{eq:quant-interpolation-loss} gives the right-hand side of
\eqref{eq:quant-feature-moment}.  This proves
\eqref{eq:quant-feature-moment}.

\smallskip\noindent\emph{Smooth-test and Wasserstein resolution.}\hspace{0.25em}
Apply Lemma~\ref{lem:quant-jackson} to a test function
\(g\in C^\sigma(E)\).  Integrating its polynomial approximation against
\(\mu-\nu\), and using
\eqref{eq:quant-feature-moment}, yields
\begin{equation}
\label{eq:quant-test-estimate}
  \abs{\int_E g\,\dd(\mu-\nu)}
  \leq C_\sigma\norm{g}_{C^\sigma(E)}
  \left(K^{-\sigma}+C_0^K(K+1)^{-M_0K}\right).
\end{equation}
The constant monomial cancels because \(\mu\) and \(\nu\) have equal total
mass.
Take \(M_\star=1\).  For every fixed \(\sigma>0\) and every
\(M_0\geq M_\star\),
\(
C_0^K(K+1)^{-M_0K}\leq C_\sigma K^{-\sigma}
\)
uniformly for \(K\geq3\); importantly, the feature scales do not depend on
\(\sigma\).  This proves \eqref{eq:quant-smooth-resolution}.  For
\eqref{eq:quant-W1-resolution}, fix \(y_0\in E\), subtract the value at
\(y_0\) from each unit-Lipschitz Kantorovich potential, and take its McShane
extension to \(\mathsf Q\), using the metric in
\eqref{eq:quant-E-metric}.  Its \(C^{0,1}(\mathsf Q)\) norm is bounded
uniformly in the potential, up to fixed norm-equivalence constants.
Kantorovich--Rubinstein duality gives the claim. In both cases, the retained
coordinate block gives
\(\norm{y-y'}_\infty\leq\tau_K\), which is smaller than every fixed power of
\(K^{-1}\).
\end{proof}

\subsection{A first-variation certificate for high-order regularity}

For \(\beta>1\), the theorem combines order-\(\beta\) smooth-test stability in
the history law with first-order stability in the retained endpoint coordinates. Here
\(C^\beta(E)\) means joint regularity in token and normalized position; it is
not merely token regularity or an unspecified notion of functional
smoothness.  The next lemma turns a familiar first-variation assumption into
exactly the seminorm required by the theorem.

\begin{lemma}[Smooth first variation implies the \(\beta\)-smooth target condition]
\label{lem:flat-derivative-certificate}
Fix \(\beta>1\).  Let \(\mathbf F^\star\) be a completed causal family
with canonical factor \(f^\star:\cM\to\R^{d'}\).  Suppose there is a map
\(\widetilde f:\cP(E)\times E\to\R^{d'}\) satisfying
\begin{equation}
\label{eq:high-beta-extension-trace}
 \widetilde f(\mu,\mathfrak e(\mu))=f^\star(\mu),
 \qquad \mu\in\cM,
\end{equation}
and, for some \(B_\beta\geq0\),
\begin{equation}
\label{eq:current-token-lipschitz-high-beta}
 \sup_{\mu\in\cP(E)}\sup_{y\neq y'}
 \frac{\norm{\widetilde f(\mu,y)-\widetilde f(\mu,y')}}
 {\norm{y-y'}_\infty}
 \leq B_\beta.
\end{equation}
Assume each coordinate has a jointly Borel measurable linear functional
derivative, with all displayed integrals well-defined, such that for
\(\mu_u=(1-u)\mu+u\nu\),
\begin{equation}
\label{eq:flat-derivative-causal}
 \widetilde f_\ell(\nu,y)-\widetilde f_\ell(\mu,y)
 =\int_0^1\!\int_E
 \frac{\delta\widetilde f_\ell}{\delta\mu}(\mu_u,y,v)
 \,\dd(\nu-\mu)(v)\,\dd u.
\end{equation}
Define
\begin{equation}
\label{eq:homogeneous-Cbeta}
 \norm{\varphi}_{\dot C^\beta(E)}
 \eqdef \inf_{c\in\R}\norm{\varphi+c}_{C^\beta(E)},
\end{equation}
and suppose
\begin{equation}
\label{eq:flat-derivative-bound-causal}
 \left(
 \sum_{\ell=1}^{d'}
 \sup_{(\mu,y)\in\cP(E)\times E}
 \left\|
 \frac{\delta\widetilde f_\ell}{\delta\mu}(\mu,y,\cdot)
 \right\|_{\dot C^\beta(E)}^2
 \right)^{1/2}
 \leq B_\beta.
\end{equation}
Then
\(
R_{\beta;\omega}(\mathbf F^\star)
\leq B_\beta
\).
\end{lemma}

\begin{proof}
\smallskip\noindent\emph{Quotient duality.}\hspace{0.25em}
For every \(\varphi\in C^\beta(E)\),
\begin{equation}
\label{eq:quant-quotient-duality}
 \abs{\int_E\varphi\,\dd(\nu-\mu)}
 \leq \norm{\varphi}_{\dot C^\beta(E)}d_\beta(\mu,\nu),
\end{equation}
because additive constants integrate to zero.  Thus
\(\dot C^\beta(E)\) is the relevant quotient seminorm; no converse duality
identity is needed.  For \(\mu,\nu\in\cM\), insert
\(\widetilde f(\nu,\mathfrak e(\mu))\) between
\(\widetilde f(\nu,\mathfrak e(\nu))\) and
\(\widetilde f(\mu,\mathfrak e(\mu))\).  The first-variation identity, the
Euclidean aggregate in
\eqref{eq:flat-derivative-bound-causal}, and
\eqref{eq:quant-quotient-duality} control the measure part by
\(B_\beta d_\beta(\mu,\nu)\), while
\eqref{eq:current-token-lipschitz-high-beta} controls the remaining part by
\(
B_\beta\norm{\mathfrak e(\mu)-\mathfrak e(\nu)}_\infty
\).
Using the trace identity \eqref{eq:high-beta-extension-trace} and
Definition~\ref{def:causal-smoothness} proves the claim.
\end{proof}

The ambient formulation makes the certificate easy to state and verify, but
the proof only uses \eqref{eq:flat-derivative-causal} and
\eqref{eq:flat-derivative-bound-causal} along mixture segments
\((1-u)\mu+u\nu\) joining \(\mu,\nu\in\cM_\omega\).  Likewise, the endpoint
estimate is needed only at the states appearing in the preceding split.

\subsection{A dimension-explicit shallow ReLU readout}

The feature dimension grows with the resolution scale, so a qualitative
universal-approximation theorem is insufficient for counting parameters.  The
next lemma makes the dimension dependence explicit while retaining one hidden
ReLU layer.

\begin{lemma}[Counted shallow approximation of a Lipschitz map]
\label{lem:quant-shallow-relu}
Fix \(B_0\geq1\).  There is \(C_{B_0}>0\) such that the following holds for
every \(m\geq1\).  Let
\(g:[-B_0,B_0]^m\to\R\) satisfy
\begin{equation}
\label{eq:quant-shallow-data}
 \norm{g}_\infty\leq M,
 \qquad
 \abs{g(u)-g(v)}\leq A\norm{u-v}_\infty,
\end{equation}
with \(A,M\geq0\).  If \(M=0\), the zero network is exact.  If \(M>0\), then
for every \(0<\delta\leq M\), there is a standard one-hidden-layer ReLU
network \(\mathcal R\) of width \(W\) such that
\begin{equation}
\label{eq:quant-shallow-error}
 \norm{g-\mathcal R}_\infty\leq\delta
\end{equation}
and
\begin{equation}
\label{eq:quant-shallow-width}
 \log(2+W)
 \leq C_{B_0}\left[
 (m+2)\log\!\left(2+\frac{A}{\delta}\right)
 +m+\log\!\left(2+\frac{M}{\delta}\right)
 +\log(m+1)
 \right].
\end{equation}
When \(A=0\), a constant network gives the conclusion.
\end{lemma}

\begin{proof}
It remains to consider \(A,M>0\).
\smallskip\noindent\emph{Extension and mollification.}\hspace{0.25em}
Let \(\widetilde g\) be the McShane extension of \(g\) to \(\R^m\), clipped
to \([-M,M]\).  It still agrees with \(g\) on the cube and is
\(A\)-Lipschitz for \(\norm{\cdot}_\infty\).  Choose a smooth tensor-product
cutoff that equals one on
\([-B_0-1,B_0+1]^m\), is supported in
\([-B_0-2,B_0+2]^m\), and multiply \(\widetilde g\) by it to obtain a
compactly supported function \(G\).  Then
\begin{equation}
\label{eq:quant-extension-L1}
 \norm{G}_{L^1(\R^m)}\leq M C_0^m
\end{equation}
for a constant \(C_0=C_0(B_0)\).

Let \(\zeta\in C_c^\infty([-1,1])\) be nonnegative with integral one, put
\(\zeta_m\eqdef\zeta^{\otimes m}\), and mollify at scale
\begin{equation}
\label{eq:quant-mollification-scale}
 h\eqdef\min\left\{1,\frac{\delta}{3A}\right\},
 \qquad G_h\eqdef G*\zeta_{m,h}.
\end{equation}
On \([-B_0,B_0]^m\), the convolution only sees the region where the cutoff
is one.  Hence
\begin{equation}
\label{eq:quant-mollification-error}
 \norm{G_h-g}_\infty\leq Ah\leq\delta/3.
\end{equation}
\smallskip\noindent\emph{Fourier complexity.}\hspace{0.25em}
The mollified function \(G_h\) lies in \(C_c^\infty(\R^m)\), so it has the
Fourier representation required below.  For a fixed Fourier-transform
convention, define
\(
v_{G_h,2}\eqdef\int_{\R^m}\norm{\omega}_1^2
\abs{\widehat G_h(\omega)}\,\dd\omega
\).
Since \(\abs{\widehat G}\leq C\norm{G}_{L^1}\), a change of variables and
the tensor structure give
\begin{equation}
\label{eq:quant-spectral-moment}
 v_{G_h,2}
 \leq
 C M m^2 C_1^m h^{-m-2},
\end{equation}
where
\(
C_1=C_1(B_0,\zeta)\geq3
\); indeed,
\(
\int\norm{\xi}_1^2\prod_{j=1}^m
\abs{\widehat\zeta(\xi_j)}\,\dd\xi\leq m^2 C_1^m
\).
All Fourier-normalization constants are absorbed into \(C_1^m\).

\smallskip\noindent\emph{Counted ridge approximation.}\hspace{0.25em}
To record the rescaling explicitly, set
\(
G_h^{(B_0)}(x)\eqdef G_h(B_0x)
\).  This function is smooth and compactly supported, so Fourier inversion
applies, and a change of variables gives
\(
v_{G_h^{(B_0)},2}=B_0^2v_{G_h,2}
\)
up to the fixed Fourier convention.  The factor \(B_0^2\) is harmless because
\(B_0\) is fixed.  Theorem~2 of
\citet{klusowski2018approximation}, applied on \([-1,1]^m\) and then composed
with the inverse scaling, supplies a one-hidden-layer ReLU ridge network with
\(W\) units and an affine term whose uniform error on
\([-B_0,B_0]^m\) is at most
\begin{equation}
\label{eq:quant-KB-bound}
 C v_{G_h,2}\sqrt{m+\log W}\,W^{-1/2-1/m}.
\end{equation}
Put \(x\eqdef v_{G_h,2}/\delta\) and take
\(
W\eqdef\left\lceil[C_*(m+1)(1+x)]^4\right\rceil
\), where \(C_*=C_*(B_0)\) is sufficiently large and independent of
\(m,g,A,M,\delta\).  Since
\(W^{-1/m}\leq1\), substitution in \eqref{eq:quant-KB-bound} gives
\begin{equation}
\label{eq:quant-KB-width-substitution}
 \frac{C v_{G_h,2}\sqrt{m+\log W}\,W^{-1/2-1/m}}{\delta}
 \leq
 \frac{C x\sqrt{m+\log W}}
 {C_*^2(m+1)^2(1+x)^2}.
\end{equation}
The ceiling obeys
\[
 \log W\leq \log2+4\log C_*+4\log(m+1)+4\log(1+x).
\]
Consequently the right-hand side of
\eqref{eq:quant-KB-width-substitution} is bounded by a constant times
\[
 C_*^{-2}\frac{x}{(1+x)^2}
 \frac{\sqrt{m+\log C_*+\log(m+1)+\log(1+x)}}{(m+1)^2}.
\]
This tends to zero uniformly over \(m\geq1\) and \(x\geq0\) as
\(C_*\to\infty\): besides the elementary polynomial bounds, one uses
\(
\sup_{x\geq0}x(1+x)^{-2}\sqrt{\log(1+x)}<\infty
\).
Choose \(C_*\) so that the bound is at most \(2/3\).  Thus
\eqref{eq:quant-KB-bound} is at most \(2\delta/3\).
Combining this with
\eqref{eq:quant-mollification-error} proves
\eqref{eq:quant-shallow-error}.  If the pointwise MLP has no direct affine
skip, its affine term is represented by the additional units
\(u_j=\operatorname{ReLU}(u_j)-\operatorname{ReLU}(-u_j)\), costing at most
\(2m\) units.  Replace \(W\) by \(W+2m\) and relabel; the logarithmic bound
is unchanged.  Finally, \eqref{eq:quant-spectral-moment} and
\(
h^{-1}\leq 1+3A/\delta
\)
give
\[
 \log(1+x)
 \leq C_{B_0}
 +\log\!\left(2+\frac{M}{\delta}\right)
 +2\log(m+1)+m\log C_1
 +(m+2)\log\!\left(2+\frac{A}{\delta}\right).
\]
Combining this inequality with the definition of \(W\), and absorbing fixed
and lower-order terms, proves \eqref{eq:quant-shallow-width} after enlarging
\(C_{B_0}\).
\end{proof}

\subsection{Quantitative readout and total parameter count}

We now assemble the proof.  Feature resolution first controls the target
oscillation on each feature cell; a shallow readout then approximates the
resulting feature-space map; finally, we count the full construction.  The
same parameters are shared across all sequence lengths, and the final budget
includes every attention and readout coefficient.

For the remainder of the proof, take \(M_0=1\) in
Proposition~\ref{prop:quant-feature-resolution}, as permitted by its proof, and
fix the corresponding \(A_h,A_\tau\).  These scale constants depend only on
\(D\) and work for the \(\sigma=\beta_+\) smooth-test estimate; in particular,
they do not depend on
\(\omega\) or the target.

\paragraph{Target resolution in feature space.}
Let \(\mu,\nu\in\cM\), set \(y\eqdef\mathfrak e(\mu)\) and
\(y'\eqdef\mathfrak e(\nu)\), and suppose that
\(\Phi_K(\mu,y)\) and \(\Phi_K(\nu,y')\) are \(\tau_K\)-close in
the supremum norm. Applying \eqref{eq:quant-smooth-resolution} with
\(\sigma=\beta_+\), and using the retained endpoint coordinates, gives
\[
 \Delta_\beta(\mu,\nu)
 =d_{\beta_+}(\mu,\nu)^{\beta_-}
   +\norm{y-y'}_\infty^{\beta_-}
 \leq C_\beta K^{-\beta_+\beta_-}+\tau_K^{\beta_-}
 \leq C_\beta K^{-\beta}.
\]
Here \(\beta_+\beta_-=\beta\), and \(\tau_K\) decays faster than every
fixed inverse power of \(K\).  Definition~\ref{def:causal-smoothness} yields
\begin{equation}
\label{eq:quant-target-resolution}
  \norm{\Phi_K(\mu,\mathfrak e(\mu))-
  \Phi_K(\nu,\mathfrak e(\nu))}_\infty\leq\tau_K
  \quad\Longrightarrow\quad
  \norm{f^\star(\mu)-f^\star(\nu)}
  \leq C_\beta R_{\beta;\omega}(\mathbf F^\star)K^{-\beta}.
\end{equation}
This is the decisive regularity step: a feature cell of radius \(\tau_K\)
has target oscillation \(O(K^{-\beta})\).  The full power \(K^{-\beta}\) is
therefore obtained before introducing the first-order Lipschitz envelope and
its shallow ReLU approximation; neither construction causes saturation at
\(\beta=1\).

\paragraph{Bounded feature range and scalar reduction.}
Every coordinate of \(\Psi_K\) is a weighted average of
\(\ip{a/K}{y}\), while the retained endpoint coordinates are bounded
on \(E\).  Thus there is a fixed \(B_0=B_0(D)\geq1\) such that
\begin{equation}
\label{eq:quant-feature-range}
 \Phi_K(\cP(E)\times E)\subset[-B_0,B_0]^{m_K}
 \qquad\text{for every }K\geq3.
\end{equation}
The seminorm \(R_{\beta;\omega}\) controls oscillations but not an absolute
offset.  We first separate these two roles.  There is
\(C_{\mathrm{osc}}=C_{\mathrm{osc}}(D)\geq1\) such that
\begin{equation}
\label{eq:quant-target-oscillation}
 \sup_{\mu,\nu\in\cM}
 \norm{f^\star(\mu)-f^\star(\nu)}
 \leq C_{\mathrm{osc}}R_{\beta;\omega}(\mathbf F^\star).
\end{equation}
Indeed, \(d_{\beta_+}\leq2\), because the full test-function norm controls
the uniform norm, while the endpoint diameter is at most two.  Since
\(0<\beta_-\leq1\), the diameter for \(\Delta_\beta\) is therefore at
most four.
Fixing \(\mu_0\in\cM\), center each output coordinate by
\begin{equation}
\label{eq:quant-centered-coordinate}
 g_\ell(\mu)\eqdef f_\ell^\star(\mu)-f_\ell^\star(\mu_0).
\end{equation}
For the moment, fix \(\ell\), write \(g\eqdef g_\ell\), and suppress the coordinate
index.  Then
\begin{equation}
\label{eq:quant-centered-target-bound}
 \norm{g}_{L^\infty(\cM)}
 \leq C_{\mathrm{osc}}R_{\beta;\omega}(\mathbf F^\star).
\end{equation}

\paragraph{Lipschitz envelope.}
Assume first that \(R_{\beta;\omega}(\mathbf F^\star)>0\).  On
\(\mathsf Q_K\eqdef[-B_0,B_0]^{m_K}\), define
\begin{equation}
\label{eq:quant-envelope}
\begin{aligned}
 R_K(u)
 &\eqdef
 \inf_{\mu\in\cM}
 \left\{g(\mu)+A_K
 \norm{u-\Phi_K(\mu,\mathfrak e(\mu))}_\infty\right\},\\
 A_K
 &\eqdef C_A R_{\beta;\omega}(\mathbf F^\star)/\tau_K,
 \qquad C_A>C_{\mathrm{osc}}.
\end{aligned}
\end{equation}
For \(\lambda\in\cM\), abbreviate
\(u_\lambda\eqdef\Phi_K(\lambda,\mathfrak e(\lambda))\).  We verify the envelope
on this feature image.  Fix \(\lambda\in\cM\).  Using \(\lambda\) itself as a
competitor gives \(R_K(u_\lambda)\leq g(\lambda)\).  Conversely, a competitor
\(\nu\) with
\(\norm{u_\lambda-u_\nu}_\infty\leq\tau_K\) has value at least
\(g(\lambda)-C_\beta R_{\beta;\omega}(\mathbf F^\star)K^{-\beta}\) by
\eqref{eq:quant-target-resolution}.  If its feature distance instead exceeds
\(\tau_K\), the global oscillation bound and
\(C_A>C_{\mathrm{osc}}\) give
\begin{equation}
 g(\nu)+A_K\norm{u_\lambda-u_\nu}_\infty
 \geq g(\lambda)
 -C_{\mathrm{osc}}R_{\beta;\omega}(\mathbf F^\star)
 +C_A R_{\beta;\omega}(\mathbf F^\star)
 \geq g(\lambda).
\end{equation}
Every competitor falls into one of these cases, so taking the infimum proves
\begin{equation}
\label{eq:quant-envelope-error}
 \sup_{\mu\in\cM}
 \abs{g(\mu)-R_K(\Phi_K(\mu,\mathfrak e(\mu)))}
 \leq C_\beta R_{\beta;\omega}(\mathbf F^\star)K^{-\beta}.
\end{equation}
The envelope is \(A_K\)-Lipschitz.  Project it onto the interval containing
the centered target range:
\begin{equation}
\label{eq:quant-clipped-envelope}
 \overline R_K
 \eqdef
 \operatorname{proj}_{[-C_{\mathrm{osc}}R_{\beta;\omega}(\mathbf F^\star),
 C_{\mathrm{osc}}R_{\beta;\omega}(\mathbf F^\star)]}\circ R_K.
\end{equation}
This projection is one-Lipschitz and fixes every centered target value, so
\eqref{eq:quant-envelope-error} remains valid with \(\overline R_K\), while
\begin{equation}
\label{eq:quant-clipped-envelope-bounds}
 \norm{\overline R_K}_\infty
 \leq C_{\mathrm{osc}}R_{\beta;\omega}(\mathbf F^\star),
 \qquad
 \operatorname{Lip}_\infty(\overline R_K)\leq A_K.
\end{equation}

\paragraph{Counted shallow readout.}
Repeat the envelope construction for each coordinate \(\ell\), using the same
feature map and constants.  Apply
Lemma~\ref{lem:quant-shallow-relu} coordinatewise with
\(
M=C_{\mathrm{osc}}R_{\beta;\omega}(\mathbf F^\star)
\),
\(
A=A_K
\), and
\(
\delta\eqdef R_{\beta;\omega}(\mathbf F^\star)K^{-\beta}
\).
The shallow-network estimate uses
\begin{equation}
\label{eq:quant-readout-log-ratio}
 \log\!\left(2+\frac{A_K}{\delta}\right)
 \leq
 C_\beta+ A_\tau K\log(K+1)+\beta\log K.
\end{equation}
The scale \(R_{\beta;\omega}(\mathbf F^\star)\) cancels from
\(A_K/\delta\); hence neither this ratio nor the feature dimension depends on
\(\omega\).  Since \(m_K\leq C_DK^D\), concatenate the hidden units of the
\(d'\) scalar networks and stack their output rows to obtain one pointwise
network.  The widths add, while coordinatewise error \(\delta\) produces
Euclidean error at most \(\sqrt{d'}\,\delta\), absorbed into
\(C_\beta=C(d,d',\beta)\).  Its total hidden width \(W_K\) satisfies
\begin{equation}
\label{eq:quant-readout-width}
 \log(2+W_K)
 \leq C_\beta K^{D+1}\log(K+1)
 =C_\beta K^{d+2}\log(K+1).
\end{equation}
This exponent has a direct origin.  The joint token--time variable has
dimension \(D=d+1\), giving \(m_K=O(K^D)\) features; each feature contributes
the \(O(K\log(K+1))\) logarithmic stability cost in
\eqref{eq:quant-readout-log-ratio}.  Their product yields
\(K^{d+2}\log(K+1)\), hence the denominator \(d+2=(d+1)+1\) in the final
rate.
Adding the base values \(f_\ell^\star(\mu_0)\) through the final affine bias
gives
\begin{equation}
\label{eq:quant-readout-error}
 \sup_{\mu\in\cM}
 \norm{f^\star(\mu)-
 \mathcal R_K(\Phi_K(\mu,\mathfrak e(\mu)))}
 \leq C_\beta R_{\beta;\omega}(\mathbf F^\star)K^{-\beta}.
\end{equation}

\paragraph{Exact attention realization.}
It remains to realize the abstract feature map inside the transformer.  Each
of its \(H_K\eqdef2K(K+1)^{D-1}=O(K^D)\) normalized probe coordinates is computed
exactly by one masked head.  With the homogeneous and work coordinates
appended, take
\begin{equation}
\label{eq:quant-exact-head-parameters}
\begin{aligned}
 \mathsf Q^{\mathrm{att}}_{a,j}&\eqdef(0_D,1,0_{H_K}),
 &\mathsf K^{\mathrm{att}}_{a,j}&\eqdef(j h_K a^\top,0,0_{H_K}),\\
 \mathsf V^{\mathrm{att}}_{a,j}&\eqdef((a/K)^\top,0,0_{H_K}),
 &\mathsf W^{\mathrm{att}}_{a,j}&\eqdef e_{D+1+r(a,j)},
\end{aligned}
\end{equation}
where \(r(a,j)\in[H_K]\) enumerates the probes.  These matrices realize
\(
\mathcal L_{a/K,Kjh_K}=K^{-1}\mathcal L_{a,jh_K}
\)
on every finite prefix and on the temporal boundary.  Every displayed entry
lies in \([-1,1]\): the query and output matrices have only zero--one
entries, \(\norm{a/K}_\infty\leq1\), and
\(
 \norm{jh_Ka}_\infty\leq2K^2h_K\leq1
\).

\paragraph{Total parameter count.}
Let \(r_K\eqdef D+1+H_K=d+2+H_K\) be the embedding width.  Thus every matrix
\(\mathsf Q^{\mathrm{att}}_{a,j}\),
\(\mathsf K^{\mathrm{att}}_{a,j}\), and
\(\mathsf V^{\mathrm{att}}_{a,j}\) lies in \(\R^{1\times r_K}\), while
\(\mathsf W^{\mathrm{att}}_{a,j}\in\R^{r_K\times1}\).  The combined
readout \(\mathcal R_K\) was constructed on the
\(m_K=D+H_K\) coordinates \((y,\Psi_K)\).  Identify it on the actual
transformer state by
\begin{equation}
\label{eq:quant-homogeneous-readout-lift}
 \widehat{\mathcal R}_K(y,1,v)\eqdef\mathcal R_K(y,v),
\end{equation}
which inserts one zero column, at the homogeneous-coordinate position, in
the first readout matrix.  The resulting width-\(W_K\) pointwise readout has the form
\eqref{eq:pointwise-mlp-v2}, with
\[
 A_{1,K}\in\R^{W_K\times r_K},\quad b_{1,K}\in\R^{W_K},
 \qquad
 A_{2,K}\in\R^{d'\times W_K},\quad b_{2,K}\in\R^{d'}.
\]
We now count the architecture just constructed, rather than an abstract feature budget. Let \(T_K\) denote the resulting shallow transformer. Counting every scalar slot in the attention parameter list \(\theta_K\) and the readout parameter list \(\eta_K\) gives the exact dense count
\begin{equation}
\label{eq:quant-total-parameter-K}
\begin{aligned}
 p_K\eqdef\operatorname{par}(T_K)
 &=4H_Kr_K+(r_K+d'+1)W_K+d',\\
 \log p_K
 &\leq C_\beta K^{d+2}\log(K+1).
\end{aligned}
\end{equation}
Here \(4H_Kr_K\) counts the four length-\(r_K\) matrices in every scalar head;
the remaining terms count both readout matrices and both biases.  Thus
\(p_K\) includes structural zeros and all biases, while the fixed coordinate
lift contributes no parameters and shared coefficients are counted once.
As stated in Theorem~\ref{thm:quantitative-rates}, this dense real-parameter
count does not itself charge coefficient magnitudes or numerical precision.
The normalized finite-precision refinement below controls these additional
resources separately.

\paragraph{Inverting the count.}
Fix the final constant \(C_\beta\) in the parameter bound, enlarging it once
and for all if necessary.  Then enlarge
\(p_0=p_0(d,d',\beta)\) so that it is at least \(16\), covers the
fixed constant-output architecture described below, and satisfies
\[
 \exp\!\left(C_\beta 3^{d+2}\log4\right)\leq p_0.
\]
Then for every \(p\geq p_0\) the set in the following definition is nonempty:
\begin{equation}
\label{eq:quant-inverted-parameter-proof}
 K(p)
 \eqdef
 \max\left\{K\in\N:\ K\geq3,
 \exp\!\left(C_\beta K^{d+2}\log(K+1)\right)\leq p\right\}.
\end{equation}
We write \(T_p\eqdef T_{K(p)}\). Inverting this monotone budget relation gives constants \(c_\beta,C'_\beta>0\) such that
\begin{equation}
\label{eq:quant-inverted-scale-proof}
 c_\beta
 \left(\frac{\log p}{\log\log p}\right)^{1/(d+2)}
 \leq K(p)\leq
 C'_\beta
 \left(\frac{\log p}{\log\log p}\right)^{1/(d+2)}.
\end{equation}
Indeed, the defining inequality bounds \(K(p)\) from above, whereas its
failure at \(K(p)+1\) gives the matching lower bound.
Equations~\eqref{eq:quant-readout-error}--
\eqref{eq:quant-inverted-scale-proof} give the discrete error bound at the
rate in \eqref{eq:quantitative-error-p} and the head bound
\eqref{eq:quantitative-heads-p}. The independent transfer identity
\eqref{eq:quantitative-error-equality} gives the stated joint error bound
with the same parameters.
For \(0<\varepsilon\leq R_{\beta;\omega}(\mathbf F^\star)\), choosing
\begin{equation}
\label{eq:quant-epsilon-scale-proof}
 K_\varepsilon
 =\max\left\{3,
 \left\lceil
 \bigl(C_\beta R_{\beta;\omega}(\mathbf F^\star)/\varepsilon\bigr)^{1/\beta}
 \right\rceil\right\}
\end{equation}
and substituting into \eqref{eq:quant-total-parameter-K} gives
\eqref{eq:quantitative-parameter-epsilon}.  When
\(R_{\beta;\omega}(\mathbf F^\star)=0\), the canonical factor is constant
on \(\cM\).  To respect the theorem's conventions \(H,M\geq1\), take one
dummy head with zero output column, a one-unit readout with zero output
weight, and set its final bias equal to that constant.  This realizes the
target exactly and uses only the fixed number of parameters already absorbed
into \(p_0\).

\begin{proof}[Proof of Theorem~\ref{thm:quantitative-rates}]
Proposition~\ref{prop:quant-feature-resolution} converts proximity of exact
attention features into proximity in the target metric, and
Definition~\ref{def:causal-smoothness} then gives
\eqref{eq:quant-target-resolution}.  The clipped Lipschitz envelope and
Lemma~\ref{lem:quant-shallow-relu} yield \eqref{eq:quant-readout-error} at the
width in \eqref{eq:quant-readout-width}.  The heads in
\eqref{eq:quant-exact-head-parameters} realize these features at every finite
resolution and on the temporal boundary, and
\eqref{eq:quant-total-parameter-K} counts the resulting transformer.
Inverting this count proves the discrete error bound, the head bound
\eqref{eq:quantitative-heads-p}, and the accuracy-to-parameter estimate
\eqref{eq:quantitative-parameter-epsilon}. The transfer identity
\eqref{eq:quantitative-error-equality}, whose proof uses only continuous
extendability and Proposition~\ref{prop:temporal-closure}, upgrades these
estimates to the joint conclusion \eqref{eq:quantitative-error-p} without
additional parameters or loss.
\end{proof}

\subsection{Finite precision on the normalized target ball}

The real-parameter construction also yields a finite binary description after
normalizing the target amplitude.  We track its coefficients and their
sensitivity to rounding, uniformly over finite and continuous prefixes.

\begin{proof}[Proof of Corollary~\ref{cor:quantitative-finite-precision}]
We first bound the coefficients of a common scale-\(K\) architecture, then
quantize its dense parameter list.  The estimates concern parameter storage;
the mathematical network still evaluates its operations exactly.

\smallskip\noindent\emph{Coefficient bounds and a common architecture.}\hspace{0.25em}
Write \(R\eqdef R_{\beta;\omega}(\mathbf F^\star)\leq1\).  If \(R>0\), the
scalar envelopes above have magnitude at most \(C_{\mathrm{osc}}R\),
Lipschitz constant \(A_K=C_AR/\tau_K\), and readout tolerance
\(\delta=RK^{-\beta}\).  Their mollification scale in
\eqref{eq:quant-mollification-scale} is therefore
\[
 h_K^{\mathrm{rd}}\eqdef\frac{K^{-\beta}\tau_K}{3C_A}<1.
\]
For each scalar mollified envelope, \eqref{eq:quant-spectral-moment} gives
\begin{equation}
\label{eq:quant-precision-spectral-bound}
 v_{G_h,2}
 \leq C R m_K^2 C_1^{m_K}(h_K^{\mathrm{rd}})^{-m_K-2},
 \qquad
 \log\left(2+\frac{v_{G_h,2}}R\right)
 \leq C_\beta K^{D+1}\log(K+1).
\end{equation}
Theorem~2 of \citet{klusowski2018approximation}, used in
\eqref{eq:quant-KB-bound}, supplies ridge directions of \(\ell^1\) norm
one, thresholds in \([0,1]\), and coefficients of the form
\(vb_k/W\), where \(\abs{b_k}\leq1\) and
\(v\leq C_{B_0}v_{G_h,2}\).  Composing with \(u\mapsto u/B_0\) keeps
each inner row's \(\ell^1\) norm and each bias bounded by one.  Its affine
part has constant coefficient \(G_h(0)\) and slope \(\nabla G_h(0)\).
Since mollification
only sees the uncut Lipschitz extension near zero,
\[
 \abs{G_h(0)}\leq C_{\mathrm{osc}}R,
 \qquad \norm{\nabla G_h(0)}_1\leq A_K.
\]
Representing this affine part by the \(2m_K\) additional ReLU units from
the proof of Lemma~\ref{lem:quant-shallow-relu} preserves the inner-row and
bias bounds.  Its scalar output row has \(\ell^1\) norm at most
\(v+2A_K\).  Concatenating coordinates and inserting the zero homogeneous
column preserve these row bounds.  The restored target offset is bounded by
one.  These estimates bound every coefficient and give the more useful
row-norm bounds
\begin{equation}
\label{eq:quant-precision-row-bounds}
\begin{gathered}
 \max_i\sum_j\abs{(A_{1,K})_{ij}}\leq1,
 \qquad \norm{b_{1,K}}_\infty\leq1,\\
 \max_\ell\sum_i\abs{(A_{2,K})_{\ell i}}
 \leq\exp(C_\beta K^{D+1}\log(K+1)).
\end{gathered}
\end{equation}

All scale choices are independent of the target.  In particular,
\(v_{G_h,2}/\delta\) has a uniform upper bound by
\eqref{eq:quant-precision-spectral-bound}.  Choose the scalar widths using
this bound in \eqref{eq:quant-KB-width-substitution}, or pad smaller
readouts with zero-output units.  Concatenating the coordinates then gives a
single topology for the entire normalized ball.  It also includes the case
\(R=0\): zero the readout weights and use its final bias for the constant
target.  After enlarging the design constant, its dense parameter count,
every coefficient magnitude, and every output-row \(\ell^1\) norm are at
most
\[
 P_K\eqdef\exp\!\left(C_\beta K^{D+1}\log(K+1)\right).
\]
Choose \(K=K(p)\) by the budget inversion with this enlarged constant.
Then these three bounds are at most \(p\), while
\eqref{eq:quant-precision-row-bounds} retains the unit inner-row and bias
bounds.  The unrounded joint approximation error is at most
\(C_\beta K^{-\beta}\).

\smallskip\noindent\emph{Uniform perturbation of dense parameters.}\hspace{0.25em}
Perturb every scalar parameter by at most \(\eta\leq1/p\), retaining the
same topology.  Each input to the attention block is
\((y,1,0_{H_K})\), with only \(D+1\) possibly nonzero entries bounded
by one.  Since the original attention entries lie in \([-1,1]\), scalar
queries, keys, values, and their perturbed versions are bounded by a constant
depending only on \(D\).  Their differences, and the differences of the
scalar scores, are at most \(C_D\eta\).

For any probability law \(\lambda\) and bounded real score functions
\(s,s'\), put
\(\pi_s\eqdef e^s\lambda/\int e^s\,\dd\lambda\).  Differentiating along the
segment joining the scores gives the dimension-free estimate
\[
 \int\abs{\dd\pi_s-\dd\pi_{s'}}
 \leq2\norm{s-s'}_\infty.
\]
Consequently each normalized head average changes by at most \(C_D\eta\).
The output column is also perturbed entrywise by at most \(\eta\), so the
same bound holds for each coordinate of a head's contribution.  Summing the
heads gives, for the exact and perturbed attention states \(u,u'\),
\begin{equation}
\label{eq:quant-rounded-attention-stability}
 \norm{u-u'}_\infty\leq C_DH_K\eta,
 \qquad \norm{u}_\infty+\norm{u'}_\infty\leq C_D.
\end{equation}
The second estimate uses \(H_K\leq p\) and the bounded original feature
range.  This argument permits perturbations of structural zeros, including
the dense output columns.  It applies to every discrete prefix and to the
normalized prefix integral for every \(t>0\), with no dependence on prefix
length.  At \(t=0\), the head average is simply the value at zero and the
same bound follows directly.

Let \(r\) and \(M\) be the residual and readout widths; both are at most
\(p\).  Write \(v\eqdef A_1u+b_1\) and \(v'\eqdef A'_1u'+b'_1\).
The unit inner-row bound and arbitrary dense entrywise perturbations give
\[
 \norm{v-v'}_\infty
 \leq \norm{u-u'}_\infty+r\eta\norm{u'}_\infty+\eta
 \leq C_Dp\eta,
 \qquad \norm{v'}_\infty\leq C_D.
\]
ReLU is one-Lipschitz, and each original output row has \(\ell^1\) norm
at most \(p\).  Perturbing the second affine layer therefore gives
\[
 \norm{A_2\rho(v)+b_2-A'_2\rho(v')-b'_2}_\infty
 \leq p\norm{v-v'}_\infty+M\eta\norm{\rho(v')}_\infty+\eta.
\]
Absorbing the fixed output-norm conversion, we obtain
\begin{equation}
\label{eq:quant-rounded-network-stability}
 \sup_{\text{all finite and boundary states}}
 \norm{T-T'}\leq C_{d,d'}p^2\eta.
\end{equation}
No sparsity is assumed in these readout perturbations.

\smallskip\noindent\emph{Dyadic rounding and binary descriptions.}\hspace{0.25em}
For a sufficiently large constant \(C\), choose
\[
 b\eqdef\left\lceil\log_2(Cp^2K^\beta)\right\rceil,
 \qquad \eta\eqdef2^{-b}\leq\frac{K^{-\beta}}{Cp^2},
\]
and round each parameter to the nearest multiple of \(2^{-b}\).  The
rounding error is at most \(\eta\leq1/p\), hence
\eqref{eq:quant-rounded-network-stability} adds at most \(K^{-\beta}\) to
the joint error.  The rounded attention entries remain in \([-1,1]\)
because its endpoints belong to this dyadic grid.  Every other rounded
entry has magnitude at most \(p+1\).  Its signed integer numerator has
magnitude at most \((p+1)2^b\), requiring
\[
 b+\left\lceil\log_2(p+2)\right\rceil+2
 \leq C_\beta\log_2 p
\]
bits.  Use this fixed field length for every dense parameter, including zeros.
Here \(\log K\leq\log p\), and \(\beta\) is fixed.  A common header
specifying the dimensions and the binary precision costs
\(O_{d,d',\beta}(\log p)\) additional bits; it is unnecessary if these
budget-dependent quantities are known to the decoder.  Thus total length is
at most \(C_\beta p\log_2 p\).
Budget inversion gives \eqref{eq:quantitative-rounded-error}.

Finally, for the finite-code radius in
\eqref{eq:finite-code-distortion}, decode each dyadic parameter list to its
completed history factor.  The resulting finite codebook approximates the
whole normalized target ball.  For a sufficiently small fixed
\(c_\beta>0\), take
\(p\eqdef\lfloor c_\beta B/\log_2 B\rfloor\).  For all sufficiently large
\(B\), this budget is above the required threshold and its total
description fits within \(B\) bits.  Padding each fixed-length description
to \(B\) bits makes the codebook cardinality at most \(2^B\), as required
in \eqref{eq:finite-code-distortion}.  Since \(\log p\asymp\log B\) and
\(\log\log p\asymp\log\log B\), this proves
\eqref{eq:quantitative-code-upper-rate}.  The architecture and the
coefficient, precision, and error bounds are independent of the chosen
admissible input modulus; the selected coefficients may depend on the input
class and target.
\end{proof}

The upper-bound proof is now complete for an arbitrary admissible modulus
\(\omega\).  For the optimality analysis of
Subsection~\ref{sec:rate-optimality}, we now specialize to a H\"older history
class.  Its multilevel history complexity yields the logarithmic
finite-description lower bound.
Throughout the remainder of this appendix, we therefore reset
\(
X_n\eqdef X_n^{\alpha,L}
\),
\(
X_\infty\eqdef X_\infty^{\alpha,L}
\),
\(
\mathfrak X\eqdef\mathfrak X_{\alpha,L}
\), and
\(
\cM\eqdef\cM_{\alpha,L}
\).

\subsection{Proofs of the history-complexity and capacity bounds}
\label{app:proof-rate-optimality}

We first identify continuous functions on the history space with valid
completed causal families.  We then construct many distinguishable H\"older
histories, convert them into a finite-description lower bound, and prove the
fixed-size infinite-capacity result.

\paragraph{The metric history space.}
Lemma~\ref{lem:unified-history-metric} shows that \(\Delta_\beta\) is a
distance inducing the original compact topology on \(\cM\).  Thus the
packing number \(\mathsf P_{\beta;\alpha,L}(r)\) is finite for every
\(r>0\), and every \(\Delta_\beta\)-Lipschitz factor is continuous in the
original history topology.

\paragraph{From history factors to completed families.}
Let \(f:\cM\to\R^{d'}\) be continuous, and define
\begin{equation}
\label{eq:factor-induced-family}
 F_n^f(z)_i\eqdef f(\mu_{n,z,i}),
 \qquad
 F_\infty^f(x)(t)\eqdef f(\mu_{\infty,x,t}).
\end{equation}
By Proposition~\ref{prop:prefix-law}, the prefix-law map
\(\boldsymbol\mu:\mathfrak X\to\cM\) is continuous; hence
\(f\circ\boldsymbol\mu\) is continuous on \(\mathfrak X\).  This gives
continuity of every finite map and joint continuity of the boundary evaluation
in \((x,t)\).  The latter is uniform on the compact set
\(X_\infty\times[0,1]\), so
\(F_\infty^f:X_\infty\to C([0,1];\R^{d'})\) is continuous in the uniform
norm.  Each finite restriction is prefix-causal because its law uses only the
observed prefix.  Lemma~\ref{lem:target-compatibility} therefore shows that
\eqref{eq:factor-induced-family} is a continuously extendable causal family
with path extension \(F_\infty^f\).  If \(f\) is one-Lipschitz for
\(\Delta_\beta\), its smooth-test seminorm is at most one.

\begin{proof}[Proof of Proposition~\ref{prop:holder-history-exponential-packing}]
\smallskip\noindent\emph{A multilevel family of admissible histories.}\hspace{0.25em}
We place many token levels in each of many disjoint time cells.  Choose
nonnegative functions
\(\psi,\varrho\in C_c^\infty((0,1))\) such that
\(\psi=1\) on a neighborhood of \(\operatorname{supp}\varrho\) and
\(\int_0^1\varrho>0\).  Put \(\lambda\eqdef\min\{L,1\}\), and fix constants
\(0<c_1<c_0\leq1\) with \(c_0/c_1\geq8\), taking \(c_0\) small enough below.
For an integer \(q\geq8\), set
\begin{equation}
\label{eq:multilevel-history-scales}
 h\eqdef q^{-1},\qquad
 a_h\eqdef c_0\lambda h^\alpha,\qquad
 \delta_h\eqdef c_1\lambda h,\qquad
 M_h\eqdef1+\left\lfloor a_h/\delta_h\right\rfloor,\qquad
 b_j\eqdef j\delta_h.
\end{equation}
Then \(M_h\geq2\), every \(b_j\), \(0\leq j<M_h\), lies in
\([0,a_h]\), and, for all sufficiently small \(h\),
\begin{equation}
\label{eq:multilevel-alphabet-size}
 \log M_h
 \geq c\left[1+(1-\alpha)\log(1/h)\right].
\end{equation}

Let \(e_1\eqdef(1,0,\ldots,0)\in\R^d\), and let
\(\psi_k(s)\eqdef\psi((s-(k-1)h)/h)\). A level-index vector
\({\boldsymbol m}=(m_1,\ldots,m_q)\in\{0,\ldots,M_h-1\}^q\)
selects the amplitude \(b_{m_k}\) in time cell \(((k-1)h,kh)\).
Its associated path is
\begin{equation}
\label{eq:multilevel-history-paths}
 x_{\boldsymbol m}(s)
 \eqdef\sum_{k=1}^q b_{m_k}\psi_k(s)e_1.
\end{equation}
The supports lie strictly inside disjoint cells.  For distances at most \(h\),
the derivative bound gives
\[
\norm{x_{\boldsymbol m}(s)-x_{\boldsymbol m}(t)}
\leq C a_h h^{-1}\abs{s-t}\leq L\abs{s-t}^\alpha
\]
after decreasing \(c_0\); on larger distances the oscillation bound gives
the same conclusion.
The same choice ensures that the paths lie in \(\Omega\).  Thus
\(x_{\boldsymbol m}\in X_\infty^{\alpha,L}\), and all paths vanish at
time one.

We select many index vectors that differ in a fixed fraction of time cells.
Write \(d_{\mathrm H}({\boldsymbol m},{\boldsymbol n})
\eqdef\#\{k:m_k\neq n_k\}\) for their Hamming distance.
A standard greedy packing of the discrete cube gives a subset
\(\Sigma_q\subset\{0,\ldots,M_h-1\}^q\) such that
\begin{equation}
\label{eq:multilevel-hamming-code}
 d_{\mathrm H}({\boldsymbol m},{\boldsymbol n})\geq q/4,
 \qquad
 \log|\Sigma_q|\geq c q\log M_h
 \quad({\boldsymbol m}\neq{\boldsymbol n}).
\end{equation}
Indeed, the Hamming-ball bound at any fixed relative radius below one half is
uniform over \(M_h\geq2\). Associate with each selected index vector the terminal history
\begin{equation}
\label{eq:multilevel-history-laws}
 \mu_{\boldsymbol m}
 \eqdef\int_0^1\delta_{(x_{\boldsymbol m}(s),s)}\,\dd s
 \in\cM_{\alpha,L}.
\end{equation}
All extracted endpoints equal \((0,1)\).

\smallskip\noindent\emph{Pairwise separation by smooth tests.}\hspace{0.25em}
For each pair of selected index vectors, we build a test adapted to the
cells in which their amplitude levels differ. Fix a
bounded nondecreasing \(\vartheta\in C^\infty(\R)\), constant outside
\([-1/4,1/4]\), with distinct limiting values.  For distinct levels \(u,v\),
put
\begin{equation}
\label{eq:multilevel-token-separator}
 \eta_{u,v}(x)
 \eqdef\operatorname{sgn}(u-v)
 \vartheta\!\left(\frac{x-(u+v)/2}{\delta_h}\right).
\end{equation}
Since \(\abs{u-v}\geq\delta_h\), there are constants independent of the
levels and of \(h\) such that
\begin{equation}
\label{eq:multilevel-token-separator-bounds}
 \eta_{u,v}(u)-\eta_{u,v}(v)\geq c,
 \qquad
 \norm{\eta_{u,v}^{(j)}}_\infty\leq C_j\delta_h^{-j}.
\end{equation}
Write \(\varrho_k(s)\eqdef\varrho((s-(k-1)h)/h)\). For two selected index vectors, define
\begin{equation}
\label{eq:multilevel-separating-test}
 \Phi_{{\boldsymbol m},{\boldsymbol n}}(x,s)
 \eqdef c_{\mathrm{test}}\delta_h^{\beta_+}
 \sum_{k:m_k\neq n_k}
 \varrho_k(s)\eta_{b_{m_k},b_{n_k}}(x_1),
\end{equation}
At each time, at most one summand is active.  If \(0<\beta\leq1\), its
spatial derivative is \(O(\delta_h/\delta_h)\), and its time derivative is
\(O(\delta_h/h)\).  Its uniform norm is \(O(\delta_h)\); hence, because
\(\delta_h\leq h\), decreasing \(c_{\mathrm{test}}\) controls the full
\(C^1(\mathsf Q)=C^{0,1}(\mathsf Q)\) norm, not only the Lipschitz
seminorm.

If \(\beta>1\), put
\(r_\beta\eqdef\lceil\beta\rceil-1\) and
\(\theta_\beta\eqdef\beta-r_\beta\).  Every mixed derivative with
\(j\) time and \(\ell\) token derivatives satisfies
\begin{equation}
\label{eq:multilevel-test-derivatives}
 \norm{\partial_s^j\partial_{x_1}^\ell
 \Phi_{{\boldsymbol m},{\boldsymbol n}}}_\infty
 \leq C\delta_h^\beta h^{-j}\delta_h^{-\ell}
 \leq C\delta_h^{\beta-j-\ell},
 \qquad j+\ell\leq r_\beta.
\end{equation}
Since \(\delta_h=c_1\lambda h\), derivatives of total order \(r_\beta\)
are \(O(h^{\theta_\beta})\).  Rescaling within one cell gives a uniform
\(\theta_\beta\)-H\"older bound.  Moreover, the compact support of
\(\varrho\) leaves a fixed relative gap: distinct active supports are
separated by at least \(c_\varrho h\), so the same bound holds across cells.
This also covers integer \(\beta\), for which \(\theta_\beta=1\) and the
required endpoint class is \(C^{\beta-1,1}\).  In either case, after one
fixed rescaling,
\begin{equation}
\label{eq:multilevel-smooth-test-bound}
 \norm{\Phi_{{\boldsymbol m},{\boldsymbol n}}}_{C^{\beta_+}(\mathsf Q)}\leq1.
\end{equation}

Because \(\psi_k=1\) on \(\operatorname{supp}\varrho_k\),
\eqref{eq:multilevel-token-separator-bounds} and the Hamming-distance bound
give
\begin{equation}
\label{eq:multilevel-test-separation}
 \int_E\Phi_{{\boldsymbol m},{\boldsymbol n}}\,
 \dd(\mu_{\boldsymbol m}-\mu_{\boldsymbol n})
 \geq c\delta_h^{\beta_+}h\,
 d_{\mathrm H}({\boldsymbol m},{\boldsymbol n})
 \geq c\delta_h^{\beta_+}.
\end{equation}
Thus \(d_{\beta_+}(\mu_{\boldsymbol m},\mu_{\boldsymbol n})
\geq c h^{\beta_+}\).  Raising this bound to the power \(\beta_-\)
gives \(\Delta_\beta(\mu_{\boldsymbol m},\mu_{\boldsymbol n})
\geq c h^\beta\), with constants depending only on the fixed parameters
\((d,\beta,L)\).
\smallskip\noindent\emph{From separation to entropy.}\hspace{0.25em}
Combining this separation with
\eqref{eq:multilevel-alphabet-size}--\eqref{eq:multilevel-hamming-code} gives
\begin{equation}
\label{eq:multilevel-packing-at-h}
 \log\mathsf P_{\beta;\alpha,L}(c h^\beta)
 \geq c h^{-1}\left[1+(1-\alpha)\log(1/h)\right].
\end{equation}
Choosing an integer \(q\asymp r^{-1/\beta}\) proves
\eqref{eq:explicit-history-packing}.

\smallskip\noindent\emph{From histories to target families.}\hspace{0.25em}
Set
\(Y\eqdef\log(B+e^e)\), \(Z\eqdef1+(1-\alpha)\log Y\), and
\(\varepsilon\eqdef c(Z/Y)^\beta\).  With \(c>0\) sufficiently small,
\eqref{eq:explicit-history-packing} gives
\(\log\mathsf P_{\beta;\alpha,L}(2\varepsilon)>\log B\) for all large
\(B\); also \(\varepsilon\leq1\).  Choose \(N>B\) histories
\(\mu_1,\ldots,\mu_N\) separated by at least \(2\varepsilon\).
For each sign vector \(\sigma\in\{-1,1\}^N\), prescribe
\(a_j^\sigma\eqdef\varepsilon\sigma_j\) at \(\mu_j\).  These values are
one-Lipschitz for \(\Delta_\beta\), because their pairwise differences
are at most \(2\varepsilon\).  Their McShane extension is
\begin{equation}
\label{eq:minimax-mcshane-extension}
 \widetilde f_\sigma(\mu)
 \eqdef \min_{1\leq j\leq N}
 \left\{a_j^\sigma+\Delta_\beta(\mu,\mu_j)\right\}.
\end{equation}
Projecting this scalar function onto
\([-\varepsilon,\varepsilon]\) and placing it in the first output
coordinate gives a continuous factor \(f_\sigma:\cM\to\R^{d'}\).
It retains the prescribed values and has
\begin{equation}
\label{eq:minimax-packed-factors}
 R_{\beta;\alpha,L}(\mathbf F_\sigma)\leq1,
 \qquad \norm{f_\sigma}_\infty\leq\varepsilon\leq1,
\end{equation}
where \(\mathbf F_\sigma\) is the completed causal family induced by
\eqref{eq:factor-induced-family}.  Thus every \(\mathbf F_\sigma\)
belongs to the normalized target ball.  Distinct sign vectors give factors
separated by \(2\varepsilon\) in uniform norm, so a ball of radius strictly
below \(\varepsilon\) contains at most one of these \(2^N\) factors.
A codebook with at most \(2^B<2^N\) decoded maps cannot approximate all of
them with error below \(\varepsilon\).  By
\eqref{eq:finite-code-distortion}, this proves
\eqref{eq:explicit-code-lower-rate}.
\end{proof}

\begin{proof}[Proof of Proposition~\ref{prop:infinite-threshold-capacity}]
We exhibit an infinite-VC subclass using one scalar head and one hidden ReLU
unit.  Put \(D=d+1\), so the residual width is \(r=d+3=D+2\).  In the lifted
state \((z,s,1,w)\), the four blocks are respectively the token, time,
homogeneous, and work coordinates.  Let \(e_1^{\mathrm{in}}\) be the first
coordinate vector of \(\R^d\).

\smallskip\noindent\emph{A one-parameter threshold family.}\hspace{0.25em}
For \(0<\vartheta<1\), choose the single head by
\begin{equation}
\label{eq:vc-one-head-parameters}
 Q\eqdef\bar e_{D+1}^\top,
 \qquad
 K\eqdef\vartheta\bar e_D^\top,
 \qquad
 V\eqdef\bar e_1^\top,
 \qquad
 W\eqdef\bar e_{D+2},
\end{equation}
where \((\bar e_j)_{j=1}^{D+2}\) is the standard basis of
\(\R^{D+2}\).  Thus \(Q\) selects the homogeneous coordinate, \(K\) reads
time, \(V\) reads the first token coordinate, and \(W\) writes the result to
the work coordinate.  On the candidate path
\(s\mapsto x(s)e_1^{\mathrm{in}}\), the terminal work coordinate is
\begin{equation}
\label{eq:vc-laplace-coordinate}
 A_\vartheta(x)
 \eqdef\frac{\int_0^1e^{\vartheta s}x(s)\,\dd s}
        {\int_0^1e^{\vartheta s}\,\dd s}.
\end{equation}
Let the one-unit readout extract this coordinate, apply ReLU, and write the
result into the first output coordinate: explicitly,
\begin{equation}
\label{eq:vc-one-unit-readout}
 A_1\eqdef\bar e_{D+2}^\top,\qquad b_1\eqdef0,
 \qquad A_2\eqdef e_1',\qquad b_2\eqdef0.
\end{equation}
Consequently, the first output coordinate is positive exactly when
\(\Lambda_\vartheta(x)\eqdef\int_0^1e^{\vartheta s}x(s)\,\dd s\) is
positive.  All displayed network weights lie in \([-1,1]\).

\smallskip\noindent\emph{Fixing histories before selecting labels.}\hspace{0.25em}
Fix \(N\geq1\), enumerate the \(2^N\) sign vectors in
\(\{-1,1\}^N\) as \(\sigma^1,\ldots,\sigma^{2^N}\), and choose distinct
\(\vartheta_1,\ldots,\vartheta_{2^N}\in(0,1)\).  The functionals
\(\Lambda_{\vartheta_\ell}\) are linearly independent on
\(C_c^\infty((0,1))\).  Indeed, a relation
\(\sum_\ell c_\ell\Lambda_{\vartheta_\ell}=0\) would give
\[
 \int_0^1\left(\sum_\ell c_\ell e^{\vartheta_\ell s}\right)x(s)\,\dd s=0
 \qquad\text{for every }x\in C_c^\infty((0,1)).
\]
The fundamental lemma for test functions makes the parenthesized smooth
function vanish on \((0,1)\), and distinct exponentials are linearly
independent.  The coordinate functionals are therefore independent, so the
finite-dimensional range of the linear map
\begin{equation}
\label{eq:vc-surjective-moment-map}
 x\longmapsto
 \bigl(\Lambda_{\vartheta_1}(x),\ldots,
       \Lambda_{\vartheta_{2^N}}(x)\bigr)
\end{equation}
has rank \(2^N\) and is onto \(\R^{2^N}\).  For each \(i\in[N]\), choose
\(\widetilde x_i\in C_c^\infty((0,1))\) whose image is the \(i\)-th row of
the full sign table:
\(
\Lambda_{\vartheta_\ell}(\widetilde x_i)=\sigma_i^\ell
\)
for every \(\ell\).  This choice fixes the same \(N\) functions for all
\(2^N\) labelings.

Now choose \(c_i>0\) small enough that
\(x_i\eqdef c_i\widetilde x_i\) takes values in the unit ball and has
\(\alpha\)-H\"older seminorm at most \(L\).  Positive rescaling changes the
magnitudes but preserves every sign.  For each of these scaled functions,
write
\begin{equation}
\label{eq:vc-smooth-history}
 x_i^{\mathrm{vec}}(s)\eqdef x_i(s)e_1^{\mathrm{in}},
 \qquad
 \mu_{x_i}\eqdef\int_0^1\delta_{(x_i^{\mathrm{vec}}(s),s)}\,\dd s.
\end{equation}
Then \(\mu_{x_i}\in\cM_{\alpha,L}\) for all \(i\).  With these histories
fixed, selecting \(\vartheta=\vartheta_\ell\) realizes the labeling
\(\sigma^\ell\), because
\(
\operatorname{sgn}\Lambda_{\vartheta_\ell}(x_i)=\sigma_i^\ell
\)
for every \(i\).  Thus the same \(N\) histories are shattered, and, since
\(N\) was arbitrary, the VC dimension is infinite.

\smallskip\noindent\emph{Fixed parameter count.}\hspace{0.25em}
For \(H=M=1\), the dense count
\eqref{eq:exact-dense-parameter-count-main} is
\(
4(d+3)+(d+3+d'+1)+d'=5(d+3)+1+2d'
\), which proves \eqref{eq:fixed-size-infinite-vc}.
\end{proof}

\begin{proof}[Proof of Corollary~\ref{cor:finite-sequence-threshold-capacity}]
For a prescribed \(N\), use the paths \(x_i^{\mathrm{vec}}\), parameters
\(\vartheta_\ell\), and labelings \(\sigma^\ell\) from the preceding
proof.  They form finite collections, and
\[
 \Lambda_{\vartheta_\ell}(x_i)=c_i\sigma_i^\ell,
 \qquad \delta_N\eqdef\min_{i\in[N]}c_i>0.
\]
Right-endpoint Riemann sums converge simultaneously for all pairs
\((i,\ell)\).  Choose one sufficiently large \(n\) so that
\[
 \left|\frac1n\sum_{j=1}^n
    e^{\vartheta_\ell j/n}x_i(j/n)
       -\Lambda_{\vartheta_\ell}(x_i)\right|
 <\delta_N/2
 \qquad\text{for every }i,\ell.
\]
The sampled sequences \(\mathcal S_nx_i^{\mathrm{vec}}\) belong to
\(X_n^{\alpha,L}\).  For parameter \(\vartheta_\ell\), the terminal
head numerator has the sign \(\sigma_i^\ell\), and its denominator is
strictly positive.  The one-unit readout in
\eqref{eq:vc-one-unit-readout} therefore realizes the same threshold
labeling as on the path.  These sign patterns also ensure that the \(N\)
sampled sequences are distinct.  The heads, readout weights, and parameter
count are unchanged, proving the claim for every \(N\).
\end{proof}

\section{Proof of Theorems~\ref{thm:discrete-universality} and
\ref{thm:path-universality-main}}
\label{app:proof-joint}

Both qualitative theorems, including the expanded path statement in
Theorem~\ref{thm:continuous-universality}, follow from one argument on the
completed evaluation space. We proceed through compactification, target
compatibility, prefix-law identification, architectural closure, Laplace
separation, one-block realization, and Stone--Weierstrass assembly. The closure
argument applies to every fixed finite-depth transformer. It uses
step interpolants; Proposition~\ref{prop:temporal-closure} in
Appendix~\ref{app:continuous-time-universality} upgrades the convergence to
affine interpolants.
Throughout this appendix, write \(X_n\eqdef X_n^\omega\),
\(X_\infty\eqdef X_\infty^\omega\), and \(\cM\eqdef\cM_\omega\) for
the fixed common modulus \(\omega\).

\subsection{Detailed proof reading guide}
\label{app:detailed-proof-guide}

The argument has four steps, common to the discrete and continuous-time
statements. This guide introduces the objects and their roles before the
detailed lemmas and proofs below.

\emph{1. One compact space of evaluation states.}
Set
\[
 \mathfrak X\eqdef
 \Bigl(\bigsqcup_{n\geq1}\{n\}\times X_n\times[n]\Bigr)
 \sqcup\bigl(\{\infty\}\times X_\infty\times[0,1]\bigr).
\]
A finite state \(\xi=(n,z,i)\) records prediction at token \(i\), while
\(\xi=(\infty,x,t)\) records prediction at path time \(t\).  Define its path,
time, and resolution coordinates by
\[
 (P\xi,\tau(\xi),h(\xi))\eqdef
 \begin{cases}(\In_nz,i/n,1/n),&\xi=(n,z,i),\\
               (x,t,0),&\xi=(\infty,x,t)
 \end{cases}
\]
and distance
\[
 d_{\mathfrak X}(\xi,\xi')\eqdef
 \norm{P\xi-P\xi'}_\infty+
 \abs{\tau(\xi)-\tau(\xi')}+
 \abs{h(\xi)-h(\xi')}.
\]
The coordinate \(h\) distinguishes finite resolutions and tends to zero only
at the path boundary.  The common modulus, the reconstruction property
\eqref{eq:modulus-reconstruction-v2}, and Arzel\`a--Ascoli make
\(\mathfrak X\) compact; sampling paths on finer grids makes the finite states
dense.  The completed target is the joint evaluation
\[
 \overline F^\star(n,z,i)\eqdef F_n^\star(z)_i,
 \qquad
 \overline F^\star(\infty,x,t)\eqdef F_\infty^\star(x)(t).
\]
Definition~\ref{def:compatible-family} is precisely the condition ensuring
that \(\overline F^\star\) is continuous across the finite-to-path boundary.
For a path target alone, take
\(F_n^\dagger(z)_i\eqdef F_\infty^\star(\In_nz)(i/n)\).
Lemmas~\ref{lem:canonical-completion}--\ref{lem:target-compatibility}
justify this construction, compactness, and target continuity.

\emph{2. A complete state variable for the causal history.}
Let \(E\eqdef\Omega\times[0,1]\), and let \(\mathcal P(E)\) carry the weak topology
(convergence against continuous functions on \(E\)).  The normalized
 token--position prefix law of \(\xi\) is
\[
 \mu_\xi\eqdef
 \begin{cases}
 \displaystyle i^{-1}\sum_{j=1}^i\delta_{(z_j,j/n)},
      &\xi=(n,z,i),\\[1.5ex]
 \displaystyle t^{-1}\int_0^t\delta_{(x(s),s)}\,\dd s,
      &\xi=(\infty,x,t),\ t>0,\\[1.5ex]
 \delta_{(x(0),0)},&\xi=(\infty,x,0).
 \end{cases}
\]
The map \(\xi\mapsto\mu_\xi\) is continuous.  Its fibers are exactly the
causal-information classes: at finite resolution the time marginal has
support \(\{1/n,\ldots,i/n\}\), which determines \((n,i)\), and the token
attached to each time recovers \(z_{1:i}\). At the boundary the time marginal
determines \(t\), and the graph law plus continuity recovers
\(x|_{[0,t]}\). Hence, with
\(\mathcal M\eqdef\{\mu_\xi:\xi\in\mathfrak X\}\),
\[
 \mu_\xi=\mu_{\xi'}
 \Longleftrightarrow
 \xi,\xi'\text{ have the same causal history},
 \qquad
 \overline F^\star(\xi)=f^\star(\mu_\xi)
\]
for a unique continuous \(f^\star:\mathcal M\to\R^{d'}\).  Here
\(\mathcal M\) is compact, and continuity of the factor follows because a
continuous surjection from a compact space onto a Hausdorff space preserves
continuous fiberwise factorizations.
In particular, resolution is the continuous statistic
\(h(\xi)=2\int s\,\dd\mu_\xi-\tau(\xi)\), which vanishes at the path
boundary. Proposition~\ref{prop:prefix-law} proves the representation,
factorization, and continuity of the endpoint and resolution extractors.

\emph{3. Separating coordinates computed exactly by attention.}
For \(a\in\R^{d+1}\) with \(\norm{a}\leq1\), \(c\in[0,1]\), and \(y=(x,s)\in E\), define
\[
 \mathcal L_{a,c}(\mu)\eqdef
 \frac{\int_Ee^{c\ip{a}{y}}\ip{a}{y}\,\dd\mu(y)}
      {\int_Ee^{c\ip{a}{y}}\,\dd\mu(y)}
 =\partial_c\log\!\int_Ee^{c\ip{a}{y}}\,\dd\mu(y).
\]
One scalar masked-attention head computes this probe exactly: the homogeneous
coordinate supplies query \(1\), the key and value are
\(c\ip{a}{y}\) and \(\ip{a}{y}\), and the output is written to one work
coordinate.  The same matrices compute the finite softmax average and the
boundary temporal integral.  Moreover,
\[
 \log\!\int_Ee^{\ip{a}{y}}\,\dd\mu(y)
 =\int_0^1\mathcal L_{a,c}(\mu)\,\dd c.
\]
Thus equality of all probes identifies the Laplace transforms near zero,
then every polynomial moment and, because \(E\) is compact, the laws. The
probes therefore separate points of \(\mathcal M\)
(Lemma~\ref{lem:laplace-probes}).

\emph{4. Stone--Weierstrass and realization.}
Finite sums and products of probes and constants form a point-separating
unital algebra on \(\mathcal M\).  Stone--Weierstrass, applied coordinatewise,
gives probes \(\mathcal L_{a_r,c_r}\), \(r\in[H]\), and a vector polynomial \(\Pi\)
with
\[
 \sup_{\mu\in\mathcal M}
 \norm{f^\star(\mu)-
 \Pi\bigl(\mathcal L_{a_1,c_1}(\mu),\ldots,\mathcal L_{a_H,c_H}(\mu)\bigr)}<\varepsilon/2.
\]
One masked-attention block evaluates these probes in parallel, and one shared
pointwise ReLU network approximates \(\Pi\) on their compact range to error
\(\varepsilon/2\).  The resulting parameters \(\Theta\), independent of
\(n\), define a joint finite/boundary output \(\overline T\) satisfying
\[
 \sup_{\xi\in\mathfrak X}
 \norm{\overline T(\xi)-\overline F^\star(\xi)}<\varepsilon.
\]
The finite and boundary restrictions prove
Theorems~\ref{thm:discrete-universality} and~\ref{thm:path-universality-main},
respectively, with the same parameters. Lemma~\ref{lem:parallel-realization}
constructs the parallel heads; Appendix~\ref{app:stone-weierstrass-assembly}
assembles the polynomial and readout approximation.

\subsection{The joint compactification}

The first issue is topological: every finite resolution must fit into one
compact space whose only cross-resolution accumulation boundary is the
common-modulus path class.

Let
\(
S\eqdef\{0\}\cup\{1/n:n\in\N\}
\)
with the Euclidean metric.

The completed evaluation space is
\begin{equation}
\label{eq:joint-state-v2}
\mathfrak X
\eqdef
\left(\bigsqcup_{n\geq1}\{n\}\times X_n\times[n]\right)
\sqcup
\left(\{\infty\}\times X_\infty\times[0,1]\right).
\end{equation}
For \(\xi=(n,z,i)\), set
\(P\xi\eqdef\In_nz\), \(\tau(\xi)\eqdef i/n\), and \(h(\xi)\eqdef1/n\); for
\(\xi=(\infty,x,t)\), set \(P\xi\eqdef x\), \(\tau(\xi)\eqdef t\), and \(h(\xi)\eqdef0\).
Equip \(\mathfrak X\) with
\begin{equation}
\label{eq:joint-metric-v2}
d_{\mathfrak X}(\xi,\xi')
\eqdef
\norm{P\xi-P\xi'}_\infty
+\abs{\tau(\xi)-\tau(\xi')}
+\abs{h(\xi)-h(\xi')}.
\end{equation}
The associated finite-to-boundary convergence criterion is
\begin{equation}
\label{eq:boundary-convergence-v2}
(n_k,z^k,i_k)\to(\infty,x,t)
\quad\Longleftrightarrow\quad
n_k\to\infty,
\quad \In_{n_k}z^k\to x,
\quad i_k/n_k\to t.
\end{equation}

\begin{lemma}[Common-modulus-preserving reconstruction]
\label{lem:modulus-reconstruction}
For every \(n\geq1\), the reconstruction and sampling operators satisfy
\begin{equation}
\label{eq:modulus-section-v2}
\In_nX_n\subset X_\infty,
\qquad
\mathcal S_n(X_\infty)=X_n,
\qquad
\mathcal S_n\circ\In_n=\operatorname{Id}_{X_n}.
\end{equation}
Moreover, for every \(x\in X_\infty\),
\begin{equation}
\label{eq:modulus-sampling-error-v2}
\norm{\In_n\mathcal S_nx-x}_\infty\leq\omega(1/n).
\end{equation}
\end{lemma}

\begin{proof}
Fix \(z=(z_1,\ldots,z_n)\in X_n\), put \(h\eqdef1/n\), and introduce knot
values \(w_0=w_1\eqdef z_1\) and \(w_j\eqdef z_j\) for \(2\leq j\leq n\).  Convexity
of \(\Omega\) ensures that \(\In_nz\) remains in \(\Omega\).  The discrete
common-modulus condition also gives
\begin{equation}
\label{eq:extended-knot-modulus-v2}
\norm{w_\ell-w_k}\leq\omega(h\abs{\ell-k}),
\qquad 0\leq k,\ell\leq n;
\end{equation}
for \(k=0<\ell\), this follows from
\(
\norm{z_\ell-z_1}\leq\omega((\ell-1)h)\leq\omega(\ell h)
\), and all other cases are immediate.

Let \(U\sim\Unif[0,1)\) and, for \(r\in[0,1]\), define the randomized rounding
\[
K_r
\eqdef
\lfloor nr\rfloor
+\mathbf{1}_{\{U<nr-\lfloor nr\rfloor\}}.
\]
Then \(K_r\) is supported on the grid indices bracketing \(nr\),
\(\mathbb{E}K_r=nr\), and
\(
(\In_nz)(r)=\mathbb{E}w_{K_r}
\); this also holds on \([0,h]\) because \(w_0=w_1\).  Equivalently,
\(K_r=\lceil nr-U\rceil\), so \(r\mapsto K_r\) is nondecreasing and
\(s\leq t\) implies \(K_s\leq K_t\) almost surely.  Fix
\(0\leq s\leq t\leq1\).  Since \(h(K_t-K_s)\in[0,1]\),
\eqref{eq:extended-knot-modulus-v2} and Jensen's inequality for the concave
function \(\omega\) give
\begin{align*}
\norm{(\In_nz)(t)-(\In_nz)(s)}
&\leq \mathbb{E}\norm{w_{K_t}-w_{K_s}}\\
&\leq \mathbb{E}\,\omega\!\left(h(K_t-K_s)\right)\\
&\leq \omega\!\left(h\,\mathbb{E}[K_t-K_s]\right)
=\omega(t-s).
\end{align*}
The reverse ordering follows by symmetry.  Thus \(\In_nz\in X_\infty\), and
evaluation at the grid points gives \(\mathcal S_n\In_nz=z\).

Now take \(x\in X_\infty\).  Its grid samples obey the all-pairs condition
in \eqref{eq:Xn-v2}, so \(\mathcal S_nx\in X_n\).  Together with the
preceding right-inverse identity, this proves
\eqref{eq:modulus-section-v2}.  On \([0,h]\), the reconstruction
\(\In_n\mathcal S_nx\) equals \(x(h)\) and differs from \(x(t)\) by at most
\(\omega(h)\).  On every later grid cell it is a convex combination of the
two endpoint samples, each of which is within \(\omega(h)\) of \(x(t)\).
This proves \eqref{eq:modulus-sampling-error-v2}.
\end{proof}

For a continuous causal path target \(F_\infty^\star\), define its canonical
finite restrictions by
\begin{equation}
\label{eq:canonical-discretization-v2}
F_n^\dagger(z)_i
\eqdef
F_\infty^\star(\In_nz)(i/n).
\end{equation}

\begin{lemma}[Canonical completion of a continuous causal target]
\label{lem:canonical-completion}
Let \(F_\infty^\star:X_\infty\to C([0,1];\R^{d'})\) be continuous and
causal, and define \(F_n^\dagger\) by
\eqref{eq:canonical-discretization-v2}.  Then every \(F_n^\dagger\) is
continuous and prefix-causal, and the discrete family
\((F_n^\dagger)_{n\geq1}\) is a continuously extendable causal family with
continuous path extension \(F_\infty^\star\) in the sense of
Definition~\ref{def:compatible-family}.
\end{lemma}

\begin{proof}
Continuity follows from continuity of \(\In_n\), \(F_\infty^\star\), and
evaluation at the finitely many grid points.  If two inputs agree through
index \(i\), their affine interpolants agree on \([0,i/n]\); causality of
\(F_\infty^\star\) therefore gives equality of their \(i\)-th outputs.

It remains to prove the uniform extension condition.  By Arzel\`a--Ascoli,
\(X_\infty\) is compact in the uniform norm. Hence, the image
\(
\mathcal K\eqdef F_\infty^\star(X_\infty)
\)
is compact in \(C([0,1];\R^{d'})\), hence uniformly equicontinuous.  With
\begin{equation*}
\omega_{\mathcal K}(r)
\eqdef
\sup_{g\in\mathcal K}
\sup_{\abs{s-t}\leq r}\norm{g(s)-g(t)},
\end{equation*}
we have \(\omega_{\mathcal K}(r)\to0\) as \(r\downarrow0\).  For
\(z\in X_n\), set \(g\eqdef F_\infty^\star(\In_nz)\in\mathcal K\).  By
definition, \(F_n^\dagger(z)=\mathcal S_ng\), and affine reconstruction on
each grid cell gives
\[
\norm{\In_nF_n^\dagger(z)-F_\infty^\star(\In_nz)}_\infty
=\norm{\In_n\mathcal S_ng-g}_\infty
\leq\omega_{\mathcal K}(1/n).
\]
Taking the supremum over \(X_n\) proves
\eqref{eq:continuous-extension-v3}.
\end{proof}

\begin{lemma}[Compact all-resolution state space]
\label{lem:joint-compactification}
The function \(d_{\mathfrak X}\) in \eqref{eq:joint-metric-v2} is a metric, \((\mathfrak X,d_{\mathfrak X})\) is compact, and the finite states are dense.  The finite-to-boundary convergence criterion is exactly \eqref{eq:boundary-convergence-v2}.
\end{lemma}

\begin{proof}
\smallskip\noindent\emph{Embedding and metric.}\hspace{0.25em}
The path space \(X_\infty\) is compact in the uniform norm by
Arzel\`a--Ascoli: its paths are uniformly bounded and share the modulus
\(\omega\), while the defining conditions are closed under uniform
convergence.  Define
\begin{equation}
\label{eq:compactification-embedding-v2}
J:\mathfrak X\longrightarrow X_\infty\times[0,1]\times S,
\qquad
J(\xi)\eqdef(P\xi,\tau(\xi),h(\xi)).
\end{equation}
Equip the product with the \(\ell^1\) metric
\[
d_{\mathrm{prod}}\bigl((x,t,h),(x',t',h')\bigr)
\eqdef
\norm{x-x'}_\infty+\abs{t-t'}+\abs{h-h'}.
\]
The map is injective.  Indeed, the last coordinate identifies either a finite
length \(n\) or the boundary.  At fixed finite \(n\), the affine interpolant
recovers every \(z_j\) at \(j/n\), and the normalized evaluation time recovers
the index.

\smallskip\noindent\emph{Closed image.}\hspace{0.25em}
The image of \(J\) is closed.  Consider a convergent sequence of image points.  If its resolution code converges to \(1/n>0\), that code is isolated in \(S\), so the lengths are eventually equal to \(n\).  The finite set of evaluation times is closed.  Moreover, \(X_n\) is a closed subset of the compact set \(\Omega^n\), so it is compact and its continuous image \(\In_nX_n\) is compact, hence closed, in \(X_\infty\).  The limit is therefore again a length-\(n\) image point.  If the resolution code converges to zero, the limiting path and time already define a boundary state.  Thus \(J(\mathfrak X)\) is a closed subset of a compact metric space and is compact.  The pullback of this \(\ell^1\) product metric is precisely \eqref{eq:joint-metric-v2}, proving both positive definiteness and compactness.  The same coordinates give \eqref{eq:boundary-convergence-v2}.

\smallskip\noindent\emph{Density of finite states.}\hspace{0.25em}
For density, fix \((\infty,x,t)\).  Set
\[
z_j^n\eqdef x(j/n),
\qquad
i_n\eqdef\max\{1,\min\{n,\operatorname{round}(nt)\}\}.
\]
Then \(z^n\in X_n\),
\(
\norm{\In_nz^n-x}_\infty\leq\omega(1/n)
\), and
\(
\abs{i_n/n-t}\leq1/n
\).
Hence, \((n,z^n,i_n)\to(\infty,x,t)\).
\end{proof}

\subsection{Compatible targets and inherited boundary causality}

We next translate continuity on the compactified state space into the uniform
output convergence needed by the theorem and show that finite causality
survives at the boundary.

For a discrete family and a candidate path extension, define the joint
evaluation by
\begin{equation}
\label{eq:joint-target-v2}
\overline F^\star(\xi)
\eqdef
\begin{cases}
F_n^\star(z)_i,&\xi=(n,z,i),\\
F_\infty^\star(x)(t),&\xi=(\infty,x,t).
\end{cases}
\end{equation}
The required boundary causality is
\begin{equation}
\label{eq:boundary-causality-v2}
x|_{[0,t]}=x'|_{[0,t]}
\quad\Longrightarrow\quad
F_\infty^\star(x)(t)=F_\infty^\star(x')(t).
\end{equation}

\begin{lemma}[Target compatibility and inherited boundary causality]
\label{lem:target-compatibility}
Assume that every \(F_n^\star\) is continuous and that
\(F_\infty^\star:X_\infty\to C([0,1];\R^{d'})\) is continuous.  Then the
following are equivalent:
\begin{enumerate}[label=(\alph*),leftmargin=*,itemsep=2pt,topsep=3pt]
    \item the joint evaluation map \(\overline F^\star\) in \eqref{eq:joint-target-v2} is continuous;
    \item for every \(n_k\to\infty\), \(z^k\in X_{n_k}\), and \(x\in X_\infty\) with \(\In_{n_k}z^k\to x\),
    \begin{equation}
    \label{eq:uniform-output-compatibility-v2}
    \In_{n_k}\bigl(F_{n_k}^\star(z^k)\bigr)
    \longrightarrow F_\infty^\star(x)
    \quad\text{uniformly on }[0,1].
    \end{equation}
    \item the uniform extension condition
    \eqref{eq:continuous-extension-v3} holds.
\end{enumerate}
If the uniform extension condition holds, its continuous path extension is
unique.
If, in addition, the finite family is prefix-causal, then any of these
conditions implies that \(F_\infty^\star\) is causal in the sense of
\eqref{eq:boundary-causality-v2}.
\end{lemma}

\begin{proof}
\smallskip\noindent\emph{From joint continuity to uniform path convergence.}\hspace{0.25em}
Suppose first that \(\overline F^\star\) is continuous.  Since \(\mathfrak X\) is compact, it is uniformly continuous.  If \(\In_{n_k}z^k\to x\), then uniformly in \(i\in[n_k]\),
\begin{equation}
\label{eq:grid-target-modulus-v2}
d_{\mathfrak X}
\bigl((n_k,z^k,i),(\infty,x,i/n_k)\bigr)
=\norm{\In_{n_k}z^k-x}_\infty+1/n_k
\longrightarrow0.
\end{equation}
Consequently,
\begin{equation}
\label{eq:grid-output-close-v2}
\max_{i\in[n_k]}
\norm{F_{n_k}^\star(z^k)_i-F_\infty^\star(x)(i/n_k)}
\longrightarrow0.
\end{equation}
To make the interpolation step explicit, put
\(g\eqdef F_\infty^\star(x)\) and let \(\omega_g\) be its modulus of continuity.
The affine interpolation operator is nonexpansive for the maximum norm, so
\[
 \norm{\In_{n_k}F_{n_k}^\star(z^k)-g}_\infty
 \leq
 \max_{i\in[n_k]}
 \norm{F_{n_k}^\star(z^k)_i-g(i/n_k)}
 +\omega_g(1/n_k).
\]
Together with \eqref{eq:grid-output-close-v2}, this proves
\eqref{eq:uniform-output-compatibility-v2}.

\smallskip\noindent\emph{From uniform path convergence to joint continuity.}\hspace{0.25em}
Conversely, consider a convergent sequence in \(\mathfrak X\).  At a finite
limit, the resolution code and normalized evaluation time force \(n\) and
\(i\) to be eventually constant, so continuity follows from continuity of
\(F_n^\star\).  A sequence converging to a boundary state may alternate
between boundary and finite terms.  On the boundary subsequence, convergence
follows from continuity into the uniform norm and continuity of evaluation.
On the finite subsequence, write
\(
(n_k,z^k,i_k)\to(\infty,x,t)
\).
Condition (b) and \(i_k/n_k\to t\) give
\[
F_{n_k}^\star(z^k)_{i_k}
=(\In_{n_k}F_{n_k}^\star(z^k))(i_k/n_k)
\longrightarrow
F_\infty^\star(x)(t).
\]
Both subsequences have the same limit, so the full sequence converges and \(\overline F^\star\) is continuous.

\smallskip\noindent\emph{Equivalent extension criterion.}\hspace{0.25em}
It remains to compare (b) and (c).  Condition (c) implies (b) because
\begin{align*}
&\norm{
\In_{n_k}F_{n_k}^\star(z^k)-F_\infty^\star(x)
}_\infty\\
&\qquad\leq
\sup_{w\in X_{n_k}}
\norm{
\In_{n_k}F_{n_k}^\star(w)
-F_\infty^\star(\In_{n_k}w)
}_\infty
+\norm{
F_\infty^\star(\In_{n_k}z^k)-F_\infty^\star(x)
}_\infty,
\end{align*}
and both terms tend to zero.  Conversely, if (c) failed, there would be
\(\varepsilon>0\) and lengths \(n_k\to\infty\) for which the supremum in
\eqref{eq:continuous-extension-v3} is at least \(\varepsilon\).  Choose
\(z^k\in X_{n_k}\) with discrepancy at least \(\varepsilon/2\).
Compactness of \(X_\infty\) gives, after extraction,
\(\In_{n_k}z^k\to x\).  Condition (b) makes the first path in that
discrepancy converge to \(F_\infty^\star(x)\), while continuity of
\(F_\infty^\star\) does the same for the second, a contradiction. Hence
(b) and (c) are equivalent.

\smallskip\noindent\emph{Uniqueness.}\hspace{0.25em}
Let \(G_\infty\) and \(\widetilde G_\infty\) be two
continuous path extensions of the same discrete family.  For
\(x\in X_\infty\), set \(z^n\eqdef\mathcal S_nx\) and
\(A_n\eqdef\In_nF_n^\star(z^n)\).  By
Lemma~\ref{lem:modulus-reconstruction}, \(\In_nz^n\to x\), while the two
extension conditions and the triangle inequality give
\[
 \norm{G_\infty(\In_nz^n)-\widetilde G_\infty(\In_nz^n)}_\infty
 \leq
 \norm{G_\infty(\In_nz^n)-A_n}_\infty
 +\norm{A_n-\widetilde G_\infty(\In_nz^n)}_\infty
 \longrightarrow0.
\]
Continuity of both extensions yields
\(G_\infty(x)=\widetilde G_\infty(x)\).

\smallskip\noindent\emph{Inherited boundary causality.}\hspace{0.25em}
Suppose
\(
x|_{[0,t]}=x'|_{[0,t]}
\).
If \(t>0\), sample
\(
z_j^n\eqdef x(j/n)
\)
and
\(
(z')_j^n\eqdef x'(j/n)
\), and set \(i_n\eqdef\lfloor nt\rfloor\) for all sufficiently large \(n\).  The two discrete prefixes agree through \(i_n\), and the corresponding states converge to \((\infty,x,t)\) and \((\infty,x',t)\).  Finite causality and continuity of \(\overline F^\star\) yield equality of the two boundary values.  At \(t=0\), instead use shifted samples
\(
z_j^n\eqdef x((j-1)/n)
\)
and
\(
(z')_j^n\eqdef x'((j-1)/n)
\), and compare the first outputs.  Define \(q_n(s)\eqdef x((s-1/n)_+)\) and
\(q_n'(s)\eqdef x'((s-1/n)_+)\).  These shifted paths lie in \(X_\infty\) because
the time shift is one-Lipschitz and \(\omega\) is nondecreasing.  The
displayed sequences are \(\mathcal S_nq_n\) and \(\mathcal S_nq_n'\), hence
belong to \(X_n\).  Lemma~\ref{lem:modulus-reconstruction} gives
\[
\norm{\In_nz^n-x}_\infty
\leq
\norm{\In_n\mathcal S_nq_n-q_n}_\infty
+\norm{q_n-x}_\infty
\leq2\omega(1/n),
\]
where \(\norm{q_n-x}_\infty\leq\omega(1/n)\).  The same estimate holds for
\(x'\).  Therefore \((n,z^n,1)\to(\infty,x,0)\) and
\((n,(z')^n,1)\to(\infty,x',0)\).  Their first tokens agree because
\(x(0)=x'(0)\); finite prefix causality and continuity of
\(\overline F^\star\) yield equality of the two boundary values at \(t=0\).
\end{proof}

\Needspace{12\baselineskip}

\subsection{Prefix laws identify causal histories}

The proof needs a concrete state variable that forgets the future without
losing any causal information; token--position prefix laws provide exactly
this quotient.

We equip \(\cP(E)\) with the weak topology.  It is compact and Hausdorff because \(E=\Omega\times[0,1]\) is compact metric.

A finite state and a boundary state determine, respectively, the prefix laws
\begin{equation}
\label{eq:finite-prefix-law-v2}
\mu_{n,z,i}
\eqdef
\frac1i\sum_{j=1}^{i}\delta_{(z_j,j/n)}
\end{equation}
and
\begin{equation}
\label{eq:continuous-prefix-law-v2}
\mu_{\infty,x,t}
\eqdef
\begin{cases}
\displaystyle
\frac1t\int_0^t\delta_{(x(s),s)}\,\dd s,&t>0,\\[1.2ex]
\delta_{(x(0),0)},&t=0.
\end{cases}
\end{equation}

\begin{proposition}[Continuous prefix-law representation]
\label{prop:prefix-law}
The map
\begin{equation}
\label{eq:prefix-law-map-v2}
\boldsymbol\mu:\mathfrak X\to\cP(E),
\qquad
\boldsymbol\mu(\xi)\eqdef\mu_\xi,
\end{equation}
defined in \eqref{eq:finite-prefix-law-v2}--\eqref{eq:continuous-prefix-law-v2} is continuous.  Moreover, \(\mu_\xi=\mu_{\xi'}\) if and only if:
\begin{enumerate}[label=(\roman*),leftmargin=*,itemsep=2pt,topsep=3pt]
    \item \(\xi=(n,z,i)\), \(\xi'=(n,z',i)\), and \(z_{1:i}=z'_{1:i}\); or
    \item \(\xi=(\infty,x,t)\), \(\xi'=(\infty,x',t)\), and \(x|_{[0,t]}=x'|_{[0,t]}\).
\end{enumerate}
Thus \(\cM=\boldsymbol\mu(\mathfrak X)\) is compact, and \(\boldsymbol\mu\) induces a canonical homeomorphism from \(\mathfrak X/\!\sim_{\mathrm{pref}}\), equipped with the quotient topology, onto \(\cM\).  The tagged endpoint also descends to a unique continuous map
\begin{equation}
\label{eq:prefix-endpoint-extractor-v2}
\mathfrak e:\cM\to E,
\qquad
\mathfrak e(\mu_\xi)
=\bigl(P\xi(\tau(\xi)),\tau(\xi)\bigr).
\end{equation}
Writing \(\tau(\mu)\) for the time coordinate of \(\mathfrak e(\mu)\), the
resolution code descends to the continuous map \(\mathfrak h:\cM\to S\)
given by
\begin{equation}
\label{eq:prefix-resolution-extractor-v2}
\mathfrak h(\mu)
\eqdef 2\int_E s\,\dd\mu(x,s)-\tau(\mu),
\qquad
\mathfrak h(\mu_\xi)=h(\xi).
\end{equation}
For every continuously extendable causal family, completed by its unique path
extension, there is a unique continuous
\begin{equation}
\label{eq:factored-target-v2}
f^\star:\cM\to\R^{d'}
\quad\text{such that}\quad
\overline F^\star=f^\star\circ\boldsymbol\mu.
\end{equation}
\end{proposition}

\begin{proof}
\smallskip\noindent\emph{Continuity on fixed strata.}\hspace{0.25em}
At a finite state, the resolution and evaluation index are locally constant
in \(\mathfrak X\), after which continuity follows directly from continuity
of the atoms in \eqref{eq:finite-prefix-law-v2}.  On the boundary, let
\((x_k,t_k)\to(x,t)\) uniformly and in time.  For every \(g\in C(E)\),
integration against \(\mu_{\infty,x_k,t_k}\) on the subsequence with
\(t_k=0\) is evaluation at \((x_k(0),0)\), which converges to
\(g(x(0),0)\).  On the subsequence with \(t_k>0\), uniform continuity of
\(g\) gives convergence of
\(
t_k^{-1}\int_0^{t_k}g(x_k(s),s)\,\dd s
\)
to the corresponding integral when \(t>0\).  If \(t=0\), the entire integration interval shrinks to zero and the paths converge uniformly, so the averages converge to \(g(x(0),0)\).  This proves weak continuity on the boundary.

\smallskip\noindent\emph{Continuity across resolutions.}\hspace{0.25em}
Now consider finite-to-boundary convergence
\(
(n_k,z^k,i_k)\to(\infty,x,t)
\).
Let \(t_k\eqdef i_k/n_k\), and define the step path
\begin{equation}
\label{eq:position-token-step-v2}
y_k(s)\eqdef\left(z_j^k,j/n_k\right)
\quad\text{for }s\in((j-1)/n_k,j/n_k],
\qquad
y_k(0)\eqdef(z_1^k,1/n_k).
\end{equation}
The token component of \(y_k\) differs from \(\In_{n_k}z^k\) by at most
\(\omega(1/n_k)\). Hence, \(\In_{n_k}z^k\to x\) implies
\begin{equation}
\label{eq:position-token-step-conv-v2}
\norm{y_k-[s\mapsto(x(s),s)]}_\infty
\leq
\norm{\In_{n_k}z^k-x}_\infty
+\omega(1/n_k)+n_k^{-1}
\longrightarrow0.
\end{equation}
For \(t_k>0\), the empirical law \(\mu_{n_k,z^k,i_k}\) is exactly the pushforward of \(\Unif[0,t_k]\) by \(y_k\).  If \(t>0\), \eqref{eq:position-token-step-conv-v2} and \(t_k\to t\) give weak convergence to \eqref{eq:continuous-prefix-law-v2}.  If \(t=0\), then
\[
\sup_{0\leq s\leq t_k}
\norm{y_k(s)-(x(0),0)}
\leq
\norm{y_k-[s\mapsto(x(s),s)]}_\infty
+\sup_{0\leq s\leq t_k}\bigl(\norm{x(s)-x(0)}+s\bigr)
\longrightarrow0.
\]
Hence, the laws converge to the required Dirac mass.  A general sequence approaching a boundary state splits into boundary and finite subsequences; the two arguments above give the same limiting law.  Thus \(\boldsymbol\mu\) is continuous everywhere.

\smallskip\noindent\emph{Identification of the fibers.}\hspace{0.25em}
Project a prefix law onto its time coordinate.  A finite state gives
\begin{equation}
\label{eq:finite-time-marginal-v2}
\frac1i\sum_{j=1}^{i}\delta_{j/n}.
\end{equation}
Its minimum support point is \(1/n\), which determines \(n\), and its number of atoms then determines \(i\).  At each distinct time \(j/n\), equality of the joint laws identifies the unique attached token \(z_j\).  Thus equality of two finite laws is equivalent to equality of the tagged finite prefixes.

A boundary state has time marginal \(\Unif[0,t]\) when \(t>0\), and \(\delta_0\) when \(t=0\).  These marginals determine \(t\) and cannot equal the finite marginal \eqref{eq:finite-time-marginal-v2}.  At \(t=0\), equality of the graph Dirac masses identifies \(x(0)\).  At \(t>0\), equality of the graph laws implies, after cancelling the common normalization factor \(1/t\), that for every continuous \(\varphi:[0,1]\to\R\) and every coordinate \(\ell\),
\[
\int_0^t\varphi(s)x_\ell(s)\,\dd s
=
\int_0^t\varphi(s)x'_\ell(s)\,\dd s.
\]
Use the explicit continuous extension
\[
\varphi(s)\eqdef
\begin{cases}
x_\ell(s)-x'_\ell(s),&0\leq s\leq t,\\
x_\ell(t)-x'_\ell(t),&t\leq s\leq1.
\end{cases}
\]
The displayed identity gives
\(
\int_0^t\abs{x_\ell(s)-x'_\ell(s)}^2\,\dd s=0
\).
Continuity yields equality at every time, so the paths agree on \([0,t]\).  This proves the fiber statement.

\smallskip\noindent\emph{Quotient, endpoint, and target factorization.}\hspace{0.25em}
Compactness of \(\cM\) follows because it is the continuous image of compact
\(\mathfrak X\).  Let \(q:\mathfrak X\to\mathfrak X/\!\sim_{\mathrm{pref}}\)
be the quotient map.  The fiber statement gives a unique bijection
\(\widetilde{\boldsymbol\mu}\) satisfying
\(\boldsymbol\mu=\widetilde{\boldsymbol\mu}\circ q\).  It is continuous by
the definition of the quotient topology.  Its domain is compact as a quotient
of compact \(\mathfrak X\), while \(\cM\) is Hausdorff; hence it is a
homeomorphism.

The tagged endpoint map
\[
\xi\longmapsto
\bigl(P\xi(\tau(\xi)),\tau(\xi)\bigr)
\]
is continuous: this follows from uniform convergence of the path coordinate,
convergence of the time coordinate, and continuity of the limiting path.  The
fiber characterization shows that it is constant on the fibers of
\(\boldsymbol\mu\).  Since \(\boldsymbol\mu\) is a quotient map, it therefore
descends uniquely to the continuous map \(\mathfrak e\) in
\eqref{eq:prefix-endpoint-extractor-v2}.

For a finite prefix, the mean time is \((i+1)/(2n)\), while its endpoint
time is \(i/n\).  Their combination in
\eqref{eq:prefix-resolution-extractor-v2} therefore equals \(1/n\).  At a
boundary prefix the mean time is \(t/2\), also at \(t=0\) under the Dirac
convention, so the same expression is zero.  Continuity follows from weak
continuity of the time moment and continuity of the endpoint extractor.
This formula is directly accessible to the architecture: one zero-score
scalar head with time as its value writes \(\int_E s\,\dd\mu\) into an
initially zero work coordinate, while the residual coordinates retain the
current time.  A pointwise affine combination then recovers \(1/n\) on
finite prefixes and zero at the path boundary, without a separate length
channel.

Finite prefix causality and the inherited boundary causality from Lemma~\ref{lem:target-compatibility} show that \(\overline F^\star\) is constant on every fiber of \(\boldsymbol\mu\), so \(f^\star\) in \eqref{eq:factored-target-v2} is well-defined and unique.  The continuous surjection \(\boldsymbol\mu\) from compact \(\mathfrak X\) to Hausdorff \(\cM\) is a quotient map; hence \(f^\star\) is continuous.
\end{proof}

\subsection{Discrete attention closes under fixed-depth compositions}

Architectural compatibility rests on one stability fact: masked sums converge
uniformly to temporal attention, and this convergence survives every fixed
composition.

For \(u=(u_1,\ldots,u_n)\in(\R^p)^n\), define its step interpolation
\begin{equation}
\label{eq:generic-step-v2}
(\Cn_nu)(s)\eqdef u_j
\quad\text{for }s\in((j-1)/n,j/n],
\qquad
(\Cn_nu)(0)\eqdef u_1.
\end{equation}

\begin{lemma}[Uniform stability of masked attention]
\label{lem:attention-stability}
Let \(n_k\to\infty\), let \(u^k\in(\R^p)^{n_k}\), and suppose
\begin{equation}
\label{eq:hidden-step-assumption-v2}
\norm{\Cn_{n_k}u^k-u}_\infty\longrightarrow0
\end{equation}
for some \(u\in C([0,1];\R^p)\).  For every fixed set of masked multi-head parameters \(\theta\),
\begin{equation}
\label{eq:attention-step-limit-v2}
\norm{
\Cn_{n_k}\bigl(\MAtt_\theta(u^k)\bigr)
-\MAtt_\theta^\infty(u)
}_\infty
\longrightarrow0.
\end{equation}
\end{lemma}

\begin{proof}
\smallskip\noindent\emph{Exact step-path identity.}\hspace{0.25em}
Extend \eqref{eq:temporal-attention-v2} to bounded Borel paths using the same
Lebesgue-integral ratios and endpoint convention, and denote this extension
by \(\MAtt_\theta^{\infty,\mathrm{B}}\).  At the grid point \(i/n_k\),
each cell of the prefix has length \(1/n_k\).  Multiplying a discrete softmax
numerator and denominator by \(1/n_k\) gives the exact identity
\begin{equation}
\label{eq:exact-step-identity-v2}
\MAtt_\theta(u^k)_i
=
\MAtt_\theta^{\infty,\mathrm{B}}(\Cn_{n_k}u^k)(i/n_k).
\end{equation}

\smallskip\noindent\emph{Joint continuity of attention evaluation.}\hspace{0.25em}
Suppose \(v_k\to v\)
uniformly, where each \(v_k\) is a bounded Borel path, \(v\) is continuous,
and \(t_k\to t\).  We claim that
\begin{equation}
\label{eq:step-attention-evaluation-continuity-v2}
\MAtt_\theta^{\infty,\mathrm{B}}(v_k)(t_k)
\longrightarrow
\MAtt_\theta^\infty(v)(t).
\end{equation}
First,
\(
v_k(t_k)\to v(t)
\), and for each head the score and score--value integrands
\[
s\longmapsto
e^{\ip{Q_rv_k(t_k)}{K_rv_k(s)}},
\qquad
s\longmapsto
e^{\ip{Q_rv_k(t_k)}{K_rv_k(s)}}V_rv_k(s)
\]
converge uniformly on \([0,1]\) to their counterparts built from \(v(t)\)
and \(v(s)\).  The paths are uniformly bounded, so for some \(C<\infty\)
all scores lie in \([-C,C]\).  If \(t>0\), then eventually \(t_k\geq t/2\),
and every denominator is at least \((t/2)e^{-C}>0\).  For either the scalar
denominator integrand or the vector numerator integrand, uniform convergence
\(g_k\to g\) gives
\[
 \left\|\int_0^{t_k}g_k(s)\,\dd s-\int_0^t g(s)\,\dd s\right\|
 \leq \min\{t_k,t\}\norm{g_k-g}_\infty
 +\abs{t_k-t}\max\{\norm{g_k}_\infty,\norm{g}_\infty\}
 \longrightarrow0.
\]
Thus both integrals, and hence their ratios, converge.  If \(t=0\), then on
indices with \(t_k=0\), the endpoint convention gives
\(V_rv_k(0)\to V_rv(0)\).  On indices with \(t_k>0\), each normalized head
output is a positive weighted average of \(V_rv_k(s)\) over \([0,t_k]\),
whose distance from \(V_rv(0)\) is bounded by
\[
\norm{V_r}\left(
\norm{v_k-v}_\infty
+\sup_{0\leq s\leq t_k}\norm{v(s)-v(0)}
\right)\longrightarrow0.
\]
Including the residual term proves
\eqref{eq:step-attention-evaluation-continuity-v2} at every \(t\).

\smallskip\noindent\emph{Uniform convergence on the grid.}\hspace{0.25em}
Suppose the grid error did not tend to zero.
Choose offending indices \(i_k\) and pass to a subsequence such that
\(t_k=i_k/n_k\to t\).  With \(v_k=\Cn_{n_k}u^k\),
\eqref{eq:step-attention-evaluation-continuity-v2} and
\eqref{eq:exact-step-identity-v2} give
\[
\MAtt_\theta(u^k)_{i_k}
\longrightarrow \MAtt_\theta^\infty(u)(t).
\]
Applying the joint-continuity claim to the constant sequence \(v_k=u\)
shows that \(\MAtt_\theta^\infty(u)\) is continuous. Hence
\(
\MAtt_\theta^\infty(u)(t_k)
\to\MAtt_\theta^\infty(u)(t)
\), contradicting the offending grid error.  Hence that error tends to zero
uniformly.  Finally, \(\MAtt_\theta^\infty(u)\) is uniformly continuous;
replacing each time by the right endpoint of its grid cell changes it by
\(o(1)\).  This upgrades the grid estimate to the step-path estimate
\eqref{eq:attention-step-limit-v2}.
\end{proof}

\begin{corollary}[Closure of every fixed transformer]
\label{cor:transformer-closure}
Let \(T_\Theta\) be any fixed finite composition of masked-attention blocks,
shared tokenwise affine maps, and pointwise ReLU networks, and let
\(T_{\Theta,\infty}\) use the same parameters with temporal attention.  Let
\(n_k\to\infty\), let \(z^k\in X_{n_k}\), and suppose
\(
\In_{n_k}z^k\to x\in X_\infty
\).
Then
\begin{equation}
\label{eq:full-transformer-closure-v2}
\norm{
\Cn_{n_k}T_\Theta(\phi_{n_k}(z^k))
-T_{\Theta,\infty}(\phi_\infty(x))
}_\infty
\longrightarrow0.
\end{equation}
\end{corollary}

\begin{proof}
The token step path differs from its affine reconstruction by at most
\(\omega(1/n_k)\).  Therefore the embedded input step paths satisfy
\[
\norm{\Cn_{n_k}\phi_{n_k}(z^k)-\phi_\infty(x)}_\infty
\leq
\norm{\In_{n_k}z^k-x}_\infty+\omega(1/n_k)+n_k^{-1}
\longrightarrow0.
\]
Lemma~\ref{lem:attention-stability} propagates uniform convergence through
every attention block.  If \(G:\R^p\to\R^q\) underlies either a shared
tokenwise affine map or a shared pointwise ReLU network, then
\(
\Cn_n(G(u_1),\ldots,G(u_n))=G\circ\Cn_nu
\).
Moreover, every such fixed \(G\) is globally Lipschitz, so
\[
\norm{G\circ v_k-G\circ v}_\infty
\leq \operatorname{Lip}(G)\norm{v_k-v}_\infty.
\]
Thus uniform convergence propagates through each pointwise block.  Induction over the finite list of layers proves \eqref{eq:full-transformer-closure-v2}.
\end{proof}

\begin{corollary}[A fixed transformer is an all-resolution causal family]
\label{cor:transformer-compatible-family}
Fix parameters \(\Theta\) for a finite causal transformer with output dimension \(m\), and define
\begin{equation}
\label{eq:transformer-induced-family-v2}
\mathcal T_{\Theta,n}(z)_i
\eqdef T_\Theta(\phi_n(z))_i,
\qquad
\mathcal T_{\Theta,\infty}(x)(t)
\eqdef T_{\Theta,\infty}(\phi_\infty(x))(t).
\end{equation}
Then the discrete family
\((\mathcal T_{\Theta,n})_{n\geq1}\) is a continuously extendable causal
family in the sense of Definition~\ref{def:compatible-family}, with continuous
path extension \(\mathcal T_{\Theta,\infty}\).
\end{corollary}

\begin{proof}
At every finite resolution, continuity follows from the finite composition of continuous operations, and prefix causality follows from the mask by induction over the layers.  For temporal attention, the change of variables \(s=ta\) writes each head ratio, for \(t>0\), as
\[
\frac{\int_0^1 e^{\ip{Q_ru(t)}{K_ru(ta)}}V_ru(ta)\,\dd a}
{\int_0^1 e^{\ip{Q_ru(t)}{K_ru(ta)}}\,\dd a}.
\]
At \(t=0\), the same expression equals \(V_ru(0)\), exactly the endpoint convention \eqref{eq:temporal-attention-zero-v2}.  If \(u_k\to u\) uniformly, choose \(B\) such that \(\norm{u_k}_\infty,\norm{u}_\infty\leq B\).  The numerator and denominator integrands converge uniformly on \([0,1]^2\), and the denominator of head \(r\) is bounded below explicitly by
\[
\int_0^1
e^{\ip{Q_ru_k(t)}{K_ru_k(ta)}}\,\dd a
\geq
\exp\!\left(-\norm{Q_r}\norm{K_r}B^2\right)>0.
\]
Thus temporal attention maps continuous paths to continuous paths and is
continuous in the uniform norm.  Shared tokenwise affine maps and pointwise
ReLU networks preserve these properties, so
\(\mathcal T_{\Theta,\infty}:X_\infty\to C([0,1];\R^m)\) is continuous.  The
temporal family is causal by the same layerwise induction, because its value
at time \(t\) uses hidden states only on \([0,t]\).

It remains to check cross-resolution compatibility.  If \(n_k\to\infty\), \(z^k\in X_{n_k}\), and \(\In_{n_k}z^k\to x\), Corollary~\ref{cor:transformer-closure} gives uniform convergence of the step-interpolated outputs to \(\mathcal T_{\Theta,\infty}(x)\).  Hence, whenever \(i_k/n_k\to t\),
\[
\mathcal T_{\Theta,n_k}(z^k)_{i_k}
\longrightarrow
\mathcal T_{\Theta,\infty}(x)(t).
\]
Together with the finite and boundary continuity already proved, this is
continuity of the joint evaluation map on \(\mathfrak X\).
Lemma~\ref{lem:target-compatibility} converts this joint continuity into the
uniform interpolated-output condition in
Definition~\ref{def:compatible-family}, so that definition applies.
\end{proof}

\subsection{Log-Laplace probes separate prefix laws}

Universality now reduces to finding continuous, attention-computable
coordinates that separate the compact prefix-law space.

For \(a\in\R^{d+1}\) and \(c\in[0,1]\), define
\begin{equation}
\label{eq:laplace-probe-v2}
\mathcal L_{a,c}(\mu)
\eqdef
\frac{\int_E e^{c\ip{a}{y}}\ip{a}{y}\,\dd\mu(y)}
{\int_E e^{c\ip{a}{y}}\,\dd\mu(y)}.
\end{equation}

\begin{lemma}[Log-Laplace probes determine a compactly supported law]
\label{lem:laplace-probes}
For every \(a\in\R^{d+1}\) and \(c\in[0,1]\), the function \(\mathcal L_{a,c}\) in \eqref{eq:laplace-probe-v2} is continuous on \(\cP(E)\).  If \(\mu,\nu\in\cP(E)\) satisfy
\begin{equation}
\label{eq:all-probes-equal-v2}
\mathcal L_{a,c}(\mu)=\mathcal L_{a,c}(\nu)
\qquad\text{for every }a\in\R^{d+1}\text{ with }\norm{a}\leq1,
\quad c\in[0,1],
\end{equation}
then \(\mu=\nu\).
\end{lemma}

\begin{proof}
The numerator and denominator in \eqref{eq:laplace-probe-v2} are weakly continuous because their integrands are continuous on compact \(E\), and the denominator is strictly positive.  Hence \(\mathcal L_{a,c}\) is continuous.

For every \(b\in\R^{d+1}\), the fundamental theorem of calculus gives the exact identity
\begin{equation}
\label{eq:log-derivative-v2}
\log\!\int_E e^{\ip{b}{y}}\,\dd\mu(y)
=\int_0^1 \mathcal L_{b,c}(\mu)\,\dd c,
\end{equation}
because the logarithmic moment-generating function vanishes at \(c=0\).

Equality of the probes with \(\norm{a}\leq1\) therefore gives
\[
M_\mu(b)=M_\nu(b)
\qquad
\text{for every }b\in\R^{d+1}\text{ with }\norm{b}\leq1,
\]
where
\[
M_\mu(b)\eqdef\int_E e^{\ip{b}{y}}\,\dd\mu(y)
\]
is the multivariate Laplace transform of \(\mu\).
Since \(E\) is compact, differentiation under the integral is justified to
every order.  The transforms agree on a neighborhood of the origin, so for
every multi-index
\(\alpha\in\N_0^{d+1}\),
\[
\partial^\alpha M_\mu(0)
=
\int_E y^\alpha\,\dd\mu(y)
=
\int_E y^\alpha\,\dd\nu(y)
=
\partial^\alpha M_\nu(0).
\]
Hence \(\mu\) and \(\nu\) agree on the integrals of all coordinate
polynomials. The restrictions of coordinate polynomials to \(E\) form a
unital point-separating algebra and are therefore uniformly dense in
\(C(E)\) by the real Stone--Weierstrass theorem. Consequently,
\[
\int_E g\,\dd\mu=\int_E g\,\dd\nu
\qquad
\text{for every }g\in C(E).
\]
It follows that \(\mu=\nu\).

\end{proof}

\subsection{Parallel realization by one masked attention layer}

Separation becomes useful only if all selected coordinates can be evaluated
by one shared shallow block at every resolution.

\begin{lemma}[Exact parallel probes with ReLU readout]
\label{lem:parallel-realization}
Fix an integer \(H\geq1\), probe parameters \((a_r,c_r)_{r=1}^H\), and a
continuous map \(\Pi:E\times\R^H\to\R^{d'}\).  For every \(\delta>0\), there
are an integer \(M\geq1\) and parameters
\(\Theta=(\theta,\eta)\), independent
of \(n\), of the exact shallow form
\eqref{eq:theorem-lift-v2}--\eqref{eq:shallow-transformer-v2}, such that
\begin{align}
\label{eq:parallel-finite-v2}
\sup_{n\geq1}\ \sup_{z\in X_n}\ \max_{i\in[n]}
\norm{
T_\Theta(\phi_n(z))_i
-\Pi\bigl((z_i,i/n),\mathcal L_{a_1,c_1}(\mu_{n,z,i}),\ldots,\mathcal L_{a_H,c_H}(\mu_{n,z,i})\bigr)
}
&<\delta,\\
\label{eq:parallel-continuous-v2}
\sup_{x\in X_\infty}\ \sup_{t\in[0,1]}
\norm{
T_{\Theta,\infty}(\phi_\infty(x))(t)
-\Pi\bigl((x(t),t),\mathcal L_{a_1,c_1}(\mu_{\infty,x,t}),\ldots,\mathcal L_{a_H,c_H}(\mu_{\infty,x,t})\bigr)
}
&<\delta.
\end{align}
If \(\norm{a_r}\leq1\) for every \(r\), the attention matrices can moreover
be chosen with the Euclidean operator-norm bounds
\begin{equation}
\label{eq:bounded-qualitative-heads}
\norm{Q_r}_{\mathrm{op}},\quad \norm{K_r}_{\mathrm{op}},\quad
\norm{V_r}_{\mathrm{op}},\quad \norm{W_r}_{\mathrm{op}}\leq1.
\end{equation}
\end{lemma}

\begin{proof}
\smallskip\noindent\emph{Probe heads.}\hspace{0.25em}
Write \(y=(x,s)\in E\), use \(p_H=d+H+2\), and identify the first
\(d+1\) coordinates of \(\R^{p_H}\) with \(y\).  The fixed tokenwise lift
\(\mathcal E_H\) in \eqref{eq:theorem-lift-v2} appends exactly one
homogeneous coordinate and \(H\) zero work coordinates.
For head \(r\), use scalar query, key, and value dimensions and choose
\begin{equation}
\label{eq:probe-head-parameters-v2}
\begin{aligned}
Q_r&\eqdef(0_{d+1},1,0_H),
&K_r&\eqdef(c_ra_r^\top,0,0_H),\\
V_r&\eqdef(a_r^\top,0,0_H),
&W_r&\eqdef e_{d+2+r}.
\end{aligned}
\end{equation}
Thus \(Q_r,K_r,V_r\in\R^{1\times p_H}\),
\(W_r\in\R^{p_H\times1}\), and
\[
Q_r(y,1,0_H)=1,
\qquad
K_r(y,1,0_H)=c_r\ip{a_r}{y},
\qquad
V_r(y,1,0_H)=\ip{a_r}{y}.
\]
The column \(W_r\) writes only into work coordinate \(r\): its
content--position and homogeneous rows are zero, so the residual addition
leaves \((y,1)\) unchanged.
The parameter rows in \eqref{eq:probe-head-parameters-v2} also give
\(\norm{Q_r}_{\mathrm{op}}=\norm{W_r}_{\mathrm{op}}=1\),
\(\norm{K_r}_{\mathrm{op}}=c_r\norm{a_r}\), and
\(\norm{V_r}_{\mathrm{op}}=\norm{a_r}\), proving
\eqref{eq:bounded-qualitative-heads} when the directions lie in the unit ball.

At a finite state \((n,z,i)\), the query
\(Q_r(\phi_n(z)_i,1,0_H)\) is exactly one.  The \(r\)-th work residual
begins at zero, so the output of head \(r\) in that coordinate is exactly
\[
\frac{\sum_{j\leq i}e^{c_r\ip{a_r}{(z_j,j/n)}}
\ip{a_r}{(z_j,j/n)}}
{\sum_{j\leq i}e^{c_r\ip{a_r}{(z_j,j/n)}}}
=\mathcal L_{a_r,c_r}(\mu_{n,z,i}).
\]
The heads have disjoint work ranges.  Consequently, the residual block output at \((n,z,i)\) is exactly
\[
\Bigl((z_i,i/n),1,
\mathcal L_{a_1,c_1}(\mu_{n,z,i}),\ldots,
\mathcal L_{a_H,c_H}(\mu_{n,z,i})\Bigr).
\]
Thus all probes are computed in parallel while \((y,1)\) is preserved.  With the same matrices, temporal attention gives the integral expression \(\mathcal L_{a_r,c_r}(\mu_{\infty,x,t})\) for \(t>0\).  At \(t=0\), the endpoint rule gives \(\ip{a_r}{(x(0),0)}\), which is exactly the probe evaluated at the Dirac law in \eqref{eq:continuous-prefix-law-v2}.

\smallskip\noindent\emph{Shared readout.}\hspace{0.25em}
The joint vector consisting of the current lifted token and all probe values
ranges over the compact set
\[
 K
 \eqdef\left\{
 \bigl(\mathfrak e(\mu),
 \mathcal L_{a_1,c_1}(\mu),\ldots,
 \mathcal L_{a_H,c_H}(\mu)\bigr):\mu\in\cM
 \right\}
 \subset E\times\R^H,
\]
because \(\cM\) is compact and every displayed coordinate is continuous.
Choose a compact box \(B\subset\R^{p_H-1}\) containing \(K\).  Coordinatewise
Tietze extension gives a continuous map on \(B\) agreeing with \(\Pi\) on
\(K\).  Apply the classical one-hidden-layer ReLU universal approximation
theorem~\citep{leshno1993nonpolynomial} to that extension on \(B\).  It gives
a finite hidden width \(M\); after inserting a zero column for the homogeneous
input coordinate, its matrices have dimensions
\[
 A_1\in\R^{M\times p_H},\quad b_1\in\R^M,
 \qquad
 A_2\in\R^{d'\times M},\quad b_2\in\R^{d'}
\]
and \(\MLP_\eta(u)=A_2\rho(A_1u+b_1)+b_2\) uniformly approximates
\(\Pi(y,v)\) to Euclidean error \(\delta\) whenever
\(u=(y,1,v)\) and \((y,v)\in K\).  The biases \(b_1,b_2\) are already part
of \eqref{eq:pointwise-mlp-v2}.  Pointwise
application after the attention block proves both estimates.  This
qualitative invocation gives no target-independent or
\(\delta\)-independent bound on \(M\).
\end{proof}

\begin{lemma}[Sequential width--depth realization of probe polynomials]
\label{lem:serialized-realization}
Fix probes \((a_r,c_r)_{r=1}^H\) and a vector polynomial
\(\Pi:\R^H\to\R^{d'}\).  For every \(\delta>0\), there is a finite-depth
causal transformer \(S_\delta\), obtained by alternating masked-attention
blocks and shared pointwise one-hidden-layer ReLU maps between a shared
tokenwise affine initialization and a shared tokenwise affine readout, such
that its
intermediate residual width is
\begin{equation}
\label{eq:serialized-width}
 p_{\mathrm{seq}}\eqdef d+1+3d',
\end{equation}
every attention block has at most \(d'\) scalar heads, and
\begin{align}
\label{eq:serialized-finite}
\sup_{n\geq1}\sup_{z\in X_n}\max_{i\in[n]}
\left\|
S_\delta(\phi_n(z))_i
-\Pi\bigl(\mathcal L_{a_1,c_1}(\mu_{n,z,i}),\ldots,
\mathcal L_{a_H,c_H}(\mu_{n,z,i})\bigr)
\right\|&<\delta,\\
\label{eq:serialized-continuous}
\sup_{x\in X_\infty}\sup_{t\in[0,1]}
\left\|
S_{\delta,\infty}(\phi_\infty(x))(t)
-\Pi\bigl(\mathcal L_{a_1,c_1}(\mu_{\infty,x,t}),\ldots,
\mathcal L_{a_H,c_H}(\mu_{\infty,x,t})\bigr)
\right\|&<\delta.
\end{align}
The block count depends on the polynomial expansion, while the pointwise hidden
widths and parameters may depend on the fixed probes, \(\Pi\), and \(\delta\).
Neither the width in
\eqref{eq:serialized-width} nor the per-block head count depends on sequence
length.
If all directions satisfy \(\norm{a_r}\leq1\), every scalar head can also
satisfy the operator-norm bounds in \eqref{eq:bounded-qualitative-heads}.
\end{lemma}

\begin{proof}
\smallskip\noindent\emph{Registers.}\hspace{0.25em}
Write each output coordinate as a finite sum of monomials,
\begin{equation}
\label{eq:serialized-polynomial-expansion}
 \Pi_\ell(v)
 =b_\ell+\sum_{q=1}^{N_\ell}b_{\ell q}
 \prod_{s=1}^{m_{\ell q}}v_{r(\ell,q,s)},
 \qquad \ell\in[d'],
\end{equation}
where constant monomials have been absorbed into \(b_\ell\).  Append to the
\(d+1\) token--position coordinates three \(d'\)-dimensional registers
\((q,p,s)\), initialized tokenwise as
\begin{equation}
\label{eq:serialized-register-initialization}
 (y,q,p,s)\eqdef(y,\mathbf 1_{d'},\mathbf 1_{d'},0_{d'}).
\end{equation}
They hold, respectively, the current probes, running products, and completed
sums.

\smallskip\noindent\emph{Execution schedule.}\hspace{0.25em}
Schedule the finitely many factors in
\eqref{eq:serialized-polynomial-expansion}, processing the output coordinates
in parallel and leaving a coordinate idle when it has no factor at the
current stage.  For every active output coordinate \(\ell\), one scalar head
uses the \(\ell\)-th probe register as its query, the unchanged
token--position coordinates for its key and value, and writes back only to
that probe register.  If the scheduled factor is
\(v_r=\mathcal L_{a_r,c_r}\), choose rows so that
\[
 Q u=q_\ell=1,
 \qquad Ku'=c_r\ip{a_r}{y'},
 \qquad Vu'=\ip{a_r}{y'},
\]
and choose \(W\) to be the probe-register basis vector for coordinate
\(\ell\).  This is the calculation in
\eqref{eq:probe-head-parameters-v2}.  After the residual addition, that
register contains \(1+v_r\).  There is at most one active head for each
output coordinate, hence at most \(d'\) heads in the block.
The query row selects the probe register and the output column writes to
that register, both with operator norm one.  The key and value rows have
norms \(c_r\norm{a_r}\) and \(\norm{a_r}\).
Thus unit-ball probe directions give the asserted per-head bounds here too.

The following pointwise map resets \(q_\ell\) to one and performs the ideal
update
\[
 p_\ell^+\eqdef p_\ell(q_\ell-1)
\]
when the monomial continues.  At its final factor, it instead adds
\(b_{\ell q}p_\ell(q_\ell-1)\) to \(s_\ell\) and resets \(p_\ell\) to one.
Idle coordinates are unchanged.  These updates are polynomial maps on the
finite-dimensional register state.  On any compact box containing the exact
finite-stage trajectories, a one-hidden-layer ReLU map approximates their
nonlinear products uniformly.  The identity on the \(y\)-coordinates,
constant resets, and other affine coordinate updates can simultaneously be
represented exactly using ReLU pairs.  A final shared tokenwise affine readout returns
\((b_\ell+s_\ell)_{\ell=1}^{d'}\).

\smallskip\noindent\emph{Uniform error control.}\hspace{0.25em}
Only finitely many stages occur.  All exact intermediate states range over a
compact set because the probe vector ranges over the compact image of
\(\cM\).  The pointwise networks preserve \(y\) and reset \(q\) exactly, and
every later key and value ignores the approximate \((p,s)\) registers.
Consequently all later probes remain exact; only the finite sequence of
register multiplications and additions propagates approximation error.

Enumerate the finitely many pointwise update stages by
\(k=1,\ldots,N\).  The exact pre- and post-attention states at each stage have
compact ranges, so choose compact neighborhoods with positive buffer around
all of them.  Each ideal polynomial update \(G_k\) is Lipschitz on the relevant
neighborhood, say with constant \(L_k\).  If its pointwise ReLU realization
has uniform error \(\eta_k\) and the register error after stage \(k\) is
\(e_k\), then, as long as the approximate trajectory stays in these
neighborhoods,
\[
 e_k\leq L_ke_{k-1}+\eta_k,
 \qquad e_0=0.
\]
Iterating gives the explicit bound
\[
 e_N\leq
 \sum_{k=1}^N\eta_k\prod_{j=k+1}^N L_j,
\]
with an empty product equal to one.  The positive buffers and this finite sum
allow the \(\eta_k\) to be chosen small enough both to keep every approximate
state in its prescribed neighborhood and to make \(e_N<\delta\), uniformly
over \(\mathfrak X\).  The exact temporal
head calculation is identical, including at \(t=0\), so the same construction
and the same error propagation prove both
\eqref{eq:serialized-finite} and
\eqref{eq:serialized-continuous}.
\end{proof}

\subsection{Stone--Weierstrass assembly}
\label{app:stone-weierstrass-assembly}

The final assembly combines separation and polynomial density with exact
parallel probes and a ReLU-readout approximation.

\begin{proof}[Proof of Theorems~\ref{thm:discrete-universality} and
\ref{thm:continuous-universality}]
\smallskip\noindent\emph{Common causal factor.}\hspace{0.25em}
For Theorem~\ref{thm:discrete-universality}, complete the given discrete
family by its unique extension.  For
Theorem~\ref{thm:continuous-universality}, complete the given continuous
target by the canonical finite family from
Lemma~\ref{lem:canonical-completion}.  In either case,
Definition~\ref{def:compatible-family} and
Lemma~\ref{lem:target-compatibility} make the joint evaluation map
\(\overline F^\star\) continuous and give causality on the path boundary.
Consequently, Proposition~\ref{prop:prefix-law} gives a unique continuous
factor
\[
f^\star:\cM\to\R^{d'}
\]
such that
\begin{equation}
\label{eq:target-history-factor-v2}
\overline F^\star(\xi)
=
f^\star(\mu_\xi),
\qquad
\xi\in\mathfrak X.
\end{equation}

Let \(\cA\) be the unital real algebra generated on \(\cM\) by the
restrictions of \(\mathcal L_{a,c}\), with
\(a\in\R^{d+1}\), \(\norm{a}\leq1\), and \(c\in[0,1]\).
Lemma~\ref{lem:laplace-probes} says that \(\cA\) separates points of the
compact Hausdorff space \(\cM\).  The real Stone--Weierstrass theorem
therefore makes \(\cA\) uniformly dense in \(C(\cM;\R)\).

\smallskip\noindent\emph{Finite approximation and realization.}\hspace{0.25em}
Apply the scalar density result to every output coordinate with tolerance
\(\varepsilon/(2\sqrt{d'})\).  Only finitely many probes occur in the
resulting algebra elements. Hence, there exist
\((a_r,c_r)_{r=1}^H\), with
\(a_r\in\R^{d+1}\), \(\norm{a_r}\leq1\), and \(c_r\in[0,1]\), and a vector-valued polynomial
\(\Pi:\R^H\to\R^{d'}\) such that
\begin{equation}
\label{eq:SW-vector-v2}
\sup_{\mu\in\cM}
\norm{
f^\star(\mu)
-
\Pi\bigl(
\mathcal L_{a_1,c_1}(\mu),\ldots,
\mathcal L_{a_H,c_H}(\mu)
\bigr)
}
<
\frac{\varepsilon}{2}.
\end{equation}
If no probe occurs because the approximating polynomial is constant,
append the dummy probe
\[
(a_1,c_1)\eqdef(0,0)
\]
and let \(\Pi\) ignore its argument.  Thus we may assume \(H\geq1\).

Define
\[
\widetilde\Pi:E\times\R^H\to\R^{d'},
\qquad
\widetilde\Pi(y,v)\eqdef\Pi(v).
\]
Lemma~\ref{lem:parallel-realization}, applied to
\(\widetilde\Pi\) with \(\delta=\varepsilon/2\), supplies one attention
block and one pointwise ReLU network, with a shared parameter list
\(\Theta=(\theta,\eta)\), that approximate this polynomial
simultaneously on every finite state and every boundary state.

By Corollary~\ref{cor:transformer-compatible-family}, the resulting joint
transformer evaluation
\[
\overline T_\Theta:\mathfrak X\to\R^{d'}
\]
is continuous and causal.  Proposition~\ref{prop:prefix-law} therefore
gives a unique continuous history factor
\[
t_\Theta:\cM\to\R^{d'}
\]
such that
\begin{equation}
\label{eq:transformer-history-factor-v2}
\overline T_\Theta(\xi)
=
t_\Theta(\mu_\xi),
\qquad
\xi\in\mathfrak X.
\end{equation}
Equivalently,
\[
\overline T_\Theta
=
t_\Theta\circ\boldsymbol\mu.
\]

More explicitly, the construction in
Lemma~\ref{lem:parallel-realization} chooses the \(r\)-th masked-attention
head to compute \(\mathcal L_{a_r,c_r}(\mu)\) exactly and to store it in
the \(r\)-th work register, while the residual coordinates preserve the
endpoint \(\mathfrak e(\mu)\).  Thus the pointwise readout supplied by that
lemma satisfies
\begin{equation}
\label{eq:explicit-transformer-factor-v2}
t_\Theta(\mu)
=
\MLP_\eta\left(
\mathfrak e(\mu),1,
\mathcal L_{a_1,c_1}(\mu),\ldots,
\mathcal L_{a_H,c_H}(\mu)
\right),
\qquad
\mu\in\cM.
\end{equation}
This is precisely the history factor induced by the shared architecture
\[
T_\Theta
=
\MLP_\eta\circ\MAtt_\theta\circ\mathcal E_H.
\]
The map \((y,v)\mapsto\MLP_\eta(y,1,v)\) approximates
\(\widetilde\Pi(y,v)=\Pi(v)\); hence the residual coordinates retain the endpoint
variable, although the target polynomial itself ignores it.

Combining
\eqref{eq:parallel-finite-v2}--\eqref{eq:parallel-continuous-v2}
with \eqref{eq:explicit-transformer-factor-v2} gives
\begin{equation}
\label{eq:transformer-polynomial-factor-v2}
\sup_{\mu\in\cM}
\norm{
t_\Theta(\mu)
-
\Pi\bigl(
\mathcal L_{a_1,c_1}(\mu),\ldots,
\mathcal L_{a_H,c_H}(\mu)
\bigr)
}
<
\frac{\varepsilon}{2}.
\end{equation}
Consequently, the triangle inequality and
\eqref{eq:SW-vector-v2} yield
\[
\sup_{\mu\in\cM}
\norm{f^\star(\mu)-t_\Theta(\mu)}
<
\varepsilon.
\]
Using
\eqref{eq:target-history-factor-v2} and
\eqref{eq:transformer-history-factor-v2}, we obtain
\[
\begin{aligned}
\sup_{\xi\in\mathfrak X}
\norm{
\overline F^\star(\xi)-\overline T_\Theta(\xi)
}
&=
\sup_{\xi\in\mathfrak X}
\norm{
f^\star(\mu_\xi)-t_\Theta(\mu_\xi)
}\\
&=
\sup_{\mu\in\cM}
\norm{
f^\star(\mu)-t_\Theta(\mu)
}
<\varepsilon.
\end{aligned}
\]
This is the completed estimate.

In the first case, restricting this estimate to the finite states gives
\eqref{eq:discrete-universality-bound-v3}.  In the second case,
restricting it to the boundary states gives
\eqref{eq:continuous-universality-bound-v2}.  These are precisely the
two theorem claims.  Moreover, when the continuous target is the
extension of the discrete family in the first theorem, the single
construction above proves both estimates with the same parameter list
\(\Theta\).
\end{proof}

\begin{remark}[Bounded heads suffice for qualitative universality]
\label{rem:bounded-qualitative-heads}
The separating directions in the proof lie in the Euclidean unit ball.
Lemma~\ref{lem:parallel-realization} therefore gives both qualitative
theorems with all four matrices of every scalar attention head satisfying
\eqref{eq:bounded-qualitative-heads}.  The same bounds hold for the
serialized realization in Lemma~\ref{lem:serialized-realization}.
In both realizations, the query is exactly one and the key depends only on
\(y\in\Omega\times[0,1]\), so every score lies in
\([-\sqrt{2},\sqrt{2}]\).  Each attention weight on an \(i\)-token prefix
is therefore at most \(e^{2\sqrt{2}}/i\): universality does not require
increasingly concentrated attention.  The shallow head count, the
serialized depth, and all pointwise hidden widths and parameter magnitudes
remain unrestricted.  These are per-head bounds, not a bounded-weight
statement for the whole network.
\end{remark}

\subsection{The exact all-resolution approximation class}

We now prove the converse to Theorem~\ref{thm:discrete-universality},
characterizing exactly the families approximable uniformly across all
lengths. Every fixed transformer has a continuous completed realization,
and this property is preserved by uniform approximation over all finite
lengths. Continuous extendability is therefore necessary as well as sufficient.

\begin{corollary}[Uniform closure characterization]
\label{cor:uniform-closure-characterization}
Fix an admissible common modulus \(\omega\) and the dimensions \(d,d'\).
For a discrete family \(\mathbf F=(F_n)_{n\geq1}\), with
\(F_n:X_n^\omega\to(\R^{d'})^n\), define
\begin{equation}
\label{eq:all-resolution-family-norm}
\norm{\mathbf F}_{\mathrm{all}}
\eqdef
\sup_{n\geq1}\sup_{z\in X_n^\omega}\max_{i\in[n]}
\norm{F_n(z)_i}.
\end{equation}
In the space of families with finite \(\norm{\cdot}_{\mathrm{all}}\), let
\(\mathsf T_{\mathrm{sh}}\) denote the families induced by the shallow
architecture \eqref{eq:shallow-transformer-v2}, let
\(\mathsf T_{\mathrm{fin}}\) denote those induced by arbitrary fixed finite
compositions of masked attention, shared tokenwise affine maps, and
pointwise ReLU networks, and let \(\mathsf C_\omega\) denote the
continuously extendable causal families of
Definition~\ref{def:compatible-family}.  Both transformer classes use the
canonical positional lift \(\phi_n\) and parameters independent of \(n\).
Each approximant has one finite architecture and parameter list fixed across
all lengths.  Widths and parameters may vary with the requested accuracy,
as may depth in \(\mathsf T_{\mathrm{fin}}\).
Then
\begin{equation}
\label{eq:uniform-closure-characterization}
\overline{\mathsf T_{\mathrm{sh}}}^{\norm{\cdot}_{\mathrm{all}}}
=
\overline{\mathsf T_{\mathrm{fin}}}^{\norm{\cdot}_{\mathrm{all}}}
=\mathsf C_\omega.
\end{equation}
The first closure is unchanged if every scalar head is required to satisfy
\eqref{eq:bounded-qualitative-heads}.
Thus increasing depth does not enlarge the uniform approximation closure;
this is not a claim about approximation efficiency.
\end{corollary}

\begin{proof}
\smallskip\noindent\emph{Sufficiency: shallow approximation of extendable families.}\hspace{0.25em}
By Lemma~\ref{lem:target-compatibility}, every family in
\(\mathsf C_\omega\) has a continuous evaluation map on compact
\(\mathfrak X\), so its all-resolution norm is finite.
Theorem~\ref{thm:discrete-universality} places it in the uniform closure of
\(\mathsf T_{\mathrm{sh}}\), even with the per-head bounds by
Remark~\ref{rem:bounded-qualitative-heads}.  Also,
\(\mathsf T_{\mathrm{sh}}\subset\mathsf T_{\mathrm{fin}}\).

\smallskip\noindent\emph{Necessity: uniform limits retain a continuous extension.}\hspace{0.25em}
Suppose \(\mathbf T^{(k)}\in\mathsf T_{\mathrm{fin}}\) converges to
\(\mathbf F\) in \(\norm{\cdot}_{\mathrm{all}}\).
Corollary~\ref{cor:transformer-compatible-family} gives continuous completed
evaluation maps \(\overline T^{(k)}:\mathfrak X\to\R^{d'}\).
Finite states are dense by Lemma~\ref{lem:joint-compactification}, so
continuity gives the same supremum for the difference norm on finite states
and on all of \(\mathfrak X\):
\[
\sup_{\xi\in\mathfrak X}
\norm{\overline T^{(k)}(\xi)-\overline T^{(\ell)}(\xi)}
=\norm{\mathbf T^{(k)}-\mathbf T^{(\ell)}}_{\mathrm{all}}.
\]
Indeed, each boundary state is a limit of finite states, so its output
difference is a limit of finite-state differences and cannot exceed their
supremum. Controlling all finite lengths therefore also controls the path limits.
Thus the completed maps are uniformly Cauchy and converge uniformly to a
continuous map \(\overline F:\mathfrak X\to\R^{d'}\), whose finite
restriction is the given family \(\mathbf F\).
Define
\[
F_\infty(x)(t)\eqdef\overline F(\infty,x,t).
\]
This is a continuous path for each \(x\).  Moreover, uniform continuity of
\(\overline F\) on compact \(\mathfrak X\), together with
\[
d_{\mathfrak X}\bigl((\infty,x,t),(\infty,x',t)\bigr)
=\norm{x-x'}_\infty,
\]
shows that \(F_\infty:X_\infty\to C([0,1];\R^{d'})\) is continuous in
the uniform norm: the same continuity estimate holds at every \(t\).
At each fixed length, continuity and prefix causality of \(F_n\) follow
from uniform convergence of the transformer maps.
Lemma~\ref{lem:target-compatibility} now applies to the continuous joint
evaluation \(\overline F\), and gives the uniform extension condition,
uniqueness of the extension, and boundary causality.  Therefore
\(\mathbf F\in\mathsf C_\omega\), proving all the asserted equalities.
\end{proof}

For example, when \(d'=1\), the family \(F_n(z)_i\eqdef(-1)^n\) is
continuous and prefix-causal at each fixed length, but lies outside this
closure.  On the zero input, any fixed transformer's terminal outputs satisfy
\[
b_n\eqdef T_\Theta\bigl(\phi_n(0,\ldots,0)\bigr)_n\longrightarrow a
\quad\text{for some }a\in\R
\]
under grid refinement.  Its uniform error against this family is therefore
at least
\(\max\{|a-1|,|a+1|\}\geq1\), by considering even and odd lengths.
The zero transformer attains error \(1\).
The obstruction is cross-resolution discontinuity, not an inability to
distinguish a particular finite pair of lengths.

\section{Controlled ODEs: continuous extension and target regularity}
\label{app:control-odes}

This appendix proves Examples~\ref{ex:control-ode-family-main}
and~\ref{ex:regular-control-integrator}. A controlled
ODE turns an input trajectory into a state trajectory; reconstructing the
input tokens and sampling the exact solution yields a discrete causal
family at every resolution. The key estimates are uniform over the input
class and include the initial interpolation cell. We then identify a
right-endpoint discretization of additive dynamics that is \(\beta\)-smooth
whenever the additive vector field is \(C^\beta\).

\subsection{Construction and uniform compatibility}
\label{app:control-ode-extension}

We first specify the dynamics and verify all parts of
Definition~\ref{def:compatible-family}, without assuming differentiability
of the input paths or Lipschitz dependence on the control.

\paragraph{Dynamics driven by a reconstructed control.}
Fix an admissible common modulus \(\omega\), the classes
\(X_n\eqdef X_n^\omega\), \(X_\infty\eqdef X_\infty^\omega\), and an initial state
\(y_0\in\R^{d'}\). Let \(b:\R^{d'}\times\Omega\to\R^{d'}\) be continuous
and satisfy, for some \(L_b\geq0\),
\begin{equation}
\label{eq:control-ode-uniform-lipschitz}
 \norm{b(y,u)-b(y',u)}\leq L_b\norm{y-y'}
 \qquad(y,y'\in\R^{d'},\ u\in\Omega).
\end{equation}
For \(x\in X_\infty\), define \(y_x\) by
\begin{equation}
\label{eq:continuous-control-ode}
 \dot y_x(t)=b(y_x(t),x(t)),\qquad y_x(0)=y_0,
 \qquad t\in[0,1].
\end{equation}
The global state-Lipschitz bound gives a unique
\(y_x\in C^1([0,1];\R^{d'})\). Using the reconstruction \(\In_n\) and
sampling \(\mathcal S_n\) from Section~\ref{sec:targets-v2}, set
\begin{equation}
\label{eq:discrete-control-solution-operator}
 F_\infty^\star(x)\eqdef y_x,\qquad
 F_n^\star(z)\eqdef\mathcal S_n y_{\In_nz},
 \quad\text{that is,}\quad
 F_n^\star(z)_i=y_{\In_nz}(i/n).
\end{equation}
Reconstruction stays in \(X_\infty\) by
\eqref{eq:modulus-reconstruction-v2}. These maps sample exact solutions;
no numerical time-stepping error is included in their definition.

\paragraph{Uniform bounds imply extendability.}
Compactness of the control range bounds the vector field at the origin,
and its state-Lipschitz constant then controls every solution and its speed.
Define
\begin{equation}
\label{eq:control-ode-bound-constants}
 C_0\eqdef\max_{u\in\Omega}\norm{b(0,u)},\qquad
 R\eqdef(\norm{y_0}+C_0)e^{L_b},\qquad M\eqdef C_0+L_bR.
\end{equation}

\begin{proposition}[Controlled ODE compatibility]
\label{prop:control-ode-compatible-family}
Under \eqref{eq:control-ode-uniform-lipschitz}, the maps
\eqref{eq:discrete-control-solution-operator} form a continuously extendable
causal family with continuous path extension \(F_\infty^\star\).
For every \(x\in X_\infty\),
\(\norm{y_x}_\infty\leq R\) and \(\norm{\dot y_x}_\infty\leq M\), and
for every \(n\geq1\),
\begin{equation}
\label{eq:control-ode-extension-rate}
 \sup_{z\in X_n}
 \norm{\In_n(F_n^\star(z))-F_\infty^\star(\In_nz)}_\infty
 \leq \frac{M}{n}.
\end{equation}
The constants \(R,M\) are independent of the input, resolution, and
admissible common modulus.
\end{proposition}

\begin{proof}
\smallskip\noindent\emph{Existence and uniform trajectory bounds.}\hspace{0.25em}
Continuity of \(b\) and compactness of \(\Omega\) give \(C_0<\infty\).
The state-Lipschitz bound implies
\(\norm{b(y,u)}\leq C_0+L_b\norm{y}\), so local solutions cannot blow up
on the fixed horizon. The integral equation and Gronwall's inequality give
\[
 \norm{y_x(t)}
 \leq(\norm{y_0}+C_0t)e^{L_bt}\leq R,
 \qquad
 \norm{\dot y_x(t)}\leq C_0+L_bR=M.
\]
Thus every solution exists uniquely on \([0,1]\) and is \(M\)-Lipschitz
in time.

\smallskip\noindent\emph{Continuity in the input control.}\hspace{0.25em}
On the compact set \(\overline B_{\R^{d'}}(0,R)\times\Omega\), define
the uniform control modulus
\begin{equation}
\label{eq:control-ode-control-modulus}
 \rho_b(r)\eqdef
 \sup_{\substack{\norm{y}\leq R,\ u,v\in\Omega\\\norm{u-v}\leq r}}
 \norm{b(y,u)-b(y,v)},\qquad r\geq0.
\end{equation}
Uniform continuity gives \(\rho_b(r)\to0\) as \(r\downarrow0\).
For \(x,x'\in X_\infty\), subtraction of their integral equations yields
\[
 \norm{y_x(t)-y_{x'}(t)}
 \leq L_b\int_0^t\norm{y_x(s)-y_{x'}(s)}\,\dd s
       +t\,\rho_b(\norm{x-x'}_\infty).
\]
Since \(t\leq1\), Gronwall gives
\begin{equation}
\label{eq:control-ode-continuous-dependence}
 \norm{F_\infty^\star(x)-F_\infty^\star(x')}_\infty
 \leq e^{L_b}\rho_b(\norm{x-x'}_\infty).
\end{equation}
Hence \(F_\infty^\star\) is continuous in the uniform norm. For each
fixed \(n\), \(\In_n\) is continuous (indeed, nonexpansive in the maximum
token norm), and sampling is continuous, so \(F_n^\star\) is continuous.

\smallskip\noindent\emph{Prefix causality.}\hspace{0.25em}
If \(z_{1:i}=z'_{1:i}\), then
\(\In_nz=\In_nz'\) on \([0,i/n]\): the initial cell uses only \(z_1\),
and each subsequent cell up to time \(i/n\) uses two tokens from this
prefix. Uniqueness for the ODE restricted to that interval gives
\(F_n^\star(z)_i=F_n^\star(z')_i\). Similarly, two continuous controls
agreeing on \([0,t]\) have identical state trajectories there, so
\(F_\infty^\star\) is causal. Although affine reconstruction between
grid points uses the next grid value, no token after index \(i\) enters the
output at time \(i/n\).

\smallskip\noindent\emph{Uniform extension at the initial and interior cells.}\hspace{0.25em}
For \(z\in X_n\), put \(y\eqdef y_{\In_nz}\). Then
\(\In_n(F_n^\star(z))=\In_n\mathcal S_ny\).
On an interior grid cell, its value is a convex combination of the two
neighboring sampled values of \(y\), each within \(M/n\) of \(y(t)\).
On the initial cell \([0,1/n]\), it equals \(y(1/n)\), and
\[
 \norm{y(1/n)-y(t)}\leq M(1/n-t)\leq M/n.
\]
In particular, the interpolated output at zero need not equal \(y_0\),
but its error has the same bound. Taking the supremum over time and input
proves \eqref{eq:control-ode-extension-rate} and verifies the extension
condition in Definition~\ref{def:compatible-family}.
\end{proof}

\subsection{Refinement, an explicit instance, and scope}

The estimates above also describe what happens when one continuous control
is sampled more finely, and clarify the distinction between this
qualitative example and the quantitative approximation theorem.

\paragraph{Refining samples of a fixed control.}
For \(x\in X_\infty\), combine
\eqref{eq:control-ode-extension-rate} and
\eqref{eq:control-ode-continuous-dependence} with
\(\norm{\In_n\mathcal S_nx-x}_\infty\leq\omega(1/n)\). The triangle
inequality gives
\begin{equation}
\label{eq:control-ode-sampled-refinement}
 \sup_{x\in X_\infty}
 \norm{\In_nF_n^\star(\mathcal S_nx)-F_\infty^\star(x)}_\infty
 \leq \frac{M}{n}+e^{L_b}\rho_b(\omega(1/n))
 \longrightarrow0.
\end{equation}
If \(b\) is also \(L_u\)-Lipschitz in the control on the bounded state
region, the second term is at most \(e^{L_b}L_u\omega(1/n)\).
No exact equality of targets at different finite resolutions is required:
the reconstruction and the normalized sampling times both depend on \(n\).

\paragraph{The controlled integrator.}
For \(d'=d\), \(y_0=0\), and \(b(y,u)=u\), the path operator is
integration. The initial constant cell followed by affine cells gives
\begin{equation}
\label{eq:controlled-integrator-exact}
 F_\infty^\star(x)(t)=\int_0^t x(s)\,\dd s,\qquad
 F_n^\star(z)_i=
 \frac1n\sum_{j=1}^i z_j+\frac{z_1-z_i}{2n}.
\end{equation}
Indeed, the integral over \([0,1/n]\) is \(z_1/n\); each later cell
contributes \((z_{j-1}+z_j)/(2n)\). Summing gives the formula, also for
\(i=1\), where the correction vanishes. Thus exact integration of the
reconstructed control differs from the plain right-endpoint cumulative sum,
but their grid outputs differ by at most \(1/n\), since
\(\norm{z_1-z_i}\leq2\). Nonexpansiveness of output interpolation then gives
the same continuous path limit for both families. Here \(M=1\), so
the compatibility defect is at most \(1/n\).

\paragraph{What the example does and does not imply.}
Theorems~\ref{thm:discrete-universality} and~\ref{thm:path-universality-main}
therefore approximate this family, uniformly at all finite resolutions and
on continuous controls, with the same transformer parameters.
The bound \(M/n\) concerns time reconstruction of the exact output, not
approximation as a function of the network budget \(p\). Applying
Theorem~\ref{thm:quantitative-universality-main} additionally requires
\(R_{\beta;\omega}(\mathbf F^\star)<\infty\); continuity of the vector field
in its control argument alone does not ensure that condition. The present
example establishes the qualitative target hypothesis. The next subsection
verifies the quantitative hypothesis for a right-endpoint variant and shows
why the discretization matters.

\subsection{\texorpdfstring{\(\beta\)}{Beta}-smooth controlled integrators}
\label{app:regular-control-integrator}

Additive dynamics provide an explicit bridge from the ODE example to the
\(\beta\)-smooth teacher class in Definition~\ref{def:attention-teacher-class}. The key is to
choose grid targets that are smooth averages of their token--position laws.

Fix \(\beta>0\) and \(g\in C^\beta(\Omega;\R^{d'})\), using the
componentwise isotropic restriction norm from \([-1,1]^d\), with the
integer-order convention \(C^m=C^{m-1,1}\). Define
\begin{equation}
\label{eq:regular-integrator-family}
 F_\infty^\star(x)(t)\eqdef y_0+\int_0^t g(x(s))\,\dd s,
 \qquad
 F_n^\star(z)_i\eqdef y_0+\frac1n\sum_{j=1}^i g(z_j).
\end{equation}
These finite outputs solve \(\dot y=g(\Cn_nz)\) exactly at grid points,
where \(\Cn_nz\) is the right-endpoint step reconstruction from
\eqref{eq:generic-step-v2}. They are right-endpoint quadrature values for
the continuous ODE, not its exact solution driven by the affine
reconstruction \(\In_nz\).

\begin{proposition}[A \(\beta\)-smooth ODE teacher]
\label{prop:regular-control-integrator}
The family \eqref{eq:regular-integrator-family} is continuously extendable
for every admissible \(\omega\). Its completed history factor is
\begin{equation}
\label{eq:regular-integrator-factor}
 f^\star(\mu)=y_0+t(\mu)\int_E g(u)\,\dd\mu(u,s),
 \qquad t(\mu)\eqdef\mathfrak e(\mu)_{d+1},
\end{equation}
and
\begin{equation}
\label{eq:regular-integrator-seminorm}
 R_{\beta;\omega}(\mathbf F^\star)
 \leq C_{d,d',\beta}\norm{g}_{C^\beta(\Omega)}.
\end{equation}
The constant is independent of \(\omega\) and \(y_0\). The optimal seminorm
may vary with \(\omega\), since its supremum is taken on \(\cM_\omega\);
the displayed upper bound is uniform over these domains. Centering the
outputs by \(y_0\) and scaling places the family in
\(\mathcal F^\beta(\omega)\).
\end{proposition}

\begin{proof}
Put \(G\eqdef\norm{g}_\infty\) in the Euclidean output norm, and let
\(\rho_g\) be its uniform modulus of continuity on \(\Omega\).
Continuity and causality of each grid map are immediate. The continuous
operator is causal and satisfies
\(\norm{F_\infty^\star(x)-F_\infty^\star(x')}_\infty
\leq\rho_g(\norm{x-x'}_\infty)\).
On each cell, \(\norm{\Cn_nz-\In_nz}\leq\omega(1/n)\), so integration
gives
\[
 \max_{i\leq n}\norm{F_n^\star(z)_i-
 F_\infty^\star(\In_nz)(i/n)}
 \leq\rho_g(\omega(1/n)).
\]
The continuous output is \(G\)-Lipschitz in time. Nonexpansiveness of
affine interpolation and its initial-cell estimate therefore imply
\begin{equation}
\label{eq:regular-integrator-compatibility}
 \sup_{z\in X_n}\norm{\In_nF_n^\star(z)-F_\infty^\star(\In_nz)}_\infty
 \leq G/n+\rho_g(\omega(1/n))\longrightarrow0.
\end{equation}
This proves continuous extendability. Formula
\eqref{eq:regular-integrator-factor} follows from \(t=i/n\) and the
normalized prefix law at finite resolution; it also holds for every
continuous history, including \(t=0\).

Write \(m_\mu\eqdef\int g\,\dd\mu\), \(t\eqdef t(\mu)\), and \(t'\eqdef t(\nu)\).
The factorization gives
\[
 \norm{f^\star(\mu)-f^\star(\nu)}
 \leq\norm{m_\mu-m_\nu}+G\abs{t-t'}.
\]
For \(0<\beta\leq1\), an optimal \(W_1\) coupling and Jensen's inequality
give
\(\norm{m_\mu-m_\nu}\leq
C_{d,d'}\norm{g}_{C^\beta}W_1(\mu,\nu)^\beta\).
Also \(\abs{t-t'}\leq\min\{1,\delta_e(\mu,\nu)\}
\leq\delta_e(\mu,\nu)^\beta\).
Since \(W_1\asymp d_1\) by Lemma~\ref{lem:unified-history-metric},
these bounds give \eqref{eq:regular-integrator-seminorm} with
\(\Delta_\beta=d_1^\beta+\delta_e^\beta\).
For \(\beta>1\), the functions \((u,s)\mapsto g_j(u)\) are admissible
smooth tests after scaling, so
\(\norm{m_\mu-m_\nu}\leq
C_{d,d',\beta}\norm{g}_{C^\beta}d_\beta(\mu,\nu)\).
Together with \(\abs{t-t'}\leq\delta_e(\mu,\nu)\), this proves the
high-order bound. Finally, \(\norm{f^\star-y_0}_\infty\leq G\), and
translation does not change the seminorm. Dividing the centered family by
\(\max\{1,G,R_{\beta;\omega}(\mathbf F^\star)\}\) gives the normalized
teacher class.
\end{proof}

\paragraph{Why the discretization matters.}
Smooth dynamics alone do not imply the \(\beta\)-smooth target condition. Even the
exact affine-control integrator \eqref{eq:controlled-integrator-exact}
fails it for every \(\beta>1\) when \(d=d'=1\) and \(\omega\not\equiv0\).
To see this, take \(n\geq3\), \(h\eqdef1/n\), and
\(a\eqdef\min\{\omega(h)/2,h\}>0\). Compare the tokens
\(z\eqdef(a,-a,0,\ldots,0)\) with the zero sequence at \(i=n\).
Both belong to \(X_n^\omega\) and have the same endpoint. Their exact
affine-control outputs differ by \(ah/2\). For
\(\gamma\eqdef\min\{\beta-1,1\}\) and any unit \(C^\beta\) test,
\[
 \int\varphi\,\dd(\mu_z-\mu_0)
 =ah\int_0^1
 \bigl[\partial_u\varphi(ra,h)-\partial_u\varphi(-ra,2h)\bigr]\,\dd r,
 \qquad
 d_\beta(\mu_z,\mu_0)\leq C_\beta ah^{1+\gamma}.
\]
The last estimate follows from \(a\leq h\) and the
\(\gamma\)-H\"older bound on \(\partial_u\varphi\), using extensions
arbitrarily close to the restriction norm. The regularity quotient is
therefore at least \(c_\beta h^{-\gamma}\), which diverges. In contrast,
the right-endpoint outputs in \eqref{eq:regular-integrator-family} agree
for this pair. Thus the two discretizations have the same path limit, but
only the right-endpoint family is \(\beta\)-smooth for every \(\beta>1\).

\section{Continuous-time transformers and path-space universality}
\label{app:continuous-time-universality}

This appendix defines the continuous-time realization of masked attention,
states its path-space universality theorem, and proves that every fixed finite
causal transformer converges to this realization under grid refinement.
Throughout this appendix, write \(X_n\eqdef X_n^\omega\) and
\(X_\infty\eqdef X_\infty^\omega\) for the fixed common modulus \(\omega\).

\subsection{Temporal attention and the continuous-time realization}
\label{sec:continuous-extension-v3}
\label{sec:temporal-attention-v3}

Whenever hidden-state interpolants converge uniformly along a grid
refinement, each masked softmax average converges to a normalized temporal
integral over the observed history.  Replacing finite sums by these integrals
defines the continuous-time realization with the same learned parameters.

For a continuous path \(x\in X_\infty\), define the continuous positional lift
\begin{equation}
\label{eq:continuous-lift-v3}
\phi_\infty(x)(t)\eqdef(x(t),t).
\end{equation}
For a continuous hidden path \(u:[0,1]\to\R^p\), the temporal attention block
corresponding to \eqref{eq:finite-attention-v2} is, for \(t>0\),
\begin{equation}
\label{eq:temporal-attention-v2}
\MAtt_\theta^\infty(u)(t)
\eqdef u(t)+
\sum_{r=1}^{H}W_r
\frac{
\int_0^t e^{\ip{Q_ru(t)}{K_ru(s)}}V_ru(s)\,\dd s
}{
\int_0^t e^{\ip{Q_ru(t)}{K_ru(s)}}\,\dd s
},
\end{equation}
with the continuous endpoint value
\begin{equation}
\label{eq:temporal-attention-zero-v2}
\MAtt_\theta^\infty(u)(0)
\eqdef u(0)+\sum_{r=1}^{H}W_rV_ru(0).
\end{equation}
Every denominator in \eqref{eq:temporal-attention-v2} is strictly positive.
For continuous \(u\), each normalized head value tends to \(V_ru(0)\) as
\(t\downarrow0\), so \eqref{eq:temporal-attention-zero-v2} is the unique
continuous endpoint value.  Replacing each discrete attention block by
\eqref{eq:temporal-attention-v2}--\eqref{eq:temporal-attention-zero-v2}, while
retaining the same fixed lifts, affine maps, and pointwise MLPs, defines the
continuous-time transformer \(T_{\Theta,\infty}\).

For a continuous path \(y:[0,1]\to\R^{d+1}\), the lift and readout act
pointwise: \((\mathcal E_Hy)(t)\eqdef(y(t),1,0_H)\) and
\((\MLP_\eta u)(t)\eqdef\MLP_\eta(u(t))\).  The continuous-time realization of
\eqref{eq:shallow-transformer-v2} is therefore
\begin{equation}
\label{eq:shallow-temporal-transformer-v2}
T_{\Theta,\infty}
\eqdef
\MLP_\eta\circ\MAtt_\theta^\infty\circ\mathcal E_H,
\end{equation}
with the scalar-head and MLP dimensions specified in
\eqref{eq:theorem-head-dimensions-v2} and
\eqref{eq:pointwise-mlp-dimensions-v2}.

This graph-based causal operator integrates the ordered trajectory
\(s\mapsto(x(s),s)\) only over \([0,t]\).  Unlike content-only unmasked
mean-field descriptions, it retains temporal order.  Standard unmasked
attention averages over a full token law even when normalized position is
included; temporal attention instead acts on the restricted and renormalized
prefix law at each \(t\).

\subsection{Path-space universality}
\label{sec:continuous-universality}

Temporal attention is universal on every compact path class controlled by an
admissible common modulus. We restate the path-only guarantee of
Theorem~\ref{thm:path-universality-main} with all dimensions explicit;
no finite-grid target appears in its assumptions or conclusion.

\begin{theorem}[Continuous-time universality for causal paths]
\label{thm:continuous-universality}
Fix an admissible common modulus \(\omega\), with
\(X_\infty^\omega\) given by \eqref{eq:Xinfty-v2}.  Let
\(F_\infty^\star:X_\infty\to C([0,1];\R^{d'})\) be continuous in the
uniform norm and causal in the sense that
\begin{equation}
\label{eq:path-causality-v2}
x|_{[0,t]}=x'|_{[0,t]}
\quad\Longrightarrow\quad
F_\infty^\star(x)(t)=F_\infty^\star(x')(t).
\end{equation}
For every \(\varepsilon>0\), there exist integers \(H,M\geq1\) and a
parameter list \(\Theta=(\theta,\eta)\), with residual width
\(p_H=d+H+2\), scalar-head dimensions
\(Q_r,K_r,V_r\in\R^{1\times p_H}\),
\(W_r\in\R^{p_H\times1}\), and readout dimensions
\(A_1\in\R^{M\times p_H}\), \(b_1\in\R^M\),
\(A_2\in\R^{d'\times M}\), \(b_2\in\R^{d'}\).
Every attention matrix \(Q_r,K_r,V_r,W_r\) may be chosen to have operator
norm at most one for the Euclidean norms.  Using the lift and readout
defined above with temporal attention, the transformer
\(T_{\Theta,\infty}\) is given by
\eqref{eq:shallow-temporal-transformer-v2}, and
\begin{equation}
\label{eq:continuous-universality-bound-v2}
\sup_{x\in X_\infty}
\norm{T_{\Theta,\infty}(\phi_\infty(x))-F_\infty^\star(x)}_\infty
<\varepsilon.
\end{equation}
\end{theorem}

The widths and readout parameter magnitudes may depend on \(\omega\), the
target, and \(\varepsilon\); the per-head norm bounds are uniform in these
choices.  Normalized time presupposes a fixed terminal horizon.
The completed-space proof in Appendix~\ref{app:proof-joint} treats this theorem
and Theorem~\ref{thm:discrete-universality} simultaneously.  When
\(F_\infty^\star\) extends a discrete family, the discrete and continuous
estimates are restrictions of one approximant.

\subsection{Uniform closure under grid refinement}

The common proof constructs discrete and temporal attention simultaneously.
The following stronger architectural statement records that every fixed
finite-depth discrete transformer converges uniformly to its temporal
realization.

\begin{proposition}[Temporal closure of a fixed causal transformer]
\label{prop:temporal-closure}
Fix a parameter list \(\Theta\), independent of sequence length, defining a
finite operator \(T_\Theta\) by a composition of masked-attention blocks,
shared tokenwise affine maps, and pointwise ReLU networks, and let
\(T_{\Theta,\infty}\) be its temporal realization.  If \(n_k\to\infty\),
\(z^k\in X_{n_k}\), and
\(\In_{n_k}z^k\to x\in X_\infty\) uniformly, then
\begin{equation}
\label{eq:temporal-closure-v3}
\In_{n_k}
T_\Theta(\phi_{n_k}(z^k))
\longrightarrow
T_{\Theta,\infty}(\phi_\infty(x))
\quad\text{uniformly on }[0,1].
\end{equation}
\end{proposition}

\begin{proof}
\smallskip\noindent\emph{Step-interpolant convergence.}\hspace{0.25em}
For \(u=(u_1,\ldots,u_n)\), let \(\Cn_nu\) denote the right-endpoint step
interpolation defined in \eqref{eq:generic-step-v2}.  Under the hypotheses
of the proposition, Corollary~\ref{cor:transformer-closure} gives
\begin{equation}
\label{eq:temporal-closure-step-v3}
\norm{
\Cn_{n_k}T_\Theta(\phi_{n_k}(z^k))
-T_{\Theta,\infty}(\phi_\infty(x))
}_\infty
\longrightarrow0.
\end{equation}
\smallskip\noindent\emph{Affine-interpolant upgrade.}\hspace{0.25em}
It remains to replace \(\Cn_n\) by the affine
interpolant \(\In_n\) used in the proposition.  More generally, if
\(u^k\in(\R^m)^{n_k}\), \(u\in C([0,1];\R^m)\), and
\(\Cn_{n_k}u^k\to u\) uniformly, then
\begin{equation}
\label{eq:step-to-affine-v3}
\norm{\In_{n_k}u^k-u}_\infty
\leq
\norm{\Cn_{n_k}u^k-u}_\infty
+\omega_u(1/n_k),
\end{equation}
where
\(\omega_u(r)\eqdef\sup_{\abs{s-t}\leq r}\norm{u(s)-u(t)}\).
Indeed, on each grid cell the affine interpolant is a convex combination of
the two adjacent token values.  If
\(e_k\eqdef\norm{\Cn_{n_k}u^k-u}_\infty\), the left and right token values are
within \(e_k\) of \(u\) at their respective grid points.  Their affine
combination is therefore within \(e_k\) of the corresponding chord of \(u\),
and both chord endpoints are within \(\omega_u(1/n_k)\) of \(u(t)\) at any
time \(t\) in that cell.  The same estimate holds on the initial cell, where
\(\In_{n_k}u^k\) is constant.  The limit path in
\eqref{eq:temporal-closure-step-v3} is continuous, so
\eqref{eq:step-to-affine-v3} proves the claim.
\end{proof}

\section{Proof of the generalization bound}
\label{app:proof-generalization}

Theorem~\ref{thm:generalization-main} combines bounded-weight approximation,
a length-independent parameter cover, and a squared-loss concentration
inequality. The statistical observation is an entire sequence: independence
is used only between the \(N\) training sequences. The architecture-specific
step is the cover; the statistical step is a classical Bernstein argument.
We then justify raw squared-loss training followed by clipping. Clipping
decreases the empirical loss, but need not turn a raw empirical minimizer
into a minimizer of the clipped loss.

\subsection{History representation and the approximation comparator}

Writing the network on the common history space makes one approximation
comparator and one parameter cover available at every resolution.
Fix the admissible modulus \(\omega\) and use the compact history space
\(\cM_\omega\), endpoint map \(\mathfrak e\), and target factor
\(f^\star\) from Appendix~\ref{app:quantitative-rates}.  Write
\(E=\Omega\times[0,1]\) and recall that the history at position \(i\) is
\[
 \mu_{n,z,i}=\frac1i\sum_{j=1}^i\delta_{(z_j,j/n)},
 \qquad
 f^\star(\mu_{n,z,i})=F_n^\star(z)_i.
\]
For a fixed width pair \((H,M)\), set \(r\eqdef d+H+2\) and
\(u(y)\eqdef(y,1,0_H)\in\R^r\).  The raw transformer acts on a history as
\begin{align}
\label{eq:gen-history-attention}
 v_\theta(\mu)
 &\eqdef u(y)+\sum_{h=1}^H W_h
 \frac{\int_E e^{(Q_hu(y))(K_hu(y'))}V_hu(y')\,\dd\mu(y')}
 {\int_E e^{(Q_hu(y))(K_hu(y'))}\,\dd\mu(y')},
 \qquad y=\mathfrak e(\mu),\\
\label{eq:gen-history-readout}
 t_\Theta(\mu)
 &\eqdef A_2\rho(A_1v_\theta(\mu)+b_1)+b_2.
\end{align}
At \(\mu_{n,z,i}\), this equals \(T_\Theta(\phi_n(z))_i\);
the normalization by \(1/i\) cancels between numerator and denominator.
The same formulas give the temporal realization on continuous-path
histories.  All parameters here have the dimensions specified in
Section~\ref{sec:architecture-v2}.

Let \(\operatorname{clip}\) be Euclidean projection onto the closed unit ball of
\(\R^{d'}\).  For the raw class \(\mathfrak T_p\) in
Theorem~\ref{thm:generalization-main}, introduce its clipped history class
\begin{equation}
\label{eq:gen-clipped-class}
 \mathcal G_p
 \eqdef\{\operatorname{clip}\circ t_\Theta:T_\Theta\in\mathfrak T_p\},
 \qquad
 \norm{g-h}_{\infty,2}
 \eqdef\sup_{\mu\in\cM_\omega}\norm{g(\mu)-h(\mu)}.
\end{equation}
Clipping is a fixed, parameter-free output operation, not an additional
learned layer of the shallow architecture.  It will be used for statistical
analysis and for the final prediction, not in the training objective of
the main theorem.

Set
\begin{equation}
\label{eq:gen-approximation-scale}
 q_\beta(p)\eqdef\left(\frac{\log\log p}{\log p}\right)^{\beta/(d+2)}.
\end{equation}
The bounds \(R_{\beta;\omega}(\mathbf F^\star)\leq1\) and
\(\norm{f^\star}_\infty\leq1\) allow us to apply
Corollary~\ref{cor:quantitative-finite-precision}.  Consequently, for
every sufficiently large integer \(p\), there is a deterministic
\(T_p^\circ\in\mathfrak T_p\), with history realization \(t_p^\circ\),
such that
\begin{equation}
\label{eq:gen-approximation-comparator}
 \norm{t_p^\circ-f^\star}_{\infty,2}
 \leq C_{\mathrm{app}}q_\beta(p).
\end{equation}
Here and below, the threshold on \(p\) and \(C_{\mathrm{app}}\) depend
only on \((d,d',\beta)\).  The corollary controls the error on both
discrete and continuous histories and gives exactly the required attention
entry bound \(1\) and readout entry bound \(p+1\).  Its precision claim
is not a restriction on the trained weights: its role here is to supply
one admissible approximation comparator. The comparator depends on the
target and budget, but not on the training sample. Since \(\operatorname{clip}\) is nonexpansive
and fixes \(f^\star(\mu)\), it also gives
\[
 \inf_{g\in\mathcal G_p}\norm{g-f^\star}_{\infty,2}
 \leq C_{\mathrm{app}}q_\beta(p).
\]

\subsection{Uniform stability and covering numbers}

The key architectural property is that each attention head is a normalized
average.  Its sensitivity depends on the parameter bounds, not on the
number of tokens or the duration of a continuous prefix.

\begin{lemma}[Stability of normalized exponential averages]
\label{lem:gen-softmax-stability}
For a probability measure \(\mu\) and bounded real scores \(a,b\), put
\(w_a\eqdef e^a/\int e^a\,\dd\mu\).  Then
\begin{equation}
\label{eq:gen-softmax-stability}
 \int\abs{w_a-w_b}\,\dd\mu\leq2\norm{a-b}_\infty.
\end{equation}
For bounded scalar values \(v,v'\), this implies
\[
 \left|\int vw_a\,\dd\mu-\int v'w_b\,\dd\mu\right|
 \leq\norm{v-v'}_\infty
       +2\norm{v'}_\infty\norm{a-b}_\infty.
\]
\end{lemma}

\begin{proof}
For \(s\in[0,1]\), interpolate \(a_s\eqdef b+s(a-b)\).  Differentiating
the normalized density gives
\[
 \partial_s w_{a_s}
 =w_{a_s}\left(a-b-\int(a-b)w_{a_s}\,\dd\mu\right).
\]
Boundedness justifies differentiation under the integral.  Since
\(\int w_{a_s}\,\dd\mu=1\), the \(L^1(\mu)\) norm of this derivative
is at most \(2\norm{a-b}_\infty\).  Integrating over \(s\) proves
\eqref{eq:gen-softmax-stability}.  Add and subtract
\(\int v'w_a\,\dd\mu\) for the second assertion.  The proof applies
to empirical measures of every size and to normalized temporal measures.
\end{proof}

\begin{lemma}[A length-independent parameter cover]
\label{lem:gen-parameter-cover}
For every \(0<\epsilon\leq1\), the class \(\mathcal G_p\) admits an
\(\epsilon\)-net in \(\norm{\cdot}_{\infty,2}\), with centers in
\(\mathcal G_p\), whose cardinality \(K_p(\epsilon)\) satisfies
\begin{equation}
\label{eq:gen-covering-number}
 \log K_p(\epsilon)
 \leq2\log p+
 p\log\left(1+\frac{160\sqrt{d'}p^9}{\epsilon}\right).
\end{equation}
For each fixed width pair, the parameter-to-predictor map is
\(40\sqrt{d'}p^8\)-Lipschitz from the parameter maximum norm to
\(\norm{\cdot}_{\infty,2}\), both before and after clipping.
\end{lemma}

\begin{proof}
Fix an admissible pair \((H,M)\).  The dense parameter count
\eqref{eq:exact-dense-parameter-count-main} implies \(H,r,M\leq p\).
Compare two parameter lists in the specified boxes, at coordinatewise
distance at most \(h\).  Since \(\norm{u(y)}_\infty\leq1\), each
scalar query, key, or value has magnitude at most \(r\) and changes by
at most \(rh\).  A query--key product therefore changes by at most
\(2r^2h\).  Lemma~\ref{lem:gen-softmax-stability} bounds the change
in a scalar head average by
\[
 rh+2r(2r^2h)=(r+4r^3)h.
\]
The head average itself has magnitude at most \(r\).  Each coordinate
of \(W_h\) has magnitude at most one and changes by at most \(h\).
Perturbing both factors, then summing over heads, yields
\begin{equation}
\label{eq:gen-attention-parameter-stability}
 \norm{v_\theta(\mu)}_\infty\leq1+Hr\leq2p^2,
 \qquad
 \norm{v_\theta(\mu)-v_{\theta'}(\mu)}_\infty
 \leq H(2r+4r^3)h\leq6p^4h.
\end{equation}
These inequalities are uniform in \(\mu\).

Write \(B\eqdef p+1\leq2p\) for the readout entry bound.  Each hidden
preactivation has magnitude at most
\(rB(2p^2)+B\leq6p^4\).  Its change is at most
\[
 rh(2p^2)+rB(6p^4h)+h\leq15p^6h.
\]
ReLU is one-Lipschitz and preserves the preceding magnitude bound.
Each coordinate of the final affine output consequently changes by at most
\[
 Mh(6p^4)+MB(15p^6h)+h\leq37p^8h.
\]
Passing to the Euclidean output norm gives the claimed, slightly enlarged
constant \(40\sqrt{d'}p^8\).  Nonexpansiveness of \(\operatorname{clip}\) gives the
same estimate after clipping.

Cover every parameter interval by intervals of radius
\(h\eqdef\epsilon/(40\sqrt{d'}p^8)\), with centers in that parameter
interval.  Its length is at most \(2(p+1)\leq4p\), so
\(1+160\sqrt{d'}p^9/\epsilon\) centers per coordinate suffice.
There are at most \(p\) coordinates.  Finally, there are at most
\(p^2\) admissible integer pairs \((H,M)\); take the union of their
covers.  This proves \eqref{eq:gen-covering-number} without any dependence
on \(\omega\) or context length.
\end{proof}

\subsection{Squared loss on independent complete sequences}

Conditional centering identifies excess loss with prediction error, while
normalization lets us concentrate one bounded loss per sequence.
Write \(S_k\eqdef(n_k,Z^k,Y^k)\) for the training observations and
\(S\eqdef(n,Z,Y)\) for an independent observation with the same law.
For a history predictor \(g\), define
\begin{equation}
\label{eq:gen-risks}
 \ell_g(S)\eqdef\frac1n\sum_{i=1}^n
       \norm{g(\mu_{n,Z,i})-Y_i}^2,
 \qquad
 \mathcal L(g)\eqdef\mathbb E\ell_g(S),
 \qquad
 \widehat{\mathcal L}_N(g)\eqdef\frac1N\sum_{k=1}^N\ell_g(S_k).
\end{equation}
Here the expectation uses a fresh observation from the training law.  We use
the same notation for raw predictors when needed.  Put
\(\mathcal R_{\mathrm{pred}}(g)\eqdef\mathcal L(g)-\mathcal L(f^\star)\), and let
\(Z_g(S)\eqdef\ell_g(S)-\ell_{f^\star}(S)\) and
\(\widehat Z_g\eqdef N^{-1}\sum_{k=1}^N Z_g(S_k)\).
For every fixed bounded causal predictor, conditional centering gives
\begin{equation}
\label{eq:gen-excess-risk-identity}
 \mathcal R_{\mathrm{pred}}(g)=\mathbb E\frac1n\sum_{i=1}^n
   \norm{g(\mu_{n,Z,i})-f^\star(\mu_{n,Z,i})}^2.
\end{equation}
Indeed, expand each squared loss and condition its cross term on
\((n,Z_1,\ldots,Z_i)\).  For random \(n\), the conditional-mean assumption
is understood on \(\{n\geq i\}\) for each fixed \(i\).  Both predictions
and the factor \(1/n\) are measurable with respect to this information.
The assumption
\(\mathbb E[Y_i\mid n,Z_1,\ldots,Z_i]=F_n^\star(Z)_i\) makes that
cross term zero; conditioning on the full future sequence is not required.
Conditioning on \(n\) is consistent with the architecture, whose normalized
positional features encode the grid resolution.
The normalized averages are bounded, so taking expectations of these
random-length sums requires no moment condition on \(n\).
For the data-dependent predictor \(\widehat g\), population risk is
evaluated with the training sample held fixed.  Thus
\eqref{eq:gen-excess-risk-identity} is the risk in
\eqref{eq:learning-risk-main} for its associated family \(\widehat F\).

\begin{lemma}[Finite-class fast oracle inequality]
\label{lem:gen-finite-class-oracle}
Let \(\mathcal H\) be a finite class of \(K\) measurable causal history predictors
with values in the unit ball.  With probability at least \(1-\delta\),
simultaneously for every \(\eta\geq0\) and every
\(\eta\)-approximate empirical minimizer \(\widehat h\) over
\(\mathcal H\),
\begin{equation}
\label{eq:gen-finite-class-oracle}
 \mathcal R_{\mathrm{pred}}(\widehat h)
 \leq3\inf_{h\in\mathcal H}\mathcal R_{\mathrm{pred}}(h)+2\eta
       +96\frac{\log(2K/\delta)}N.
\end{equation}
\end{lemma}

\begin{proof}
For a unit-ball prediction \(h\), target \(f^\star\), and label \(Y\),
\[
 \left|\norm{h-Y}^2-\norm{f^\star-Y}^2\right|
 =\left|\ip{h-f^\star}{h+f^\star-2Y}\right|
 \leq4\norm{h-f^\star}.
\]
For a whole sequence, Jensen's inequality for the normalized average gives
\begin{equation}
\label{eq:gen-bernstein-condition}
 \abs{Z_h}\leq4,
 \qquad
 \mathbb E Z_h^2\leq16\mathcal R_{\mathrm{pred}}(h),
 \qquad
 \abs{Z_h-\mathbb E Z_h}\leq8.
\end{equation}
The first bound also follows directly because both sequence losses belong
to \([0,4]\).  No within-sequence independence is used for the second
bound: square the sequence average, apply Jensen, and then use
\eqref{eq:gen-excess-risk-identity}.

Bernstein's inequality for the \(N\) independent sequences and a union
bound over the \(K\) predictors give, with
\(t\eqdef\log(2K/\delta)/N\), the simultaneous event
\begin{equation}
\label{eq:gen-bernstein-event}
 \abs{\widehat Z_h-\mathcal R_{\mathrm{pred}}(h)}
 \leq\sqrt{32\mathcal R_{\mathrm{pred}}(h)t}+\frac{16}3t
 \leq\frac12\mathcal R_{\mathrm{pred}}(h)+22t,
 \qquad h\in\mathcal H.
\end{equation}
We have enlarged the usual linear Bernstein constant; the last inequality
uses \(\sqrt{32at}\leq a/2+16t\).  On this event,
\[
 \mathcal R_{\mathrm{pred}}(h)\leq2\widehat Z_h+44t,
 \qquad
 \widehat Z_h\leq\tfrac32\mathcal R_{\mathrm{pred}}(h)+22t.
\]
For any comparator \(h\in\mathcal H\), approximate empirical
minimization therefore gives
\[
 \mathcal R_{\mathrm{pred}}(\widehat h)
 \leq2\widehat Z_h+2\eta+44t
 \leq3\mathcal R_{\mathrm{pred}}(h)+2\eta+88t.
\]
Take the infimum and enlarge \(88\) to \(96\).  The event did not
depend on the chosen minimizer or tolerance.
\end{proof}

\paragraph{Transfer from a net to the real-parameter class.}
If \(g,h\) take values in the unit ball and
\(\norm{g-h}_{\infty,2}\leq\epsilon\), then
\begin{equation}
\label{eq:gen-loss-continuity}
 \abs{\ell_g(S)-\ell_h(S)}\leq4\epsilon,
 \qquad
 \abs{\mathcal R_{\mathrm{pred}}(g)-\mathcal R_{\mathrm{pred}}(h)}\leq4\epsilon.
\end{equation}
Let \(\mathcal H\subset\mathcal G_p\) be the net of
Lemma~\ref{lem:gen-parameter-cover}.  The lower half of
\eqref{eq:gen-bernstein-event}, together with
\eqref{eq:gen-loss-continuity}, yields simultaneously for every
\(g\in\mathcal G_p\), with probability at least \(1-\delta\),
\begin{equation}
\label{eq:gen-uniform-lower-concentration}
 \mathcal R_{\mathrm{pred}}(g)
 \leq2\widehat Z_g+12\epsilon
       +44\frac{\log(2K_p(\epsilon)/\delta)}N.
\end{equation}
Indeed, choose a net point \(h\) within \(\epsilon\) of \(g\),
then use \(\mathcal R_{\mathrm{pred}}(g)\leq\mathcal R_{\mathrm{pred}}(h)+4\epsilon\) and
\(\widehat Z_h\leq\widehat Z_g+4\epsilon\).

For comparison, if the empirical objective itself uses clipped outputs,
a \(\tau_{\mathrm{opt}}\)-approximate empirical minimizer \(\widetilde g\) over
\(\mathcal G_p\) has a nearest net point which is a
\((\tau_{\mathrm{opt}}+4\epsilon)\)-approximate minimizer over \(\mathcal H\).
Lemma~\ref{lem:gen-finite-class-oracle}, transfer back to
\(\widetilde g\), and approximation of each comparator by a net point give
\begin{equation}
\label{eq:gen-clipped-erm-oracle}
 \mathcal R_{\mathrm{pred}}(\widetilde g)
 \leq3\inf_{g\in\mathcal G_p}\mathcal R_{\mathrm{pred}}(g)+2\tau_{\mathrm{opt}}+24\epsilon
       +96\frac{\log(2K_p(\epsilon)/\delta)}N.
\end{equation}
Taking \(\epsilon=1/N\) and using
\eqref{eq:gen-covering-number} makes the last two terms at most
\[
 C_{d'}\frac{p\log(C_{d'}pN)+\log(2/\delta)}N.
\]
This establishes the usual fast oracle inequality for clipped-loss
training.  The raw-loss theorem requires the additional argument below.

\subsection{Raw empirical minimization followed by clipping}

The clipped-class concentration estimate and a single bounded raw
comparator suffice to analyze the main theorem's ordinary training loss.

\begin{proof}[Proof of Theorem~\ref{thm:generalization-main}]
Increase \(p_0\), depending only on \((d,d',\beta)\), so that the
comparator in \eqref{eq:gen-approximation-comparator} satisfies
\[
 \norm{t_p^\circ-f^\star}_{\infty,2}
 \leq C_{\mathrm{app}}q_\beta(p)\leq1,
 \qquad
 \norm{t_p^\circ}_{\infty,2}\leq2.
\]
It follows from \eqref{eq:gen-excess-risk-identity} that
\begin{equation}
\label{eq:gen-comparator-risk}
 \mathcal R_{\mathrm{pred}}(t_p^\circ)\leq C_{\mathrm{app}}^2q_\beta(p)^2.
\end{equation}
For this one raw comparator, the squared-loss difference at a position
satisfies
\[
 \left|\norm{t_p^\circ-Y}^2-\norm{f^\star-Y}^2\right|
 \leq5\norm{t_p^\circ-f^\star}\leq5.
\]
The same Jensen argument as before therefore proves
\(\mathbb E Z_{t_p^\circ}^2\leq25\mathcal R_{\mathrm{pred}}(t_p^\circ)\) and
\(\abs{Z_{t_p^\circ}-\mathbb E Z_{t_p^\circ}}\leq10\).
Bernstein's inequality, with \(s\eqdef\log(4/\delta)/N\), gives an event
of probability at least \(1-\delta/2\) on which
\begin{equation}
\label{eq:gen-raw-comparator-concentration}
\begin{aligned}
 \widehat Z_{t_p^\circ}
 &\leq\mathcal R_{\mathrm{pred}}(t_p^\circ)
       +\sqrt{50\mathcal R_{\mathrm{pred}}(t_p^\circ)s}+\frac{20}3s\\
 &\leq\tfrac32\mathcal R_{\mathrm{pred}}(t_p^\circ)+32s.
\end{aligned}
\end{equation}
Here \(\sqrt{50as}\leq a/2+25s\).

On another event of probability at least \(1-\delta/2\),
\eqref{eq:gen-uniform-lower-concentration} holds with
\(\epsilon=1/N\) and \(\delta\) there replaced by \(\delta/2\).
Let \(\widehat t\eqdef t_{\widehat\Theta_N}\) be the history realization of
the raw approximate empirical minimizer \(T_{\widehat\Theta_N}\), and put
\(\widehat g\eqdef\operatorname{clip}\circ\widehat t\).
Since all labels belong to the unit ball, projection decreases their
squared distance.  The raw optimization condition
\eqref{eq:learning-erm-main} thus implies
\begin{equation}
\label{eq:gen-raw-clipped-bridge}
 \widehat Z_{\widehat g}
 \leq\widehat Z_{\widehat t}
 \leq\widehat Z_{t_p^\circ}+\tau_{\mathrm{opt}}.
\end{equation}
Combining the two concentration events, whose intersection has probability
at least \(1-\delta\), gives
\begin{align*}
 \mathcal R_{\mathrm{pred}}(\widehat F)
 &=\mathcal R_{\mathrm{pred}}(\widehat g)\\
 &\leq3\mathcal R_{\mathrm{pred}}(t_p^\circ)+2\tau_{\mathrm{opt}}+\frac{12}N
   +44\frac{\log(4K_p(1/N)/\delta)}N\\
 &\qquad
   +64\frac{\log(4/\delta)}N\\
 &\leq C_{d,d',\beta}
 \left\{
 \left(\frac{\log\log p}{\log p}\right)^{2\beta/(d+2)}
 +\frac{p\log(C_{d'}pN)+\log(2/\delta)}N
 \right\}+2\tau_{\mathrm{opt}}.
\end{align*}
The last step uses \eqref{eq:gen-covering-number} and
\eqref{eq:gen-comparator-risk}.  This proves the stated bound after
enlarging its constants.  In particular, no uniform concentration bound
for the potentially large raw losses of all of \(\mathfrak T_p\) was
needed.  Only the clipped class and one bounded raw comparator were used.

For \(p=\lfloor\sqrt N\rfloor\) and sufficiently large \(N\),
\(q_\beta(p)\leq C_{d,\beta}q_\beta(N)\), while
\[
 \frac{p\log(C_{d'}pN)}N
 =O_{d'}\!\left(\frac{\log N}{\sqrt N}\right)
 =o\bigl(q_\beta(N)^2\bigr).
\]
Consequently,
\begin{equation}
\label{eq:gen-root-sample-rate}
 \mathcal R_{\mathrm{pred}}(\widehat F)^{1/2}
 \leq C_{d,d',\beta}
 \left(\frac{\log\log N}{\log N}\right)^{\beta/(d+2)}
 +C_{d'}\sqrt{\frac{\log(2/\delta)}N}+\sqrt{2\tau_{\mathrm{opt}}},
\end{equation}
using \(\sqrt{a+b+c}\leq\sqrt a+\sqrt b+\sqrt c\).
This is a conditional test-risk bound on the same event of probability
at least \(1-\delta\) over training samples. In particular, the logarithmic
root rate is retained whenever
\(\tau_{\mathrm{opt}}=O(q_\beta(N)^2)\) and confidence is fixed.
Every estimate is independent of \(\omega\), the law of sequence
lengths, and any maximum length.
\end{proof}

\paragraph{Existence and optimization.}
At fixed \(p\), the class is a finite union of compact parameter boxes,
and the raw empirical loss is continuous on each box.  Thus an empirical
minimum exists.  The theorem assumes a measurable choice; measurable
positive-tolerance choices can also be obtained by searching a countable
dense parameter subset. More concretely, uniform continuity on each box
ensures that exhaustive minimization on sufficiently fine finite parameter
grids attains any prescribed tolerance \(\tau_{\mathrm{opt}}>0\). This is an
existence argument, not an efficient training method. For a practical
optimizer, \(\tau_{\mathrm{opt}}\) must bound the gap to the global empirical
minimum; a small gradient or a decrease in training loss does not by itself
certify this gap. The proof does not give such a guarantee for SGD. To retain the
logarithmic statistical rate, it suffices that
\(\tau_{\mathrm{opt}}=O(q_\beta(N)^2)\) when
\(p=\lfloor\sqrt N\rfloor\), as shown in
\eqref{eq:gen-root-sample-rate}.

\subsection{Sampling variants and scope}
\label{app:learning-scope}

The proof extends to other independently sampled causal observations,
but its loss normalization and same-distribution interpretation remain
essential.

\paragraph{Why continuity alone does not give uniform learning rate.}
This obstruction already occurs for length-one sequences. Given \(N\),
choose \(2N\) distinct tokens in \(\Omega\), give them independent random
labels in \(\{-1,1\}\), and sample uniformly from these tokens. Any such
label assignment extends to a continuous map
\(h:\Omega\to[-1,1]\); the family
\(F_n^\star(z)_i\eqdef h(z_i)e_1\) is continuously extendable and causal, with
path extension \(h\circ x\) in its first output coordinate. After \(N\)
noiseless training observations, the sign at an unseen test token remains
independent and uniform. Its conditional expected squared error is at
least one, and the probability of an unseen test token is
\((1-1/(2N))^N\geq1/2\). Averaging over label assignments shows that,
for every learner, some bounded continuous teacher has expected prediction
risk at least \(1/2\) for this design. Thus no rate tending to zero can hold
uniformly over all continuous teachers and sampling laws. Some quantitative
class restriction is necessary; the \(\beta\)-smooth target condition is one
sufficient choice, not a necessary characterization of learnability.

\paragraph{Independent histories and fully observed paths.}
The proof is unchanged for iid observations \((\mu,Y)\), where
\(\mu\in\cM_\omega\), \(\norm{Y}\leq1\), and
\(\mathbb E[Y\mid\mu]=f^\star(\mu)\), using loss
\(\norm{g(\mu)-Y}^2\).  Such a design may mix discrete and continuous
histories.  It also applies to iid, fully observed labeled paths
\((x,Y)\), with \(x\in X_\infty^\omega\) and loss
\[
 \ell_g(x,Y)\eqdef\int_0^1\norm{g(\mu_{x,t})-Y(t)}^2\,\dd t,
 \qquad
 \mu_{x,t}\eqdef\frac1t\int_0^t\delta_{(x(s),s)}\,\dd s\quad(t>0),
\]
and \(\mu_{x,0}\eqdef\delta_{(x(0),0)}\).  Assume joint measurability,
\(\norm{Y(t)}\leq1\) almost surely for almost every \(t\), and
\(\mathbb E[Y(t)\mid x|_{[0,t]}]=f^\star(\mu_{x,t})\) for almost
every \(t\).  Conditional centering and Jensen's inequality under the
time integral replace their finite-average versions.  Sparse observation
or approximate quadrature of this loss requires additional error terms.

\paragraph{What the sample size counts.}
Each complete sequence has total weight one.  The bounds need no moment
assumption on the random length, because the normalized sequence loss is
bounded after clipping.  Arbitrary dependence of tokens and labels within
an observation is allowed; overlapping windows from one series are not
thereby independent observations.  Nor does the theorem silently cover
the different objective obtained by dividing the dataset's total token
loss by its total number of tokens.

\paragraph{Prediction risk versus length extrapolation.}
For sequence sampling, define the induced history law
\(\overline P\eqdef\mathbb E[n^{-1}\sum_{i=1}^n\delta_{\mu_{n,Z,i}}]\).
The conclusion controls \(\norm{\widehat g-f^\star}_{L^2(\overline P)}\),
not uniform error on all histories.  Distribution-independent constants
do not turn training on short sequences into a guarantee under an
arbitrary new long-sequence distribution.  If another history law \(Q\)
satisfies \(\dd Q/\dd\overline P\leq\kappa\), then one can infer
\(\norm{\widehat g-f^\star}_{L^2(Q)}^2\leq
\kappa\mathcal R_{\mathrm{pred}}(\widehat F)\); without coverage assumptions, no such
distribution-shift conclusion follows.

\paragraph{The role of the hypotheses.}
Weight control is essential to the parameter cover proved here.  Finite
covers at a positive prediction accuracy do not imply finite
zero-threshold VC dimension: thresholding is discontinuous, and arbitrarily
small classification margins are not controlled by this argument.  The
fast estimation rate uses bounded squared regression loss and conditional
centering; unrestricted cross-entropy is not covered.  The target
condition is precisely boundedness of the \(\beta\)-smooth target seminorm in
Definition~\ref{def:causal-smoothness}, not generic Fr\'echet smoothness
of a path operator.  Finally, the theorem is an upper bound only: neither
its statistical optimality nor a statistical lower bound follows from
the finite-description approximation lower bounds.

\section{Empirical protocol and full results}
\label{app:empirical-details}

This appendix records the sampling, pretrained maps, scale sensitivity, and
refinement checks behind Section~\ref{sec:numerics}.  Raw physical-signal
baselines complement the continuous-content and layer-wise text experiments.
The main comparison contrasts continuous patch embeddings with BigBird text
input embeddings. Together these measurements characterize finite-scale,
representation-dependent regularity; they do not verify a common asymptotic
H\"older class.

\subsection{Estimator interpretation and sensitivity}

The maximum in \eqref{eq:numeric-uniform-increment} ranges over every starting
position at each tested dyadic lag, making it closer to the all-pairs H\"older
geometry used in the optimality benchmark than an \(L^p\) structure function.
The subscript \(\infty\) denotes this maximum, not a continuum-limit estimate.
The factor \(d^{-1/2}\) reports root-mean-square coordinate displacement and
removes a trivial \(\sqrt d\) factor when coordinates have comparable scales;
it neither makes arbitrary representations metrically comparable nor
normalizes them into the theorem's token ball.  The raw physical baselines and
BigBird input comparison use the endpoint-inclusive grid \(t_k\eqdef k/(N-1)\)
for each observed sequence.  The continuous-content comparison preserves the raw
window's time axis: a lag of \(r\) feature positions at stride \(s\) has
\(h_r\eqdef sr/(N_{\rm raw}-1)\).  Changing a fixed time denominator changes the
intercept, but matching physical bands is essential when comparing cadences.

Equation~\eqref{eq:numeric-uniform-fit} is fit separately for every sequence or
window. Thus \(C\) is a finite-band prefactor in the chosen encoding and RMS
convention, not a common theorem-level H\"older constant.  A fixed positive
rescaling preserves \(\widehat\alpha_\infty\) and \(R^2\) while rescaling
\(C\).  The raw-baseline table below reports the mean and sample standard
deviation of \(\widehat\alpha_\infty\) and the mean \(R^2\).
As sensitivity checks, we use Theil--Sen slopes and replace the maximum by the
99th percentile of increments.  On the BigBird input rows, Theil--Sen and
ordinary least squares agree within a few \(10^{-3}\).  The percentile slope is not a
uniform H\"older exponent.  More generally, a finite-band fit does not certify
regularity: the exact RMS-normalized grid H\"older seminorm takes the maximum of
\(M_\infty(r)/h_r^\alpha\) over every lag, and a low \(R^2\) argues against a
single power law even on the fitted band.  We do not use regression standard
errors from the correlated lag points as uncertainty across recordings.

In Figure~\ref{fig:continuous-holder-main}, every window \(w\) is normalized
separately: \(R_w(r)\eqdef M_{\infty,w}(r)/M_{\infty,w}(1)\). The displayed mean
and sample SD are computed across these ratios, so the first-lag point is
exactly one with zero SD. This compares curve shapes with equal weight per
window and leaves each window's fitted slope unchanged.

\subsection{Data and sampling}

The benchmark uses a fixed sampling rule in every domain.  For the six scalar
series, we predeclare the dyadic-plus-one lengths reported in
Table~\ref{tab:uniform-holder-results} and retain up to eight time-spread,
mutually disjoint complete windows, yielding between three and eight windows
per series.  The series are Jena air temperature~\citep{jenaweather}, Beijing
PM2.5 concentration~\citep{chen2015beijing}, building appliance
consumption~\citep{candanedo2017appliances}, Metro Interstate traffic
volume~\citep{hogue2019metro}, ETT transformer-oil temperature, and ETT
high-useful-load measurements~\citep{zhou2021informer}.

ECG comprises centered five-second Lead-I windows from 24 CPSC/PhysioNet
records~\citep{alday2020challenge,alday2022physionetdataset,pollard2026physionet}, converted to
millivolts using each record's gain and baseline and antialias-resampled from
500 Hz to 100 Hz.  Ground motion
comprises 12 earthquake and 12 quiet 30-second, three-component windows from
station \texttt{CI.PASC.00}, sampled at 100 Hz; waveforms come from the IRIS
FDSN dataselect service and event origins from the USGS event service.  The
text controls are WikiText-103~\citep{merity2017wikitext}, AG
News~\citep{zhang2015character}, IMDb~\citep{maas2011imdb}, and this
manuscript's mathematical \LaTeX{} source.  For the first three, we concatenate
documents with the tokenizer separator and extract four nonoverlapping
\(N=16{,}384\) windows.  For the \LaTeX{} control, we join the frozen manuscript
snapshot with blank lines and periodically repeat streams shorter than 65,536
tokens.  Its four windows are offsets of a synthetic periodic stream, not
independent documents; rerunning after manuscript edits changes this control.

The scalar replicates are disjoint windows from individual recorded series,
not independent stations or buildings; the two ETT channels share one source.
The reported SD is therefore descriptive of the retained collection.  Complete
window selection can favor well-observed periods, and missing-data fractions
in source metadata describe the entire source rather than each selected
window.  The seismic cache records waveforms, generic event/noise identifiers,
channels, and cadence, but lacks the acquisition timestamps, event identifiers,
and request manifest needed for exact re-acquisition.  Its event/quiet labels
and source coverage are consequently less reproducible than the scalar data;
the new continuous-content comparison uses only the six scalar series.

\subsection{Pretrained representations}

Table~\ref{tab:uniform-holder-encoders} summarizes the two pretrained input
maps. Chronos-Bolt applies a continuous residual ReLU network to standardized
measurement patches, as detailed below. BigBird combines 768-dimensional
learned word and token-type embeddings with its learned
absolute-position table and applies the pretrained input
LayerNorm~\citep{zaheer2020bigbird}. At \(N=16{,}384\), we linearly interpolate
the learned position vectors over the normalized interval. This is our
input-only adaptation, not the RoPE index-rescaling and fine-tuning procedure
of \citet{chen2023position}; no attention or fine-tuning is applied.
The separate depth experiment
in Appendix~\ref{app:text-depth} uses complete models at native context lengths.

The window-dependent normalization used by Bolt is part of the chosen
empirical encoder. It does not certify that one fixed representation map has
a common theorem-level H\"older prefactor across sequences and resolutions.
The saved metadata pins the Bolt and BigBird snapshots and the three streamed
dataset revisions by commit hash, and records the Bolt tensor hashes,
BigBird LayerNorm epsilon, and software versions.

\begin{table}[!htbp]
\centering
\caption{\textbf{Pretrained front ends.} Every row
uses Equations~\eqref{eq:numeric-uniform-increment}--
\eqref{eq:numeric-uniform-fit}; \(d\) is the increment dimension.}
\label{tab:uniform-holder-encoders}
\small
\setlength{\tabcolsep}{4pt}
\begin{tabular}{@{}p{18mm}p{25mm}cp{75mm}@{}}
\toprule
data type & representation & \(d\) & embedding map \\
\midrule
Scalar series & Chronos-Bolt-tiny & 256 & Standardized 16-value patch plus
16 observation-mask entries, pretrained residual ReLU MLP; content only,
evaluated at stride 16 and diagnostic stride one. \\
NLP & BigBird-RoBERTa & 768 & Learned word and token-type embeddings, linearly
interpolated learned absolute position, pretrained input LayerNorm. \\
\bottomrule
\end{tabular}
\end{table}
\FloatBarrier

\subsection{The continuous-content comparison}
\label{app:continuous-holder-protocol}

We extract the input patch embedding of Chronos-Bolt-tiny~\citep{ansari2024chronos}, revision
\href{https://huggingface.co/amazon/chronos-bolt-tiny/tree/93a81296e5ea438ce4d2d2dd87decec2fd16099b}{\texttt{93a8129}},
following the pinned official implementation~\citep{chronosbolt2024}.\footnote{\href{https://github.com/amazon-science/chronos-forecasting/blob/v2.0.0/src/chronos/chronos_bolt.py}{Official source: \texttt{chronos-forecasting v2.0.0}, \path{src/chronos/chronos_bolt.py}.}}
For each retained scalar
window, mean and population standard deviation are computed once and frozen
for all subsequent comparisons; a zero standard deviation is replaced by
\(10^{-5}\).  Let \(p_k\in\R^{16}\) contain 16 consecutive standardized
measurements.  The observed-value mask is \(\mathbf1\in\R^{16}\), since all
retained windows are complete.  With \(u_k\eqdef[p_k;\mathbf1]\in\R^{32}\), the
frozen content map is
\begin{equation}
 E(p_k)\eqdef W_o\operatorname{ReLU}(W_hu_k+b_h)+b_o+W_ru_k+b_r,
 \label{eq:bolt-content-map}
\end{equation}
where \(W_h\in\R^{1024\times32}\),
\(W_o\in\R^{256\times1024}\), and \(W_r\in\R^{256\times32}\).
All six tensors are pretrained and dropout is inactive.  No positional vector
or attention output is added.  Bolt is an encoder--decoder model, so these
measurements concern a local content map, not causal processing by its encoder.

Native complete patches are end-aligned at stride 16, after excluding a
possible incomplete left patch.  Diagnostic dense evaluation applies exactly
the same map to every valid 16-sample patch at stride one.  Native outputs
match the corresponding dense outputs with zero numerical discrepancy in all
29 retained windows.  We omit the non-temporal REG token and do not apply the
forecasting pipeline's 2,048-sample truncation.  Thus ``native'' specifies the
patch cadence and alignment under the shared full-window normalization, not
the full forecasting pipeline.  Raw endpoints are restricted to the valid
dense patch endpoints, excluding the first 15 raw samples; this accounts for
small differences from the full-window raw-baseline table below.

All dyadic feature lags with \(h_r\le1/4\) are retained.  The primary band is
\(h_r\le1/32\), with at least four positive, finite measurements required for
a fit.  We also report the first four dyads, the coarse band
\(1/32\le h_r\le1/4\), and the matched physical fine band
\(16\le sr\le(N_{\rm raw}-1)/32\).  No band is chosen according to its slope
or \(R^2\).  The dense and native primary fits in
Table~\ref{tab:continuous-holder-main} include different finest scales;
Table~\ref{tab:holder-matched-bands} removes that difference.  Remaining
differences arise because the maximum uses more starting positions at dense
cadence.  Beijing and traffic have only two native fine-band lags, so we leave
their estimates unavailable.

\begin{table}[!htbp]
\centering\small
\caption{\textbf{Matched physical fine band for the same Bolt map.}
Mean \(\widehat\alpha_\infty\pm\) sample SD; the last column counts dyadic
lags.  Native cadence uses a subset of the dense starting positions.}
\label{tab:holder-matched-bands}
\begin{tabular}{@{}lrrr@{}}
\toprule
domain & dense stride 1 & native stride 16 & lags \\
\midrule
Jena weather & $.157\pm.039$ & $.167\pm.040$ & 7 \\
Beijing PM2.5 & -- & -- & 2 \\
Building appliances & $.020\pm.022$ & $.020\pm.017$ & 4 \\
Road traffic & -- & -- & 2 \\
Transformer oil temp. & $.133\pm.044$ & $.175\pm.041$ & 4 \\
Transformer load & $.025\pm.027$ & $.043\pm.055$ & 4 \\
\bottomrule
\end{tabular}
\end{table}

Two permutation controls distinguish temporal organization from the patching
operation: one seeded permutation of the standardized scalar input before
patch extraction, and one of the resulting dense vectors.  Input-shuffled
mean slopes are \(0.005\)--\(0.035\), whereas output-shuffled means are within
\(0.003\) of zero.  These are controls with one permutation per window, not a
permutation significance test.  Unprojected delay patches \(p_k\in\R^{16}\)
are also measured: their positive slopes show that the learned map does not
create all the observed regularity.

The connection to the H\"older model is an upper-bound preservation property.
If \(x\) is a scalar \(L\)-H\"older-\(\alpha\) path, the patch
\(p(t)\eqdef(x(t-\delta_j))_{j=1}^{P}\) uses fixed physical offsets on its valid
time interval, and \(E\) is \(K\)-Lipschitz in Euclidean norm, then
\begin{equation}
 \frac{\|E(p(t))-E(p(s))\|_2}{\sqrt d}
 \le K\sqrt{P/d}\,L|t-s|^\alpha.
 \label{eq:content-holder-transfer}
\end{equation}
The fixed ReLU map in \eqref{eq:bolt-content-map} is Lipschitz.  This
observation does not imply equality of exponents, and a common constant also
requires controlled normalization.  Across resolutions, fixed sample-count
patches change their physical support; the native frontend is not thereby
identified with one fixed delayed-path map.

\subsection{Raw physical-signal baselines}

Table~\ref{tab:uniform-holder-results} reports the full-window physical
baselines before learned encoding. The scalar curves use each window divided
by its mean absolute value, which does not change its fitted slope. ECG uses
the resampled waveform in millivolts. Ground motion uses three waveform
components, each centered and divided by its window SD plus \(10^{-6}\).
Thus ``raw'' denotes measurements with continuous preprocessing, not a common
cross-domain amplitude convention. Mean maximum-increment slopes range from
\(0.041\) to \(0.389\); their variability and fit quality qualify any
regularity interpretation.

\begin{table}[!htbp]
\centering
\caption{\textbf{Raw physical-signal worst-increment regressions.}
\(N\) is waveform length, \(d\) the number of measured coordinates, and
\(m\) the number of windows. Slopes are mean \(\pm\) sample SD on
\(h\le1/32\); \(R^2\) is the mean fit quality, and
\(\alpha_{.99}\) the mean 99th-percentile sensitivity slope.}
\label{tab:uniform-holder-results}
\small
\setlength{\tabcolsep}{4pt}
\begin{tabular}{@{}lrrrrrr@{}}
\toprule
domain & \(N\) & \(d\) & \(m\) & \(\widehat\alpha_\infty\) & \(R^2\) & \(\alpha_{.99}\) \\
\midrule
Jena weather & 32,769 & 1 & 8 & $.253\pm.040$ & .89 & .421 \\
Beijing PM2.5 & 1,025 & 1 & 6 & $.139\pm.061$ & .73 & .365 \\
Building appliances & 4,097 & 1 & 4 & $.041\pm.020$ & .61 & .101 \\
Road traffic & 1,025 & 1 & 3 & $.172\pm.019$ & .60 & .178 \\
Transformer oil temp. & 4,097 & 1 & 4 & $.134\pm.027$ & .85 & .291 \\
Transformer load & 4,097 & 1 & 4 & $.129\pm.025$ & .77 & .176 \\
Cardiac ECG & 500 & 1 & 24 & $.185\pm.174$ & .53 & .261 \\
Ground motion & 3,000 & 3 & 24 & $.389\pm.281$ & .65 & .439 \\
\bottomrule
\end{tabular}
\end{table}
\FloatBarrier

A common normalized cutoff spans different physical durations and contains
four fine lags for ECG versus eleven for Jena. Percentile slopes can differ
substantially from maximum-increment slopes: for example, Jena gives
\(0.421\) versus \(0.253\). RMS and percentile slopes describe weaker
increment statistics and cannot replace the uniform maximum in the theory.

\subsection{Calibration and direct refinement checks}
\label{app:holder-calibration}

We calibrate the same fine-band estimator on 12 independent realizations per
synthetic scenario, each of length \(4,097\): a smooth sinusoid with random
phase; fractional Brownian motion (fBm) with \(H=0.25,0.50,0.75\), generated
by circulant embedding of fractional Gaussian noise; independent Gaussian
noise; and a sinusoid with one height-two jump at a random location in
\([0.3,0.7]\).  The raw and delay-patch baselines use the same standardization
and valid endpoints as Bolt.  For fBm, \(H\) is the supremal pathwise exponent:
paths are H\"older for every \(\alpha<H\), not generally at \(H\) itself.

\begin{figure}[!htbp]
\centering
\includegraphics[width=\textwidth]{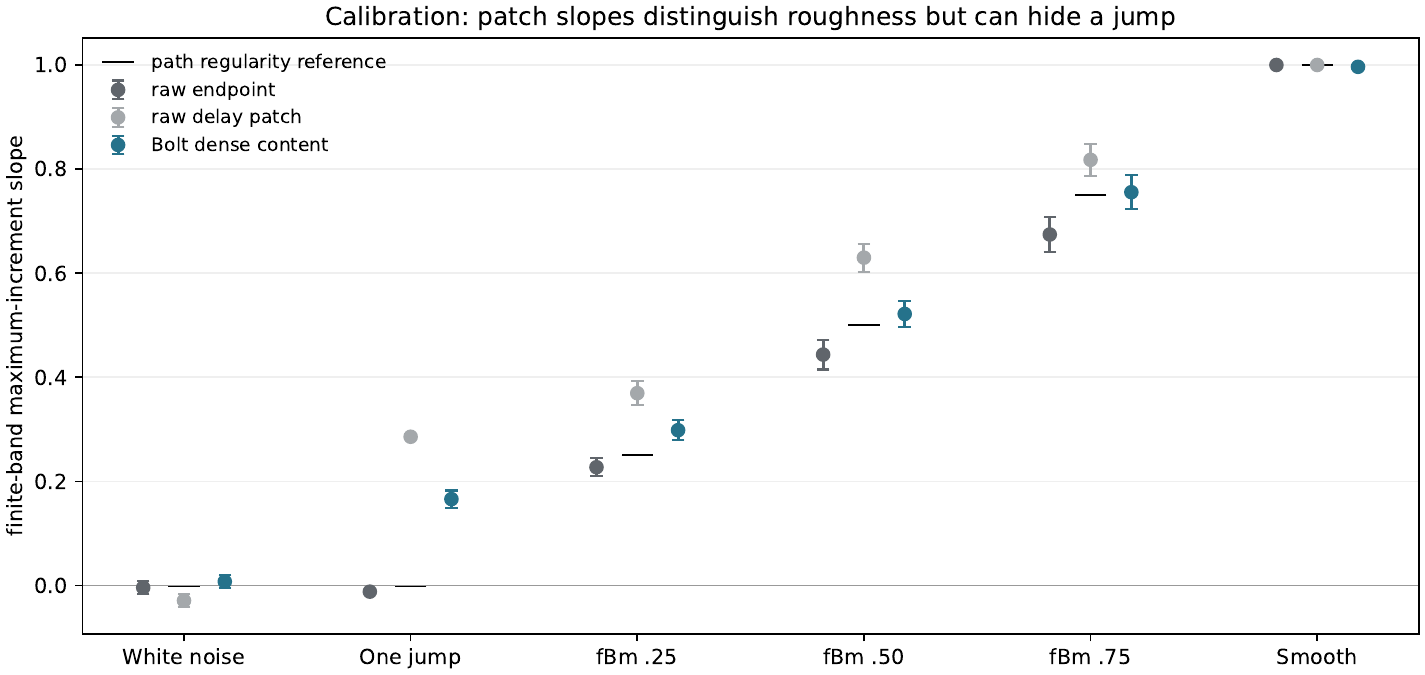}
\caption{\textbf{Calibration of finite-band slopes.} Twelve realizations per
scenario, with mean \(\pm\) sample SD.  The ordering of fractional-Brownian
roughness is visible, but the positive slope of the jump control after patch
encoding prevents interpreting the regression as an asymptotic certificate.}
\label{fig:holder-calibration}
\end{figure}

In Figure~\ref{fig:holder-calibration}, the raw fBm slopes are
\(0.227,0.444,0.674\), and dense Bolt gives
\(0.298,0.522,0.755\); smooth and white-noise Bolt slopes are \(0.996\) and
\(0.008\).  These values support sensitivity to roughness while exposing
finite-scale and representation bias.  Crucially, the jump gives
\(0.166\pm0.017\) after dense Bolt, compared with \(-0.012\pm0.005\) before
encoding: distributing a transition across patch coordinates changes its
observed scaling.  On the matched band of raw lags at least 16, this same
dense jump representation gives \(-0.013\pm0.006\), exposing the scale
crossover.  Positive fitted slopes therefore need refinement checks,
not a universal bias correction.

For a complementary check, restrict each fixed dense feature path, without
interpolation or re-encoding, to nested grids of \(n=65,129,257,513\) points.
On each grid compute the exact all-pair seminorm
\begin{equation}
 A_n(\alpha)\eqdef\max_{i<j}
 \frac{d^{-1/2}\|z_j-z_i\|_2}{|t_j-t_i|^\alpha},
 \qquad \alpha\in\{0.1,0.25,0.5\}.
 \label{eq:numeric-exact-seminorm}
\end{equation}
The grid endpoints span the largest power-of-two number of sample intervals
inside the valid dense path, centered in that path.  For the retained lengths
this is one half of the raw window; times remain normalized by
\(N_{\rm raw}-1\).  Every positive lag and every starting position on each
grid enter the maximum.  The grids are nested, so \(A_n\) cannot decrease.
The 513-point cap reaches raw spacing one for Beijing and traffic, four for
appliances and the two ETT channels, and 32 for Jena.  This check therefore
does not reach the finest measurement cadence in every domain, nor cover the
whole observation interval.

\begin{figure}[!htbp]
\centering
\includegraphics[width=\textwidth]{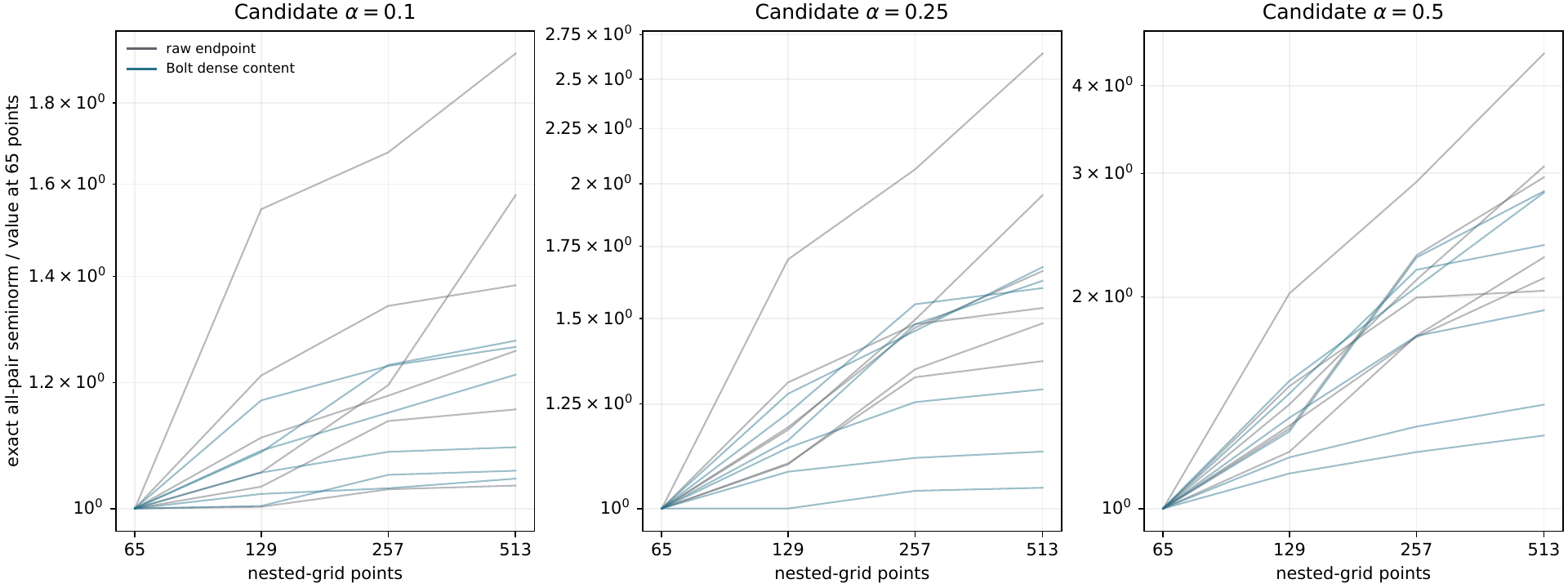}
\caption{\textbf{Exact candidate H\"older prefactors on nested restrictions.}
Each thin curve is a domain's mean ratio \(A_n(\alpha)/A_{65}(\alpha)\),
for raw endpoints or dense Bolt content. Domains are not pooled as independent
samples. The finite cadence and central subinterval limit the scope of this
stability check.}
\label{fig:holder-seminorm}
\end{figure}

In Figure~\ref{fig:holder-seminorm}, at \(\alpha=0.25\), dense traffic and Beijing have mean
\(A_{513}/A_{65}\) of \(1.045\) and \(1.129\). Dense weather and appliances still grow
by factors \(1.626\) and \(1.674\).  Thus some continuous representations
have encouraging prefactor stability over these grids, but the observations
do not establish a shared \(0.25\)-H\"older class.  We extract no pass/fail
exponent.
\FloatBarrier

\subsection{A structured-text corpus screen}
\label{app:text-corpus-screen}

We extend the original four-source input baseline with three candidates
chosen before measuring their increment curves: Shakespeare drama and verse,
the King James Bible as formulaic historical prose, and a Python-code
control. The motivation is to test whether more structured textual sources
exhibit stronger geometric regularity. All three outcomes are retained;
neither sources, windows, nor fitting bands are selected for favorable slopes.

\paragraph{Data and unchanged embedding.}
The public texts are the \emph{Tiny Shakespeare} corpus and Project
Gutenberg's King James Bible.\footnote{Sources:
\href{https://github.com/karpathy/char-rnn/tree/master/data/tinyshakespeare}{\texttt{char-rnn/data/tinyshakespeare}} and
\href{https://www.gutenberg.org/ebooks/10}{Project Gutenberg, ebook 10}.
Download-prefix hashes and exact retained token IDs are saved with the results.}
We read at most one million bytes of each source, remove the Gutenberg
header, and retain the first four disjoint \(N=16{,}384\)-token windows.
The Python source is a sorted snapshot of Git-tracked files under
\path{code/}, including vendored files, concatenated once. The vendored files
were included because top-level code alone was too short; this adjustment
preceded measurement of any candidate's increments. Neither the new public
texts nor the Python stream is periodically repeated. Only token windows,
source hashes, and small statistics are stored, not full corpora or embeddings.

Every vector uses the same frozen BigBird \(d=768\) input map as the
original baseline: word and token-type embeddings plus the learned position
table linearly interpolated to \(N\) positions, followed by the pretrained
LayerNorm. No attention, training, or smoothing is applied. A paired control
permutes each new window's content IDs before this map, keeping the same
positional grid. Fits use the same nine dyadic lags \(r=1,\ldots,256\)
satisfying \(r/(N-1)\le1/32\). In addition to the maximum, we fit the RMS
increment \(M_2(r)\eqdef[(N-r)^{-1}\sum_k d^{-1}\|z_{k+r}-z_k\|_2^2]^{1/2}\).
Table~\ref{tab:text-corpus-screen} reports all seven sources; SD describes
four windows per source, not independent documents or corpus-level uncertainty.

\begin{table}[!htbp]
\centering\small
\caption{\textbf{The text screen does not reveal a more regular input regime.}
Mean \(\pm\) sample SD on the fixed fine band, using identical full-input
BigBird embeddings. Shuffled controls are measured for all three new sources;
dashes denote unmeasured controls at this long context. The original
\LaTeX{} stream is a repeated technical control, as described above.}
\label{tab:text-corpus-screen}
\setlength{\tabcolsep}{4pt}
\begin{tabular}{@{}lccc@{}}
\toprule
Source & $\widehat\alpha_\infty$ & $\widehat\alpha_2$ & Shuffled $\widehat\alpha_\infty$ \\
\midrule
WikiText-103 & $0.0122\pm0.0012$ & $0.0188\pm0.0006$ & -- \\
AG News & $0.0126\pm0.0018$ & $0.0191\pm0.0003$ & -- \\
IMDb & $0.0130\pm0.0008$ & $0.0180\pm0.0001$ & -- \\
\LaTeX{} control & $0.0119\pm0.0012$ & $0.0148\pm0.0004$ & -- \\
Shakespeare drama & $0.0106\pm0.0018$ & $0.0173\pm0.0002$ & $0.0109\pm0.0017$ \\
KJV Bible & $0.0115\pm0.0005$ & $0.0178\pm0.0001$ & $0.0117\pm0.0024$ \\
Python control & $0.0124\pm0.0007$ & $0.0150\pm0.0003$ & $0.0103\pm0.0027$ \\
\bottomrule
\end{tabular}

\end{table}

\paragraph{Outcome and scope.}
The new maximum-slope means are \(0.0106\), \(0.0115\), and \(0.0124\):
none exceeds the original IMDb mean \(0.0130\). Shuffling produces similarly
small values, and mean RMS slopes remain between \(0.0148\) and \(0.0191\)
across all seven sources. Thus textual structure or formulaic syntax does
not by itself make these subword embeddings geometrically regular.
The first-four-dyad and expanded \(h\le1/8\) fits are also retained; their
largest admissible dyads are \(8\) and \(1024\), respectively, since
\(2048/(16384-1)>1/8\). No alternative band supplies a strong positive
regime in these additional sources.
This remains a finite-scale negative result for one input pipeline, not a
statement about every possible text representation. The separate layer-wise
study below investigates whether processing changes this geometry.

The experiment is reproduced by \path{code/text_corpus_holder_screen.py}:
\texttt{--prepare} freezes small token inputs, \texttt{--run} evaluates the
cached input map, and \texttt{--report} reanalyzes saved increments. The
executed \path{code/holder_increment_comparison.ipynb} regenerates the
publication panels and the full screen report without new inference.
\FloatBarrier

\subsection{Text representations across transformer depth}
\label{app:text-depth}

Nearly flat input-embedding curves need not describe the geometry of internal
representations. We therefore repeat the same multiscale diagnostic after
every transformer block, asking whether depth improves the effective
H\"older scaling of a fixed text window. The main observation is a modest
but consistent intermediate-layer gain in BigBird, rather than a monotone
increase through the network.

\paragraph{Matched windows and native inference.}
We evaluate the frozen 12-block BigBird encoder~\citep{zaheer2020bigbird}
at its native 4096-token context and the six-block causal DistilGPT2 decoder
from the GPT-2 family~\citep{radford2019language} at 1024 tokens.\footnote{\raggedright Pinned
snapshots: \href{https://huggingface.co/google/bigbird-roberta-base/tree/5a145f7852cba9bd431386a58137bf8a29903b90}{BigBird \texttt{5a145f7}} and
\href{https://huggingface.co/distilbert/distilgpt2/tree/2290a62682d06624634c1f46a6ad5be0f47f38aa}{DistilGPT2 \texttt{2290a62}}.
Inference uses \texttt{transformers==4.57.6}, CPU float32, and evaluation mode.
BigBird's unmodified block-sparse evaluation helper returns block-zero
indices rather than sampled random links; these are links to block zero,
not zero attention weights. We impose no custom attention pattern.}
For each model and each of WikiText-103, AG News, IMDb, and a frozen
\LaTeX{} source-tree control, we retain the first eight nonoverlapping full
content-token windows. The technical control includes retained source variants and is
concatenated once, without periodic repetition. Public documents are joined
with the tokenizer's separator. A paired, seeded control permutes content IDs
before inference, including internal separators but preserving BigBird's
outer markers. The two models use different
tokenizers and window lengths; only within-model comparisons are paired.

All depths use exactly the same windows and lags, with no training, smoothing,
or positional interpolation. BigBird's two added outer markers are excluded
from measurement, leaving \(N=4094\) content vectors; DistilGPT2 has
\(N=1024\). Both have increment dimension \(d=768\). Depth zero denotes the
embedding representation. We record every block output and, for DistilGPT2,
its final LayerNorm separately from block six. Primary fits use
\eqref{eq:numeric-uniform-fit}, with \(h_r\eqdef r/(N-1)\): seven dyadic lags
through \(r=64\) for BigBird and five through \(r=16\) for DistilGPT2.
These native-length results are distinct from the earlier 16,384-token
input-only baseline.

\paragraph{Comparing multiscale shape across depth.}
Let \(M_{\infty,\ell}^{(w)}\) denote
\eqref{eq:numeric-uniform-increment} in window \(w\) at depth \(\ell\).
Figure~\ref{fig:text-depth-curves} overlays the normalized mean curves
\begin{equation}
 Q_\ell(r)\eqdef
 \frac{\overline M_{\infty,\ell}(r)}{\overline M_{\infty,\ell}(1)},
 \qquad
 \overline M_{\infty,\ell}(r)\eqdef
 \frac18\sum_{w=1}^8 M_{\infty,\ell}^{(w)}(r).
 \label{eq:text-depth-normalization}
\end{equation}
Dividing by a positive lag-independent constant does not change a log--log
slope. This normalization separates changes in curve shape from changes in
amplitude; the tables still summarize fits made separately in each window,
not fits to the displayed mean curves. The magnified vertical range is
important: visibly different curves here often differ by only 10--20\%
across hundreds of token lags.

\begin{figure}[!htbp]
 \centering
 \includegraphics[width=\textwidth]{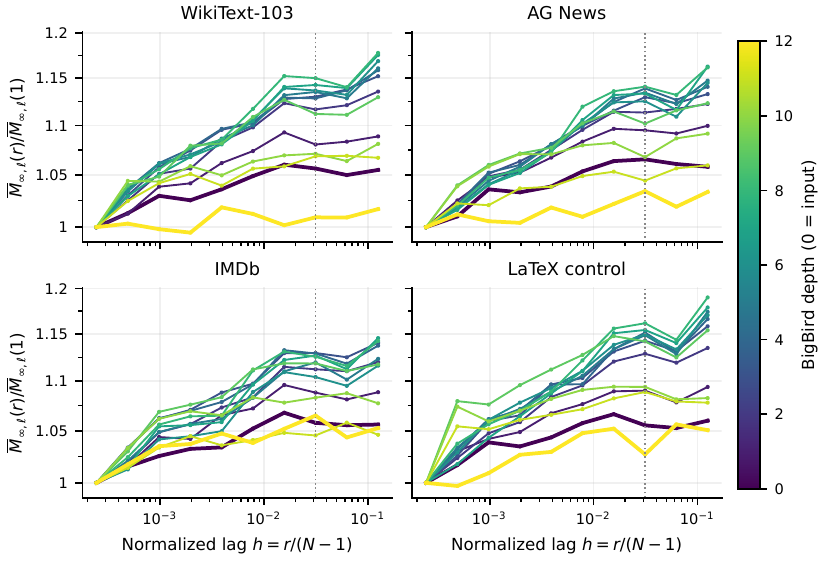}
 \caption{\textbf{BigBird's multiscale geometry varies with depth.}
 All depths are overlaid on the same log--log axes, using the same eight
 windows per source. Curves are ratios of mean worst increments as in
 \eqref{eq:text-depth-normalization}; purple and yellow mark the input and
 final block. The dotted line is the fixed fine-band cutoff \(h=1/32\).
 Intermediate blocks have steeper finite-band curves, but the gain does
 not persist to the final block.}
 \label{fig:text-depth-curves}
\end{figure}

\needspace{5\baselineskip}
\paragraph{An intermediate-layer gain in BigBird.}
The mean primary slope rises from \(0.013\)--\(0.014\) at the input to
\(0.028\)--\(0.033\) around blocks seven and eight
(Table~\ref{tab:text-depth-slopes}). Block eight improves on the input in
all eight retained windows of each source on the primary band. The mean
gain remains positive on the first-four-dyad and expanded \(h\le1/8\)
bands, although not every short-band paired difference is positive.
Thus the improvement is not merely a visual normalization effect.
It is non-monotone: final-block means return to \(0.002\)--\(0.014\).
The mean depth profiles peak at block eight for the three natural-language
sources and block seven for \LaTeX{}; these peak locations are descriptive,
not choices used to tune the fitting band.

\begin{table}[!htbp]
 \centering\footnotesize
 \caption{\textbf{Effective maximum-increment slopes across depth.}
 Mean \(\pm\) sample SD across eight matched natural-order windows per
 source, on \(h\le1/32\). BigBird's final state is block 12; DistilGPT2's
 is the output after the separate final LayerNorm. SD is descriptive, not
 a confidence interval: windows may share documents or repeated content.}
 \label{tab:text-depth-slopes}
 \setlength{\tabcolsep}{4.5pt}
 \begin{tabular}{@{}lccc@{}}
 \toprule
 \multicolumn{4}{@{}l}{BigBird (12 bidirectional blocks)}\\
 Source & Input & Block 8 & Final\\
 \midrule
 WikiText-103 & $.0130\pm.0016$ & $.0308\pm.0027$ & $.0023\pm.0060$\\
 AG News & $.0141\pm.0015$ & $.0306\pm.0042$ & $.0039\pm.0082$\\
 IMDb & $.0143\pm.0014$ & $.0279\pm.0037$ & $.0106\pm.0082$\\
 \LaTeX{} control & $.0143\pm.0019$ & $.0320\pm.0043$ & $.0140\pm.0086$\\
 \midrule
 \multicolumn{4}{@{}l}{DistilGPT2 (6 causal blocks)}\\
 Source & Input & Block 6 & Final\\
 \midrule
 WikiText-103 & $.0165\pm.0105$ & $.0108\pm.0399$ & $.0161\pm.0152$\\
 AG News & $.0176\pm.0072$ & $.0171\pm.0328$ & $.0129\pm.0204$\\
 IMDb & $.0174\pm.0093$ & $-.0105\pm.0293$ & $-.0002\pm.0187$\\
 \LaTeX{} control & $.0227\pm.0096$ & $.0150\pm.0361$ & $.0206\pm.0256$\\
 \bottomrule
 \end{tabular}
\end{table}

The gain is not specific to intact linguistic order. Shuffled-input BigBird
slopes at block eight are \(0.0363,0.0382,0.0354,0.0291\), respectively
(Figure~\ref{fig:text-depth-comparison}). This does not isolate the effects
of learned weights, positions, and sparse-attention locality, but it
prevents attributing the intermediate gain solely to meaningful word order.
The final layer also contracts amplitudes: its centered token RMS is
\(0.55\)--\(0.70\) times the input value, whereas the adjacent maximum is
\(0.82\)--\(0.90\) times its input value. Smaller increments therefore
need not imply a larger scaling exponent.

\paragraph{Causal representations and average-increment scaling.}
DistilGPT2 does not show a positive mean final-minus-input change in the
primary maximum slope for any source. For comparison, we fit the same slopes
to the root-mean-square increment
\[
 M_{2,\ell}(r)\eqdef\left(
 \frac{1}{d(N-r)}\sum_{k=0}^{N-r-1}
 \|z_{k+r}^{(\ell)}-z_k^{(\ell)}\|_2^2\right)^{1/2}.
\]
This statistic exhibits a stronger depth effect: input means are
\(0.0016\)--\(0.0062\), while final
means are \(0.0216,0.0393,0.0569,0.0661\). The corresponding shuffled means
are \(0.0231,0.0218,0.0293,0.0347\). Thus natural order improves this
average statistic except on WikiText, but an RMS gain does not establish
the uniform increment control used in the theory. Figure~\ref{fig:text-depth-comparison}
keeps the two statistics distinct and shows the final normalization
separately; it is not a seventh transformer block.

\begin{figure}[!t]
 \centering
 \includegraphics[width=\textwidth]{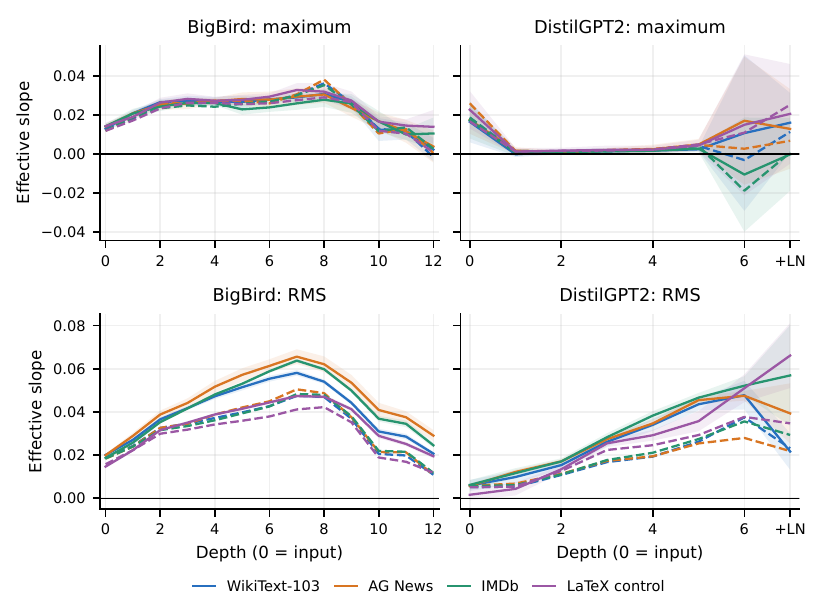}
 \caption{\textbf{Depth, token order, and the choice of increment statistic.}
 Mean fine-band slopes at every block for BigBird and causal DistilGPT2.
 Solid lines use natural order, dashed lines use shuffled inputs; shading
 is one sample SD for natural-order windows. Top: maximum increments.
 Bottom: RMS increments, a weaker statistic. DistilGPT2's \(+\mathrm{LN}\)
 point applies the final normalization to block six. Architectures,
 tokenizers, and context lengths differ, so cross-model differences do
 not isolate an effect of masking.}
 \label{fig:text-depth-comparison}
\end{figure}

\paragraph{Sensitivity and scope.}
Removing the first and last 64 content positions after inference, without
changing the context or grid denominator, leaves final maximum-slope means
between \(0.002\) and \(0.020\) for BigBird and between \(-0.001\) and
\(0.017\) for DistilGPT2. Alternative bands and percentile/RMS statistics
are retained in the saved results. Exact-lag maxima need not be monotone;
negative fits to nearly flat curves are not negative H\"older exponents.
A post-run cumulative maximum over tested dyads is also available, but it
is not the supremum over all integer lags. These fixed-context experiments
support depth-dependent \emph{effective} regularity, including a genuine
intermediate gain, not asymptotic H\"older improvement under refinement or
a common all-resolution bound. No prediction-quality claim is made.

\FloatBarrier

\end{document}